\documentclass[aoas]{imsart}
\usepackage{bm}

\RequirePackage{amsthm,amsmath,amsfonts,amssymb}
\RequirePackage[authoryear]{natbib}
\usepackage{url}
\usepackage{makecell}
\usepackage{enumitem}

\usepackage{amsmath}
\usepackage{newfloat}
\usepackage{mathtools}
\usepackage{placeins}
\usepackage{amsmath}
\usepackage{amssymb}
\usepackage{amsfonts}
\usepackage{multirow}
\usepackage{bm}
\usepackage{amsthm}
\usepackage{array}
\usepackage{graphicx}
\usepackage{graphics}
\usepackage{epsfig,latexsym,verbatim}
\usepackage{color}
\usepackage{multicol}
\usepackage{threeparttable}
\usepackage{float}
\usepackage{cancel}
\usepackage{ulem}
\usepackage{rotating}
\usepackage{graphicx}  
\usepackage{subcaption}  
\usepackage{algorithm}
\usepackage{algorithmic}
\usepackage{amsthm}
\def\bs{\boldsymbol}
\theoremstyle{plain}
\newtheorem{theorem}{Theorem}[section]

\newtheorem{lemma}[theorem]{Lemma}

\theoremstyle{definition}

\newtheorem{remark}[theorem]{Remark}
\usepackage{xcolor}  
\def\bs{\boldsymbol}
\DeclareMathOperator*{\argmax}{arg\,max}

\definecolor{shancolor}{HTML}{0072B2}

\newcommand{\Kappa}{\mathrm{K}}
\newcommand{\secref}[1]{\hyperref[#1]{Section~\ref*{#1}}}
\usepackage{hyperref}

\startlocaldefs

\endlocaldefs

\begin{document}

\begin{frontmatter}
\title{Structured Nonparametric Variational Inference for Dependent Latent Modeling}
\runtitle{Structured Nonparametric Variational Inference}

\begin{aug}
\author[A]{\fnms{Yuda}~\snm{Shao}\ead[label=e1]{ys7zg@virginia.edu}},
\author[B]{\fnms{Zhiling}~\snm{Gu}\ead[label=e2]{zhiling.gu@yale.edu}}
\and
\author[C]{\fnms{Shan}~\snm{Yu}\ead[label=e3]{sy5jx@virginia.edu}}
\address[A]{Department of Statistics, University of Virginia\printead[presep={,\ }]{e1}}

\address[B]{Department of Biostatistics, Yale University\printead[presep={,\ }]{e2}}

\address[C]{Department of Statistics, University of Virginia\printead[presep={,\ }]{e3}}
\end{aug}

\begin{abstract}
Variational inference (VI) is a core engine of modern AI, enabling scalable approximate Bayesian learning and uncertainty-aware training of large probabilistic and generative models. In this paper, we propose Structured Nonparametric Variational Inference (SN-VI), a novel framework for modeling complex dependencies among latent variables in posterior approximation, leveraging multivariate spline techniques. Unlike traditional methods that rely on the mean-field assumption, SN-VI preserves intricate latent variable dependencies, providing a flexible and accurate approximation of posteriors with arbitrary shapes. We establish rigorous theoretical guarantees, including the derivation of the lower bound for the variational objective and proof of asymptotic consistency in posterior estimation. To facilitate practical implementation, we develop an algorithm that automatically identifies dependent latent variables and their underlying dependence structure, without requiring manual specification. Simulation studies validate the effectiveness of SN-VI in approximating posterior distributions with bounded support and complex dependencies. The proposed method has been successfully applied to high-dimensional structured data, including computer vision datasets and spatial transcriptomics. In these applications, SN-VI demonstrates improved generative model performance and effectively uncovers coupled biological signals through the learned dependency structure. 
\end{abstract}


\end{frontmatter}

\section{Introduction}
Variational inference (VI) has become a cornerstone of modern AI by enabling scalable approximate Bayesian learning for complex probabilistic and deep generative models. It provides a practical route to posterior approximation when exact inference is computationally infeasible, and has therefore played a pivotal role in the development of widely used frameworks such as variational autoencoders and other latent-variable models \citep{jordan1999introduction, wainwright2008graphical,mostafa2023review}. Specifically, by selecting an approximating distribution from a predefined family $Q$, VI seeks a variational distribution $q_\phi \in Q$ that minimizes the Kullback--Leibler (KL) divergence between $q_\phi$ and the true posterior $p$, where $\phi$ denotes tractable variational parameters. As a computationally efficient alternative to Markov Chain Monte Carlo (MCMC), VI reformulates posterior inference as an optimization problem, making it highly adaptable to neural network architectures. In the era of AI, this adaptability has led to the widespread use of VI in text processing, probabilistic clustering, and genomics analysis, particularly for high-dimensional data \citep{kingma_auto-encoding_2014,lee2020stochastic,sordoni2024joint, ruiz_shopper_2020}.


Classic VI approximates intractable posteriors using a fixed parametric family, such as the Gaussian. Furthermore, many approximation families $Q$ rely on mean-field variational inference (MFVI) methods \citep{hoffman2013stochastic, kingma_auto-encoding_2014, rezende2014stochastic, kucukelbir2017automatic}, which assume that members of $Q$ can be fully factorized over latent variables. While effective in many scenarios, this approach limits the flexibility of the variational family to capture more complex posterior shapes and intricate latent variable dependencies. A common issue in VI is the underestimation of posterior variance, and this problem is particularly pronounced under the MFVI assumption, which assumes full independence among latent variables \citep{blei2017variational}. 
{In addition, the MFVI assumption becomes particularly limiting in contexts where cross-group relationships carry important information and improves interpretability, for example, correlation between topics in text processing, shared factors between latent brain subnetworks, and coupling effects between two disease subtypes.}

In order to alleviate the limitations of MFVI assumption, various Structured Variational Inference methods are proposed, which preserve dependencies between latent variables in the posterior approximation. Among these, structured mean-field methods \citep{saul1995exploiting, hoffman2015structured} provide a partial relaxation of the MFVI assumption by maintaining correlations within predefined substructures. Copula Variational Inference  \citep{tran2015copula, pmlr-v51-han16} first models the copula density to capture dependencies among latent variables, then infers the univariate marginal distributions of these latent variables, which is usually restricted to the copula family .

Richer and more accurate posterior approximations have been shown to enhance the performance of variational inference methods. A unified framework of Variational Inference with Normalizing Flows (NFVI) is proposed to preserve the dependent structure and construct richer posterior approximations \citep{rezende2015variational}. NFVI sequentially transforms a simple initial distribution into a more expressive one through a series of invertible transformations. 
{Common problems of NFVIs lie in the requirement of invertible transformation, the computational cost that is linear to the input dimensionality \citep{kingma2016improved,papamakarios2017masked} and the lack of capacity of approximating complex distribution \citep{dinh2014nice, dinh2016density}.
To overcome these limitations, a specific NFVI, Neural Spline Flows, is proposed to implement monotonic rational-quadratic splines as a building block and provide more flexible density approximations with better computational efficiency \citep{durkan2019neural}. Despite these efforts, NFVIs are unable to capture the explicit latent structures that are essential to interpretability in real applications. In addition, the assumed invertible transformation does not necessarily exist in practice \citep{kulikov_flowedit_nodate}.}


In this paper, we propose a novel framework for VI based on multivariate spline approximation, named Structured Nonparametric Variational Inference (SN-VI), designed to capture complex posterior shapes and intricate latent variable dependencies. By leveraging multivariate spline approximation, SN-VI provides the flexibility to model posteriors with complex features, such as multimodality, bounded support, and skewness. Additionally, SN-VI excels at preserving dependencies among latent variables by approximating the joint density function of dependent latent variables with multivariate splines. This approach bridges the gap between traditional variational inference methods that rely on the mean-field assumption and more flexible posterior approximations. Multivariate spline approximation effectively approximates the unknown density function \citep{gu1993smoothing}.
The theory of multivariate spline approximation has been well-developed, indicating that smooth density functions can be well approximated by linear combinations of predefined basis functions.

Unlike other structured variational inference methods, SN-VI eliminates the need for pre-specified model probability distributions. Also, compared to Normalizing flow-based variational inference (NF-VI), SN-VI provides more flexible approximations by removing the restriction on invertible transformations. Moreover, in NF-VI, increasing the number of flows enhances flexibility but also introduces significant computational overhead. In contrast, SN-VI does not rely on a sequence of invertible transformations; instead, the flexibility of the approximated posterior is governed by the number of interior knots, which typically remains small. To enhance efficiency, a Hierarchical Structure Identification Process is introduced to effectively identify the dependency structure among latent variables. Therefore, SN-VI offers a user-friendly and adaptable implementation pipeline that ensures broad applicability across diverse data structures.

\vskip 0.1in

In summary, our major contributions are as follows:
\begin{itemize}
\item[(i)] We establish a novel nonparametric structured variational inference framework, SN-VI, based on multivariate spline approximation for estimating complex posterior distributions with nonlinear dependencies. 

\item[(ii)] We establish theoretical properties for the proposed framework, including a lower bound for the importance-weighted autoencoder (IWAE) objective in Lemma \ref{THE:ELBO} and variational approximation error bounds in Theorem \ref{THE:Approximation-Regression}. 

\item[(iii)] We develop a streamlined procedure for identifying the dependence structure among latent variables, and demonstrate its effectiveness through numerical experiments. 

\item[(iv)] We conduct extensive real-data applications to spatial transcriptomics, where SN-VI reveals more comprehensive latent dependence structures that facilitate the interpretation of gene expression pathology. We further apply SN-VI to image reconstruction tasks, where it achieves superior performance compared with competing methods.
\end{itemize}

More broadly, this work highlights how statistical methodology can directly enhance modern AI tools. In contrast to mean-field variational families, which often break important latent dependencies, and flow-based approaches, which rely on invertible transformations and may not clearly reveal the underlying structure, spline-based approximation provides a transparent and theoretically grounded way to model joint latent behavior and quantify approximation error. From this perspective, SN-VI is not only a new variational inference algorithm, but also an example of how ideas from nonparametric statistics can improve the reliability, interpretability, and scientific usefulness of AI models in applications involving structured high-dimensional data.

The remainder of this paper is organized as follows. Section \ref{gen_inst} reviews the background of variational inference and introduces the key concepts of spline approximation and latent variable dependence that motivate our framework. Section \ref{headings} presents the proposed Structured Nonparametric Variational Inference (SN-VI) model in detail, including the multivariate spline-based posterior approximation and the associated estimation algorithm. Section \ref{sec:properties} establishes the theoretical properties of SN-VI, providing lower-bound guarantees and asymptotic convergence results. Section \ref{sec:str} describes the hierarchical structure identification procedure for learning latent-variable dependencies. Section \ref{sec:simulation} reports extensive simulation studies to evaluate the performance of SN-VI, while Section \ref{sec:application} demonstrates its practical utility through real-world applications to spatial transcriptomics and image data. Concluding remarks and potential future research directions are discussed at the end.

\section{Background}
\label{gen_inst}

\subsection{Variational Inference}
In VI, given the observed variables $\bm x$ and the latent variables $\bm z$ with model likelihood $p(\bm x| \bm z)$ and prior distribution $p(\bm z)$, the posterior distribution is $p(\bm z| \bm x) = \frac{p(\bm x |\bm z)p(\bm z)}{\int p(\bm x| \bm z)p(\bm z)d\bm z}$. However, the integral marginal likelihood function is intractable for most flexible generative models. VI aims to minimize KL$(q_{\phi}(\bm z| \bm x)||p(\bm z|\bm x))$ w.r.t. $\phi$, where $q_\phi(\bm z|\bm x)$ is the variational distribution with $\phi$ being a set of unknown parameters depending on the observed data $\bm x$. However, model optimization requires computing $\log p(\bm x)$, which is usually unavailable in closed form.

Since $\log{p(\bm x)}=\mathcal{L}_{\text{ELBO}}\left\{\phi(\bm x)\right\} + \text{KL}\left(q_{\phi}(\bm z| \bm x)||p(\bm z|\bm x)\right)$, with $\mathcal{L_{\text{ELBO}}}\left\{\phi(\bm x)\right\}$ defined as the evidence lower bound (ELBO): 
$
\mathcal{L_{\text{ELBO}}}\left\{\phi(\bm x)\right\} = \mathbb{E}_{q_{\phi}(\bm z|\bm x ) }\left[\log{p(\bm x| \bm z)}\right] - \text{KL}(q_{\phi}(\bm z|\bm x )|| p(\bm z)),
$
$\log{p(\bm x)}$ is independent with $q_{\phi}(\bm z | \bm x)$ and KL$(\cdot)\geq 0 $, 
minimizing KL$(q_{\phi}(\bm z| \bm x)||p(\bm z|\bm x))$ is equivalent to maximizing ELBO \citep{blei2017variational}.
A tighter bound, named importance-weighted autoencoder (IWAE) \citep{DBLP}, is proposed with a strictly tighter log-likelihood lower bound than ELBO to achieve more accurate posterior distribution.
By drawing $T$ independent samples from the posterior $\left\{\bm z^{(t)}\right\}_{t=1}^T \sim q_\phi(\bm z| \bm x)$, IWAE is formulated as the log of the average of the ratio of the joint distribution and posterior for each sample:
\begin{align}
  \label{EQU:IWAE} 
   & \mathcal{L}_{\text{IWAE}}\{\phi(\bm x)\}
  \triangleq \mathbb{E}_{\left\{\bm z^{(t)} \sim q_\phi(\bm z| \bm x)\right\}_{t=1}^T}\left[\log \frac{1}{T} \sum_{t=1}^T \frac{p\left(\bm x|\bm z^{(t)}\right) p\left(\bm z^{(t)}\right)}{q_\phi\left(\bm z^{(t)}| \bm x\right)}\right].
\end{align}
While a tighter bound of ELBO or IWAE is not always optimal \citep{rainforth2018tighter}, prior works suggest that multiple posterior samples improve IWAE’s adaptation to multimodal and complex posteriors \citep{DBLP, morningstar2021automatic}. Accordingly, we adopt IWAE as the objective for its overall advantages, noting that the approximation capability of SN-VI is not constrained by the specific choice of loss function.

\subsection{Nonlinear Dependence among Latent Variables}
\label{sec:latent_dep}
For latent variables $\bm z = \{z_1, \ldots, z_J\}$, let $\mathbb{G}=\{\mathcal{G}_{1}, \ldots, \mathcal{G}_{M}\}$ be a partition of $\{1, \ldots, J\}$, which represents the underlying group structure. The latent variables from different groups are independent of each other. There exists nonlinear dependence for the latent variables from the same group. Let denote $|\mathcal{G}_m|$ as the number of latent variables in each group $\mathcal{G}_m$ and $\widetilde{\bm z}_m = \{z_{j} \mid  j \in \mathcal{G}_m\}$ as the vector of latent variables belonging to the group $\mathcal{G}_{m}$. Then, we re-define $\bm z = \{\widetilde{\bm z}_1, \ldots, \widetilde{\bm z}_M\}.$ Figure \ref{fig:ill} illustrates the latent variables group structures. With the above assumption, the underlying true posterior distribution can be factorized as $p(\bm{z}|\bm{x}) = \prod_{m=1}^{M}p(\widetilde{\bm{z}}_{m}|\bm{x})$.

\begin{figure}
    \centering
    \includegraphics[width=.7\linewidth]{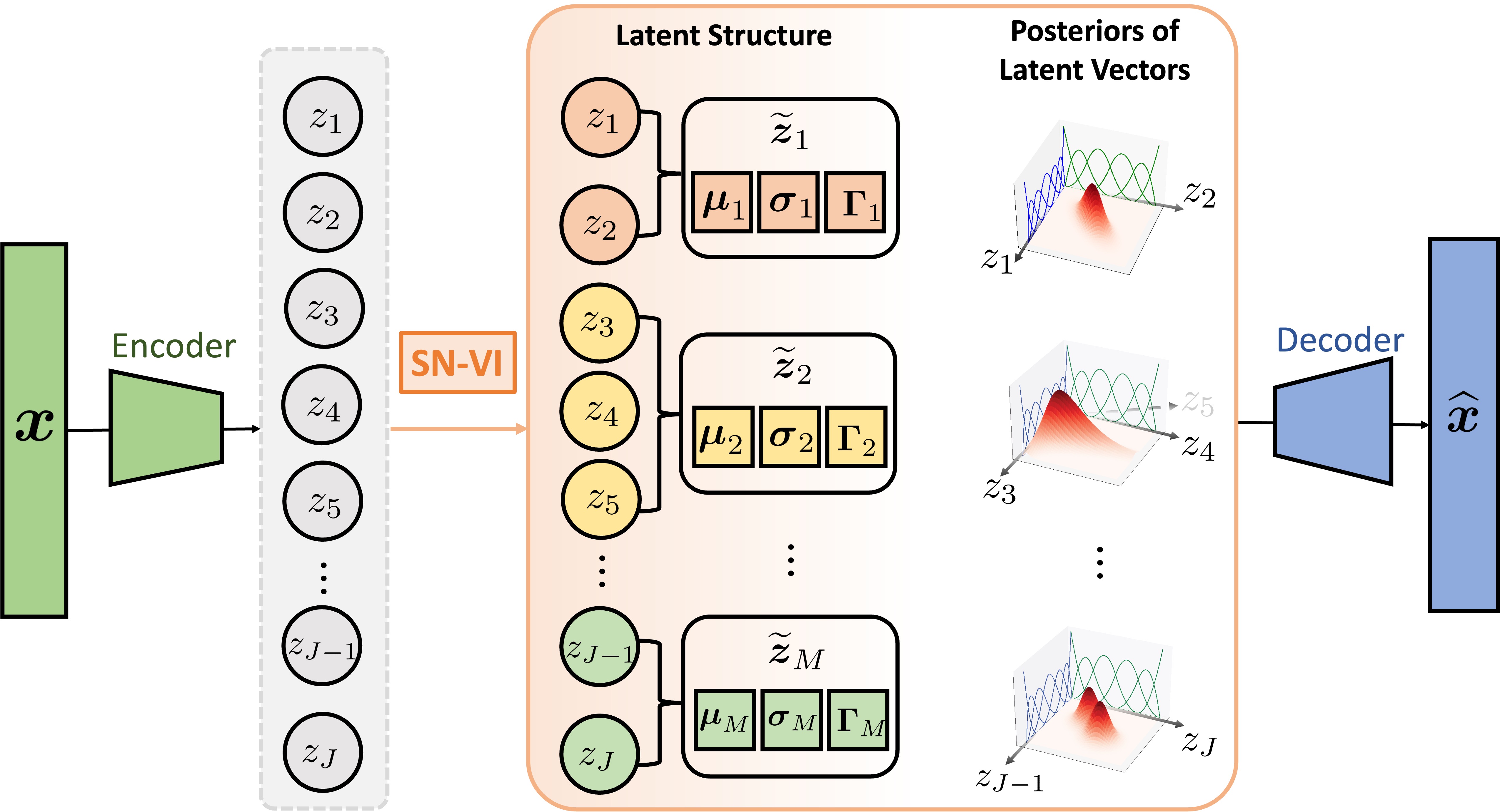}
    \caption{{Architecture of the proposed SN-VI framework. SN-VI identifies the structures of latent variables and produces joint posterior distribution of latent vectors. Among all latent variables $z_j, j=1,\ldots, J$, some share dependency structures that can be statistically described using location $\bs\mu_m (\bs x)$, scale $\bs\sigma_m (\bs x)$ and shape $\bs\Gamma_m(\bs x)$ for latent group $m = 1, \ldots, M$. The joint and marginal posterior density of two latent variables within latent groups are also illustrated.}}
    \label{fig:ill}
\end{figure}


\subsection{Spline Approximation}
\label{SEC:Spline-App}

Spline approximation efficiently models complex curves with a compact set of parameters and is widely used in machine learning and statistics, including spatial-temporal analysis \citep{Yu:etal:20,yu2021multivariate}, neural networks \citep{pmlr-v80-balestriero18b,fakhoury2022exsplinet}, and Bayesian inference \citep{wood2002bayesian}.

\paragraph*{Univariate Spline} In Spline Automatic Differentiation Variational Inference (S-ADVI) \citep{shao2024nonparametric}, univariate spline has been applied in approximating complex posterior distributions for mean-field variational distributions. Let us review the univariate spline. Denote $\bs{\upsilon}$ as a partition of the interval $\mathcal{T} = [\upsilon_0, \upsilon_{H+1}]$ with $H$ interior knots, where $\bs{\upsilon}=\left\{\upsilon_{0}<\upsilon_{1}<\cdots <\upsilon_{H}<\upsilon _{H+1}\right\}$. 
Denote $\{B_{1,\mathcal{T}}(z),\ldots ,B_{K,\mathcal{T}}(z)\}^{\top}$ as a vector of the spline basis functions with degree $\varrho$ and partition $\bs{\upsilon}$, where $K = H + \varrho+1$. For the sake of notation simplicity,  for the rest of the paper, we define the normalized spline basis by $b_{k,\mathcal{T}}(z) \triangleq B_{k,\mathcal{T}}(z)/a_{k,\mathcal{T}}$, where $a_{k,\mathcal{T}} = \int_{\mathcal{T}}B_{k,\mathcal{T}}(z) dz$. It implies that $\int_{\mathcal{T}} b_{k,\mathcal{T}}(z) dz = 1$. Let $\bm b_\mathcal{T}(z)=\{b_{1,\mathcal{T}}(z),\ldots ,b_{K,\mathcal{T}}(z)\}^{\top}$. 
For any polynomial spline $s(z)$, it can be uniquely represented via a linear combination of spline basis functions, that is, $s(z) = \sum_{k=1}^K \gamma_kb_{k,\mathcal{T}}(z)$.

\paragraph*{Multivariate Spline Approximation} Tensor Product Spline, which was first proposed in \citep{grosse1980tensor}, have been widely adopted by multivariate function approximation \citep{wheeler2019bayesian,rugamer2024scalable}. For the proposed method, under the assumed latent variables group structure, we apply the tensor product spline to approximate the $p(\widetilde{\bm{z}}_{m}|\bm{x})$, $m =1,\dots, M$. Let us first construct $D$-dimensional spline surfaces from families of spline functions. For $d = 1,\dots, D$, suppose that for some integer $\varrho_d$ and knot vector $\bs{v}_d$ with $H_d$ interior knots, we have the univariate spline space
$\mathbb{S}_d=\mathbb{S}_{\varrho_d, \boldsymbol{v}_d}=\operatorname{span}\left\{b_{k_d, \mathcal{T}_d}\right\}_{k_d=1}^{K_d}.$ Define the space of $D$-dimensional tensor product splines
\begin{equation}
    \mathcal{S}^{(D)} = \bigotimes_{d=1}^{D} \mathbb{S}_{d}  = \text{span}\left\{  b_{\bm \kappa} \right\}_{\bm \kappa \in \bm \Kappa_{D}},
\end{equation}
where $\bm \kappa = (k_1, \ldots, k_D)$ as the D-dimensional index set and $b_{\bm \kappa}(\bm z) = b_{k_{1},\mathcal{T}_1}(z_1) \times \cdots \times  b_{k_{D},\mathcal{T}_D}(z_D)$ are the spline basis functions and $\bm \Kappa_D = \left\{(k_1,\dots,k_D)\right\}_{k_1 =1,\dots,k_D = 1}^{K_{1},\dots,K_D}$.
Then, any $D$-dimensional function can be approximated by
\begin{align*}
    &s(\bm z)  = \sum_{k_1 =1}^{K_1} \cdots \sum_{k_D =1}^{K_D} \gamma_{k_1,\dots,k_D}b_{k_1,\mathcal{T}_1}(z_1) \cdots b_{k_D,\mathcal{T}_D}(z_D) = \sum_{\bm \kappa \in \bm \Kappa_D} \gamma_{\bm \kappa} b_{\bm \kappa}(\bm z),
 \end{align*}
where $\bm z = (z_1,\dots,z_D)$.
The number of basis functions for $D$-dimensional tensor product spline is $\prod_{d=1}^{D} K_d$. Figure \ref{fig:ill} illustrates examples of tensor product spline basis functions when $D=2$. 

Then, the joint density of $\widetilde{\bm z}_m$ can be represented by $\sum_{\bm \kappa \in \bm K_{|\mathcal{G}_{m}|}} \gamma_{\bm \kappa} b_{\bm \kappa}(\widetilde{\bm z}_m)$.
Without loss of generality, in the following, for each group $\mathcal{G}_{m}$,  we consider that $\mathbb{S}_1 = \mathbb{S}_2 = \cdots = \mathbb{S}_{|\mathcal{G}_{m}|} = \mathbb{S}$, and $\mathbb{S} = \text{span}\left \{b_{k,\mathcal{T}}\right\}_{k=1}^{K=1}$ with the same $H$ interior knots and different degree $\varrho_m$. 

We study the relationship between the smoothness of a function and its approximation accuracy using tensor product splines. Let \(\mathcal{H}^{(\varrho)}(\mathcal{T}^{D})\) denote the space of \(D\)-dimensional functions \(\psi\) defined on the tensor product domain \(\mathcal{T}^{D} = \mathcal{T}_{1} \times \cdots \times \mathcal{T}_{D}\). For any multi-index \(\bm{\nu} = (\nu_1, \dots, \nu_D)^\top\), the mixed partial derivative $
\mathcal{D}^{(\bm \nu)}\psi(\bm z) = \frac{\partial^{[\bm \nu]}\psi(\bm z) }{\partial z_1^{\nu_1} \cdots \partial z_D^{\nu_D}}
$
is assumed to be continuous and satisfies a Lipschitz condition of order \(\delta\), i.e.,
$
\left|\mathcal{D}^{(\bm \nu)}\psi(\bm z)-\mathcal{D}^{(\bm \nu)}\psi\left(\bm z^{\prime}\right)\right| \leq C_{\bm \nu}\left\|\bm z-\bm z^{\prime}\right\|_{2}^{\delta}, 
$
for all \(\bm{z}, \bm{z}^{\prime} \in \mathcal{T}^{D}\). The smoothness parameter \(\varrho\) is defined by the relation \(\varrho+1 = [\bm \nu ] + \delta\), where \(0 < \delta \leq 1\). Here, \(\|\bm{z}\|_{2} = \left(\sum_{d=1}^{D}z_{d}^{2}\right)^{1/2}\) denotes the Euclidean norm of \(\bm{z}\). The parameter \(\varrho\) characterizes the level of smoothness of the function \(\psi\).

The following lemma provides the approximation power of tensor product spline. The proof readily follows from Theorem 12.7 in \citep{schumaker2007spline}. 
\begin{lemma}
\label{LEM:approxi_tensor}
For any function  $\psi \in \mathcal{H}^{(\varrho)}(\mathcal{T}^{D})$, there exists a spline $\psi^{\ast} \in \mathcal{S}^{(D)}$, such that
$\sup_{\bm z \in \mathcal{T}^{D}}|\psi^{\ast}(\bm z)-\psi(\bm z)|\leq C H^{-(\varrho+1)}$ for some constant $C$ depending on $D$, $C_{\bm \nu}$.
\end{lemma}

\begin{remark}
\label{Remark2}
According to Lemma \ref{LEM:approxi_tensor}, for density function $p(\widetilde{\bm z}_{m}) \in \mathcal{H}^{(\varrho_{m})}\left(\mathcal{T}^{|\mathcal{G}_{m}|}\right)$, where $\mathcal{T}^{|\mathcal{G}_{m}|} \subseteq \mathbb{R}^{|\mathcal{G}_{m}|}$ could be either finite or infinite support, there exists $q( \widetilde{\bm z}_{m})  \in \mathcal{S}^{(|\mathcal{G}_{m}|)} $  such that $q( \widetilde{\bm z}_{m})$ is a valid density function and $\sup_{\widetilde{\bm z}_{m}\in \mathcal{T}^{|\mathcal{G}_{m}|}}|q(\widetilde{\bm z}_{m})-p(\widetilde{\bm z}_{m})|\leq C_{1} H^{-(\varrho_{m}+1)}$.
\end{remark}

\section{Nonparametric Posterior Estimation with Multivariate Spline Approximation}
\label{headings}
With the given group structure $\mathbb{G}$, we assume each group $\mathcal{G}_{m}$'s true posterior $p({ \widetilde{\bm{z}}_{m}|\bm x})\in \mathcal{H}^{(\varrho_m)}\left(\mathcal{T}^{|\mathcal{G}_{m}|}\right)$. Therefore, according to Lemma \ref{LEM:approxi_tensor} and Remark \ref{Remark2}, the posterior distribution of latent variables $\widetilde{\bm{z}}_{m}$, $p(\widetilde{\bm{z}}_{m}| \bm x)$ can be well approximated by a spline function: $q(\widetilde{\bm{z}}_{m}|\bm x) = \sum_{\bm \kappa \in \bm \Kappa_{|\mathcal{G}_{m}|}} \gamma_{\bm \kappa} b_{\bm \kappa}(\widetilde{\bm{z}}_{m}).$ By the definition of normalized spline basis in Section \ref{SEC:Spline-App}, 
$\int_{\mathcal{T}^{|\mathcal{G}_{m}|}} b_{\bm \kappa}(\widetilde{\bm{z}}_{m})d\widetilde{\bs{z}}_m  =1$, for $\bm \kappa \in \bm \Kappa_{|\mathcal{G}_{m}|}$ and $ m=1, \ldots, M$. Therefore, to guarantee that $q_\phi(\widetilde{\bm{z}}_{m}|\bm x)$ is a valid density function for $\widetilde{\bm{z}}_{m} \in \mathcal{T}^{|\mathcal{G}_{m}|}$, the spline coefficients must satisfy $\gamma_{\bm \kappa}\geq 0$ and $\sum_{\bm \kappa \in \bm{\Kappa}_{|\mathcal{G}_{m}|}}\gamma_{\bm \kappa} =1 $. For simplicity, we set $\mathcal{T} = [0,1]$ for the rest of paper. For $j$-th latent variable, we apply location-scale transformations, such that $z_j = \mu_j(\bm x) + \sigma_j(\bm x) \epsilon_j$, and $\epsilon_j \in [0,\,1]$ is a random variable. With the location-scale transformation, the flexible support of $z_j$, $[\mu_j(\bm x), \mu_j(\bm x) + \sigma_j(\bm x)]$, can be captured by estimating the unknown parameters $\mu_j(\bm x)$ and $\sigma_j(\bm x)$.   In summary, the proposed posterior of group $\mathcal{G}_{m}$ latent variables is determined by unknown parameters $\bm \mu_{m}(\bm x) =\left\{\mu_{j}(\bm x)\mid j\in \mathcal{G}_{m}\right\}$, $\bm \sigma_{m}(\bm x) = \left\{\sigma_{j}(\bm x)\mid j\in \mathcal{G}_{m}\right\}$, and the spline coefficients vector $\bm \Gamma_{m}(\bm x) = \left\{\gamma_{\bm \kappa}(\bm x) \mid \bm \kappa \in \bm \Kappa_{|\mathcal{G}_{m}|}\right\}$, which capture the location, scale, and shape of each group $\mathcal{G}_m$ density function. Let $J = \sum_{m=1}^{M}|\mathcal{G}_{m}|$ be the total number of latent variables and $\bm \epsilon_{m} = (\widetilde{\bm z}_{m} - \bm \mu_{m}(\bm x))/\bm \sigma_{m}(\bm x)$. Collectively, for the vector of latent variables $\bm z$, the posterior $p\left(\bm z| \bm x\right)$ can be approximated by
\begin{align}
\label{EQU:S-ADVI}
 q_\phi \left(\bm z| \bm x\right) & =  \frac{1}{\prod_{j=1}^J\sigma_j(\bm x)} \cdot \prod_{m=1}^{M} q_{\phi}(\bm \epsilon_{m}|\bm x) \\
 & = \frac{1}{\prod_{j=1}^J\sigma_j(\bm x)} \prod_{m=1}^{M} \sum_{\bm \kappa \in \bm \Kappa_{|\mathcal{G}_{m}|}} \gamma_{\bm \kappa}(\bm x) b_{\bm \kappa}\left(\frac{\widetilde{\bm z}_{m} - \bm \mu_{m}(\bm x)}{\bm \sigma_{m}(\bm x)}\right). \notag
\end{align}
In the SN-VI, the parameters of approximation family are $\phi = \{ \bm \mu_m(\bm x), \bm \sigma_m(\bm x), \bm \Gamma_{m}(\bm x) \mid m =1 , \dots, M\}$.
The objective of the proposed SN-VI is to maximize the IWAE defined in (\ref{EQU:IWAE}), where the term $q_\phi\left(\bm z_t| \bm x\right)$ is given in (\ref{EQU:S-ADVI}). The spline degree $\varrho_m$, and the number and values of interior knots are hyperparameters to be specified. In the numerical studies, we choose the cubic spline ($\varrho_m = 3$) with equal-space knots, which is commonly used in nonparametric estimation \citep{hastie2017generalized,Yu:etal:20}.

\subsection{Model Estimation}
\label{SEC:modelestimation}
\vspace{-2mm}
This section introduces the algorithm based on Automatic Differentiation Variational Inference (ADVI), a scalable and flexible variational inference method, to automatically optimize the IWAE for SN-VI \citep{blei2017variational}. Given the structure in SN-VI, the algorithm follows these steps:

\paragraph*{Generating Random Samples from Mixture Models} One key component of SN-VI is to generate random samples from the distribution $q_\phi(\widetilde{\bm{z}}_{m}|\bm x)$, for $m=1,\dots,M$. The multivariate spline posterior approximation can be interpreted as a mixture of density functions constructed from spline bases. Optimizing ELBO and IWAE with mixture densities often involves Stratified ELBO (SELBO) and Stratified IWAE (SIWAE) \citep{roeder2017sticking, morningstar2021automatic}. In SN-VI, applying stratified techniques is challenging due to the need to sum over all spline coefficient-latent variable products, which becomes computationally expensive for high-dimensional latent spaces


To improve computational efficiency while retaining a differentiable sampling procedure, we use the concrete distribution \citep{maddison2017concrete}, which provides a continuous relaxation of the categorical distribution. Let
\[
\bm{u} \sim \operatorname{Concrete}(\bm{\alpha},\tau),
\]
where $\bm{\alpha}=(\alpha_1,\ldots,\alpha_L)$ denotes the event-probability parameter of the target categorical distribution, and $\tau>0$ is the temperature parameter of the concrete distribution. The random vector $\bm{u}=(u_1,\ldots,u_L)$ lies on the probability simplex, with $u_l\geq 0$ and $\sum_{l=1}^L u_l=1$, and can be viewed as a differentiable approximation to a one-hot categorical indicator. As $\tau\to 0$, $\bm{u}$ becomes increasingly close to a one-hot vector, and the concrete distribution converges to the categorical distribution with event probabilities $\bm{\alpha}$. Additional details on the concrete distribution are provided in Section~B.4.1 of the Supplementary Materials.

We further adopt an annealing strategy \citep{abid2019concrete} for the temperature parameter. Specifically, training begins with a relatively large temperature $\tau_0$, which yields smooth weights across multiple spline basis components and improves numerical stability during the early stages of optimization. The temperature is then gradually decreased after each epoch. As the temperature decreases, the relaxed weights become increasingly concentrated on a small number of basis components, eventually providing a close approximation to discrete component selection. Further details and discussion are provided in Sections B.4.2. -- B.4.5.
The procedure follows these steps: \vspace{-1mm}
\begin{enumerate}
    \item For the group of latent variables $\mathcal{G}_{m}$ and $\bm \kappa \in \bm \Kappa_{|\mathcal{G}_{m}|}$, randomly pick the samples $\bm{\omega}_{\bm \kappa} = \left\{\omega_{j} \mid j \in \mathcal{G}_{m}\right\}$ from a set of random samples based on $b_{\bm \kappa}(\bm \epsilon_{m})$.
   
    \item Generate random sample $ \bm u_{m} = \left\{u_{\bm \kappa}\mid \bm \kappa \in \bm \Kappa_{|\mathcal{G}_{m}|}\right\}$ from $\operatorname{Concrete}(\bm \Gamma_{m}(\bm x),\Lambda(c))$, where $c$ is the current epoch and $\Lambda(c)$ is the annealing function. Define $\bm \epsilon_m = \sum_{\bm \kappa \in \bm \Kappa_{|\mathcal{G}_{m}|}} u_{\bm \kappa} \bm{\omega}_{\bm \kappa}.$
   
\end{enumerate}

\paragraph*{Backpropagation with Reparameterization Trick} To optimize the IWAE in (\ref{EQU:IWAE}),  instead of directly optimizing over $\bm{z}$, we reparameterize in terms of an auxiliary variable $\bm{\epsilon}$ drawn from the mixture of tensor product basis spline functions:
\begin{align}
     & \mathbb{E}_{\left\{\bm \epsilon^{(t)}\right\}_{t=1}^T}\left[\log \frac{1}{T} \sum_{t=1}^T \frac{p\left\{\bm x,\bm \mu(\bm x) + \bm \sigma (\bm x) \cdot \bm \epsilon^{(t)} \right\}}{\prod_{m=1}^M\sum_{\bm \kappa \in \bm \Kappa_{|\mathcal{G}_{m}|}} \gamma_{\bm \kappa}(\bm x) b_{\bm \kappa}\left(\bm \epsilon_{m}^{(t)}\right)}\right] + \sum_{j=1}^J\log \sigma_j(\bm x).
      \label{EQU:backpropagation}   
\end{align}
Derivation of (\ref{EQU:backpropagation}) can be found in Supplement Section B.3.  This ensures differentiability with respect to $\bm{\mu}$ and $\bm{\sigma}$, facilitating gradient-based optimization. The expectations are estimated using Monte Carlo sampling, based on the random samples generated from the mixture of spline basis functions.


%

\section{Properties of SN-VI}
\label{sec:properties}
This section studies theoretical properties of SN-VI.

\subsection{Asymptotic Error Bounds}
\label{assumptions}
 Here are the assumptions for our theoretical studies. 
\begin{itemize}
\item[\textit{(A1)}] The prior $p(\bm z)$ and the likelihood function $p(\bm x| \bm z)$ are bounded over the support regions. 
    \item[\textit{(A2)}] The true posteriors can be partially factorized, that is, $p(\bm z| \bm x) = \prod_{m=1}^M p(\widetilde{\bm z}_m | \bm x)$. For group $m=1, \ldots, M$, the posterior $p(\widetilde{\bm z}_{m}|\bm x) \in \mathcal{H}^{(\varrho_{m})}\left(\mathcal{T}^{|\mathcal{G}_{m}|}\right)$. There exists some region $\mathcal{T}^{\ast|\mathcal{G}_{m}|} \subset \mathcal{T}^{|\mathcal{G}_{m}|}$ such that $\int_{\mathcal{T}^{|\mathcal{G}_{m}|} - \mathcal{T}^{\ast|\mathcal{G}_{m}|}}p(\widetilde{\bm z}_{m} | \bm x) d\widetilde{\bm z}_{m} < \epsilon$ for some $\epsilon >0$. In addition, the posterior $p(\widetilde{\bm z}_{m}| \bm x)$ are bounded by some constant over $\mathcal{T}^{\ast|\mathcal{G}_{m}|}$. 
    \item[\textit{(A3)}] There exist constants $c$ and $C$ such that $c H^{-1} \leq \upsilon_{h} - \upsilon_{h-1} \leq C H^{-1}$ for $1\leq h \leq H$. 
\end{itemize}
\vspace{1ex}

\begin{remark}
    Assumption (A1) is a mild assumption on the prior $p(\bm z)$ and the likelihood function.  Assumptions (A2) -- (A3) are typical assumptions under the framework of spline approximation \citep{Yu:etal:20}. Assumption (A2) assumes that the true posteriors can be factorized given the group structure $\mathbb{G}$. 
\end{remark}

First, the lower bound of IWAE (Lemma \ref{THE:ELBO}) is derived, leading to an upper bound on the KL divergence for the tensor product spline density approximation to the true posterior (Theorem \ref{THE:Approximation}). Then, Theorem \ref{THE:Approximation-Regression} quantifies the posterior approximation error between the SN-VI estimator and the true posterior. See Supplement Section A.1
for the detailed proof.

\setcounter{theorem}{0}

 \begin{lemma}
 \label{THE:ELBO}
  Under Assumptions (A1) -- (A3), the optimal IWAE is bounded by $\log p(\bm x) - C \sum_{m=1}^{M} H^{-(\varrho_{m}+1)}- M\epsilon$ where $C$ is some positive constant.
\end{lemma}
 \begin{theorem}
 \label{THE:Approximation} Under Assumptions (A1) -- (A3), the difference between the true posterior and the spline estimator is bounded by the order of $\sum_{m=1}^{M}H^{-(\varrho_m+1)}$, that is, there exists a constant $C$, such that $D_{KL}\{q_{\phi(\bm x)}(\bm z)||p(\bm z| \bm x)\} \leq C\left(\sum_{m=1}^{M}H^{-(\varrho_m+1)} +M\epsilon\right)$.
\end{theorem}


Theorem \ref{THE:Approximation-Regression} quantifies variational approximation error by applying multivariate spline approximation and theoretical differences between the SN-VI estimator and true posterior. See Supplement Section A.2 for the detailed proof. 
 
 \begin{theorem}
 \label{THE:Approximation-Regression} Under Assumptions (A1) -- (A3), the average KL divergence of the multivariate spline approximation from the true posterior satisfies
\begin{align*}
&\lim_{n \to \infty} \Pr\Bigg[\frac{1}{n} \textstyle \sum_{i=1}^n D_{KL}\{q_{\widehat{\phi}(\bm x_i)}(\bm z)||p(\bm z| \bm x_i)\} \leq C \sum_{m=1}^{M}(H^{-2(\varrho_{m}+1)} + \epsilon^2 + H^{2|\mathcal{G}_{m}|}  \Delta^2)\Bigg] = 1,
\end{align*}
 where $C$ is a positive constant and $\Delta$ is the $L_2$ estimation error of nonparametric regression for $\bm \mu_m(\bm x)$, $\bm \sigma_m(\bm x)$, and $\bm \Gamma_{m}(\bm x), m = 1, \ldots, M.$
\end{theorem}
\begin{remark}
     $\Delta$ is the $L_2$ estimation error, which is usually decreasing when increasing the sample size \citep{NonparametricRegression}. Theorem \ref{THE:Approximation} suggests increasing the number of interior knots $H$ can reduce the SN-VI approximation errors and increase the model estimation errors ($\sum_{m=1}^MH^{2|\mathcal{G}_{m}|}  \Delta^2$). However, when applied to real data analysis, according to the results in Theorem \ref{THE:Approximation-Regression}, choosing the optimal number of interior knots balances approximation bias and estimation variance.
\end{remark}


\subsection{Computation Complexity Analysis}
The main computational cost of SN-VI comes from evaluating and optimizing the spline-based variational approximation within each dependence group. Suppose the latent variables are partitioned into groups $\{\mathcal{G}_m\}_{m=1}^M$, where $|\mathcal{G}_m|$ denotes the group size. For group $\mathcal{G}_m$, $K$ spline basis functions are used for each dimension and the total number of tensor-product spline basis functions grows as $
K_m = O(K^{|\mathcal{G}_m|}).$
More specifically, if $S$ samples are used to update the model per iteration, the computational cost per iteration is approximately
$
O\left(S\sum_{m=1}^M K_m\right),
$
up to the cost of evaluating the likelihood and prior terms. The memory cost is also of order
$
O\left(\sum_{m=1}^M K_m\right),
$
because the spline coefficients need to be stored and updated during optimization.

This analysis explains why SN-VI is most computationally efficient when the latent dependence structure is relatively low-dimensional or can be decomposed into small groups. In our experiments, we mainly consider settings with $|G_m|\leq 4$, where the number of spline coefficients remains moderate and the optimization can be efficiently carried out on a standard GPU such as an NVIDIA A100 80GB. When $|G_m|$ becomes larger, the number of tensor-product basis functions can grow rapidly, leading to substantially higher computation and memory requirements, see Section C.1 in the supplementary materials for detailed experiment results. Thus, in practice, SN-VI is particularly suitable for models with sparse, local, or low-rank dependence structures, where complex posterior dependence can be captured through several low-dimensional dependence groups rather than one high-dimensional joint approximation.

To further reduce the computational burden, one can adopt a basis-subsampling strategy \citep{consagra2026neuropmd}. Instead of optimizing over the full tensor-product spline basis, we randomly sample a subset of basis functions at each optimization step or in a preliminary screening stage. This reduces the effective number of basis functions from $K_m$ to $\widetilde K_m\ll K_m$, leading to an approximate per-iteration cost $O\left(S\sum_{m=1}^M \widetilde K_m\right).$
Empirically, our additional experiments in Section C.1.1 in the supplementary materials show that this subsampling strategy preserves the main dependence patterns and achieves performance close to the full-basis implementation, while substantially reducing runtime. These results suggest that basis subsampling provides a practical way to scale SN-VI to larger models.

\section{Structure Identification}
\label{sec:str}
In previous sections, we introduced the objective function (\ref{EQU:backpropagation}) based on the group structure \(\mathbb{G}\). However, in real-world applications, the group structure of latent variables is often not provided in advance. To identify the structures of these latent variables, we utilize the predictive log-likelihood as a metric for model selection. This approach has proven effective for model selection in various contexts \citep{nott2012regression}. The predictive log-likelihood, denoted as \(L(\mathbb{G})\), for SN-VI is defined as follows:
\begin{equation}
\label{equ:pred_log_like}
    L(\mathbb{G})
 = \log{\int p(\bm x_{\text{valid}}|
\bm z) q^{\mathbb{G}}_{\phi}(\bm z| \bm x_{\text{train}})d\bm{z}} ,
\end{equation}
where $\bm x_{\text{train}}$, $\bm x_{\text{valid}}$, and $q^{\mathbb{G}}_{\phi}(\bm z| \bm x_{\text{train}})$ indicates training, validation dataset and approximated posterior distributions given structure $\mathbb{G}$, respectively. Due to the intractability of the integral in (\ref{equ:pred_log_like}), we apply the Monte Carlo approximation to approximate the integral. 

The detailed procedure is provided in Algorithm 1 in Supplementary Section B.5. Initially, all latent variables are assumed to be independent, with each variable considered as its own group. Variables within the same group are assumed to be dependent, while variables from different groups are independent. At each step, the algorithm merges groups that result in the largest increase in predictive log-likelihood. This process continues until the predictive log-likelihood no longer increases, at which point the optimal model is identified as the one with the largest predictive log-likelihood.

During hierarchical structure identification, we impose an explicit upper bound on the maximum allowable group size $|\mathcal{G}_m|$. This restriction is necessary because the number of spline basis coefficients also grows quickly as the group dimension increases. Without such a restriction, searching over all possible high-order groups would become computationally prohibitive in large latent spaces.

The maximum group size can be viewed as a tuning parameter that controls the complexity of the structured variational family. A small upper bound favors computational efficiency and captures low-order dependence, while a larger upper bound allows more flexible higher-order dependence at the cost of increased computation. In our experiments, we set this upper bound to a small value, such as $|\mathcal{G}_m|\leq 4$, which provides a practical balance between flexibility and scalability. This choice is also consistent with the goal of SN-VI, which is to capture structured and local dependence among latent variables rather than unrestricted fully dense dependence. For applications with strong higher-order dependence, the upper bound can be increased when computational resources permit.

\section{Experimental Results}
\label{sec:simulation}
In this section, we present simulation studies and real data applications to evaluate the effectiveness of the proposed method. All experiments were conducted on an NVIDIA A100 80GB GPU. Throughout the numerical studies, we distinguish among three likelihood-related quantities. 
First, the training objective is the IWAE lower bound defined in (\ref{EQU:IWAE}). Second, for structure identification, we use 
the predictive log-likelihood defined in (5), which evaluates the validation likelihood averaged 
over samples from the fitted variational posterior. Third, in the real-data applications, when we 
report negative log-likelihood (NLL), it refers to the negative conditional log-likelihood of the 
observation model, $-\log p(x\mid z)$, evaluated using posterior samples or posterior-mean 
embeddings from the fitted model. NLL values are comparable across methods within 
the same experiment, but not intended to be compared across different datasets or likelihood 
models.
\subsection{Posterior Inference with SN-VI}
\label{sec:post_inf}
In this section, we design the experiments for 2D posterior distribution approximation and show the explicit visualization of approximation based on spline basis functions and $\widehat{\phi}(\bm x)$.

We generate data under four distinct scenarios, as summarized in Table~C.2 in the Appendix. For each case, 2,048 samples are simulated, and the models are trained using mini-batches of size 64 over 40 epochs. A two-layer multilayer perceptron (MLP) with 20 hidden units per layer is used to estimate the unknown parameters. For the spline basis functions, we set 
$\varrho_1 = \varrho_2 = 3$ and 
$H = 5$. The computational cost analysis is reported in Figure C.3.

We compare the proposed method with several baseline approaches: Gaussian-ADVI, which approximates the posterior using a Gaussian distribution under the mean-field assumption; S-ADVI, which employs spline-based posterior approximations with the mean-field assumption; and two flow-based methods, Planar Flow (Planar) and Neural Spline Flows (NSF). Both normalizing-flow-based methods use 10 flows.

To assess the performance of the proposed method and to compare it with existing approaches, we employ two evaluation metrics: the root integrated squared error (RISE), defined as
$
\mathrm{RISE} = \left[ \int \{ q(\boldsymbol{z}\mid \boldsymbol{x}) - p(\boldsymbol{z}\mid \boldsymbol{x}) \}^{2} \, d\boldsymbol{z} \right]^{1/2},
$
and the predictive log-likelihood (Pred) described in Section~\ref{sec:str}. Each experiment is repeated 20 times, and the mean and standard deviation of the results are reported in Table~\ref{tab:rise_rotated}. The predictive log-likelihood is computed using a test set of 100 samples generated from the same inference model as the training data, providing a measure of how well the trained model predicts new observations.

Table~\ref{tab:rise_rotated} demonstrates that SN-VI consistently outperforms the competing methods. Figure~\ref{sn-vi_fig:post_inf} further compares the true posterior distributions with the approximations obtained from Gaussian-ADVI, S-ADVI, Planar, NSF, and SN-VI. As seen in Figure~\ref{sn-vi_fig:post_inf}, Gaussian-ADVI and S-ADVI fail to capture the dependence structure among the latent variables due to the restrictive mean-field assumption underlying these approaches. Although Planar and NSF provide more flexible posterior approximations than MFVI-based methods, they still do not recover the true dependency patterns among latent variables. In particular, SN-VI successfully preserves the nonlinear dependence structure in Cases~2 and~3. In Cases~2 and~4, SN-VI also accurately captures posterior skewness and multimodality, whereas competing methods do not. Overall, SN-VI yields substantially more accurate posterior approximations and is especially effective in settings where the latent variables exhibit dependence, reflecting its advantage in relaxing mean-field constraints.

In addition, Figure C.3 in the Appendix shows the computational cost analysis, indicating that SN-VI achieves faster convergence than the compared methods and remains robust to the choice of knot number. Figure C.2 in the Appendix demonstrates how the roughness penalty parameter $\lambda$ in SN-VI will affect the posterior approximation of Case 4. When $\lambda$ is relatively large, it encourages the approximated surface to be smoother with less fluctuation. Conversely, when 
As $\lambda$ decreases, the approximated surface becomes more flexible, allowing it to capture the multimodality of the true posterior distribution. Thus, the roughness parameter $\lambda$ plays a crucial role in balancing bias and variance in the posterior approximation.

\begin{figure}[!h]
\captionsetup{width=1\textwidth}
\noindent
\begin{minipage}[t]{0.5\textwidth}
  \includegraphics[width=\linewidth, height=0.2cm]{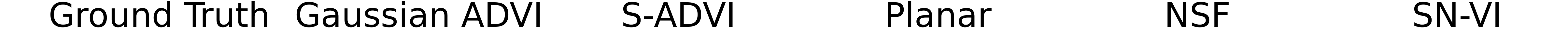}
\end{minipage}
\par\vspace{0.1em} 
\noindent
\begin{minipage}[t]{0.5\textwidth}
  \includegraphics[width=\linewidth, height=1.5cm]{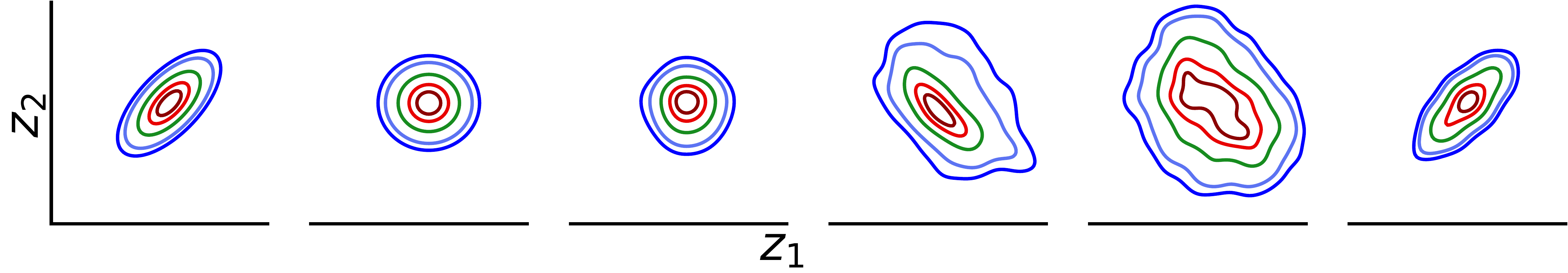}
\end{minipage}
\hfill
\begin{minipage}[t]{0.5\textwidth}
  \includegraphics[width=\linewidth, height=1.5cm]{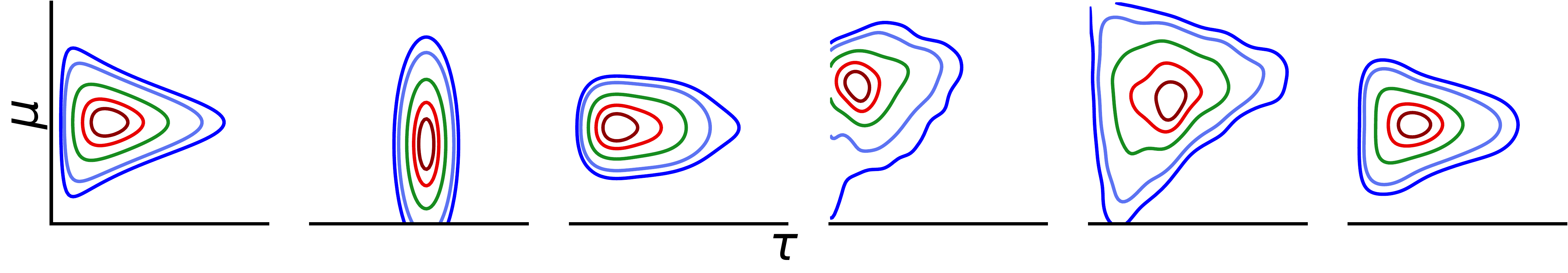}
\end{minipage}
\par\vspace{0.1em}
\noindent
\begin{minipage}[t]{0.5\textwidth}
  \includegraphics[width=\linewidth, height=1.5cm]{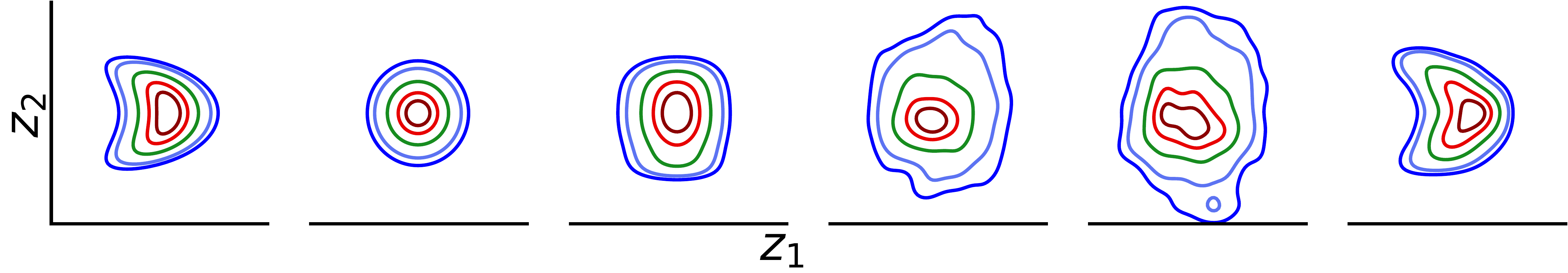}
\end{minipage}
\hfill
\begin{minipage}[t]{0.5\textwidth}
  \includegraphics[width=\linewidth, height=1.5cm]{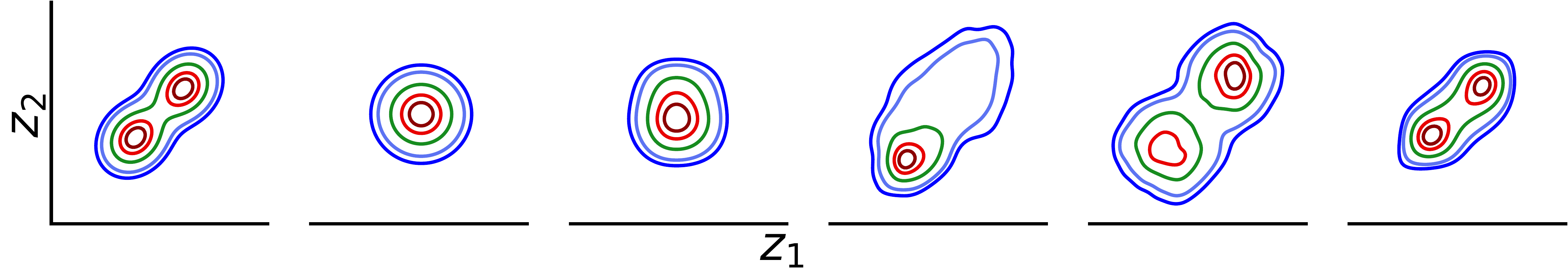}
\end{minipage}
\caption{Posterior approximation results for the four cases, shown row by row, comparing the proposed method with Gaussian-ADVI, S-ADVI, Planar flows, and NSF.}
\label{sn-vi_fig:post_inf}
\end{figure}

\begin{table}[h!]
\caption{Mean (standard deviation)  of RISE and Predictive log-likelihood (Pred) for Cases 1-4.}\label{tab:rise_rotated}
    \small
    \centering
    { 
    \addtolength{\tabcolsep}{-0.1em}
      \begin{tabular}{lcrrrrr}
    \hline\hline
        Metrics & Case  & Gaussian-ADVI  & S-ADVI & Planar & NSF & SN-VI \\ 
        \hline
    \multirow{4}{*}{RISE} 
    & 1 & 0.25 (0.01) & 0.26 (0.01)& 0.46 (0.01)& 0.49 (0.02) & \textbf{0.09} (0.02) \\
    & 2 & 0.42 (0.22) & 0.12 (0.01) &0.26 (0.03) & 0.28 (0.02) & \textbf{0.09} (0.01) \\
    & 3 & 0.11 (0.01) & 0.12 (0.02) &0.21 (0.01) & 0.21 (0.01) & \textbf{0.09} (0.02) \\
    & 4 & 0.26 (0.01) & 0.24 (0.01)& 0.17 (0.01)& 0.19 (0.01) & \textbf{0.13} (0.03)\\
    \hline
    \multirow{4}{*}{Pred} 
    & 1 & -194.79 (3.02)  & -195.75 (1.26)  & -310.50 (1.59) & -310.31 (2.06) &\textbf{-168.03} (1.09)  \\
    & 2 & -321.30 (0.20)  & -132.59 (1.15)  &-205.02 (3.17) &-204.52 (2.80) &\textbf{-130.44} (1.49)  \\
    & 3 & -151.22 (1.75)  & -147.09 (3.71)  & -206.37 (2.91) &-208.17 (3.19) &\textbf{-137.09} (2.45)  \\
    & 4 & -230.26 (2.14) & -234.15 (1.70) & -348.71 (1.48) &-349.46 (2.61) &\textbf{-224.09} (0.91) \\
    \hline
    \end{tabular}
    }
\end{table}

\subsection{Structure Identification}\label{sec:str_exp}
In Section \ref{sec:post_inf}, we assume that the structure of the latent variables is known during training. However, in practical applications, this structure is typically not available. In this section, we design simulation studies to demonstrate that the predictive log-likelihood can serve as an effective metric for SN-VI model selection in more complex Bayesian model settings. The proposed Algorithm 1 provides an effective approach for identifying the dependence structure. Our method is compared with MCMC samplers implemented in JAGS \citep{rjags}, which are considered to represent the true posterior distributions.  We apply the same neural network structures discussed in Section \ref{sec:post_inf}.



Let $\bm{\Sigma} = \big(\begin{smallmatrix}
1 & 0.9&0\\
0.9 & 1&0\\
0&0&1
\end{smallmatrix}\big).$ and $\bm{\mu} = (\mu_{0},\mu_{1},\mu_{2})^{\top}$. We consider the following two settings: 
\begin{itemize}
\item Experiment 1: $\mu_{0},\mu_{1},\mu_{2}$ are independently generated from $\mathcal{N}(0,0.5)$ and $\bm{y} \sim \mathcal{N}(\bm{\mu},\bm{\Sigma})$.
\item Experiment 2: $\mu_{0}, \mu_{1}, \mu_{2}$  are independently generated from LogNormal \((0.1,0.25)\) and $\bm{y}\sim\mathcal{N}(\bm{\mu},\bm{\Sigma})$. 
\end{itemize}
In both experiments, the posterior of \(\mu_0\) and \(\mu_1\) are dependent, while the posterior of \(\mu_2\) is independent of \(\mu_0\) and \(\mu_1\).

 The predictive log-likelihoods were compared across three possible latent-variable structures: fully independent, fully dependent, and mixed with two dependent latent variables and one independent latent variable. For Experiment 1, the corresponding predictive log-likelihoods were $-313.52$, $-303.42$, and $-300.43$, respectively. For Experiment 2, they were $-319.07$, $-310.41$, and $-309.24$, respectively. In both experiments, the mixed structure achieved the highest predictive log-likelihood, indicating that a partially dependent latent structure provides the best fit among the three candidates.

Figure \ref{fig:struc_2} presents the contour plots of the posteriors obtained from JAGS and from the SN-VI model selected for Experiment 2. The posterior estimated by SN-VI closely aligns with the posterior generated by JAGS. These findings indicate that the predictive log-likelihood serves as an effective criterion for model selection. Moreover, in both experiments, the fully dependent model produces predictive log-likelihoods only slightly lower than those of the true structure, suggesting that adopting a more general and complex model results in a modest deterioration in performance. However, the resulting performance decline is minor, indicating that the additional complexity does not severely compromise model quality.

\begin{figure}[!htbp]
\begin{tabular}{c c}
   \includegraphics[width=0.45\columnwidth,height = 4.3cm]{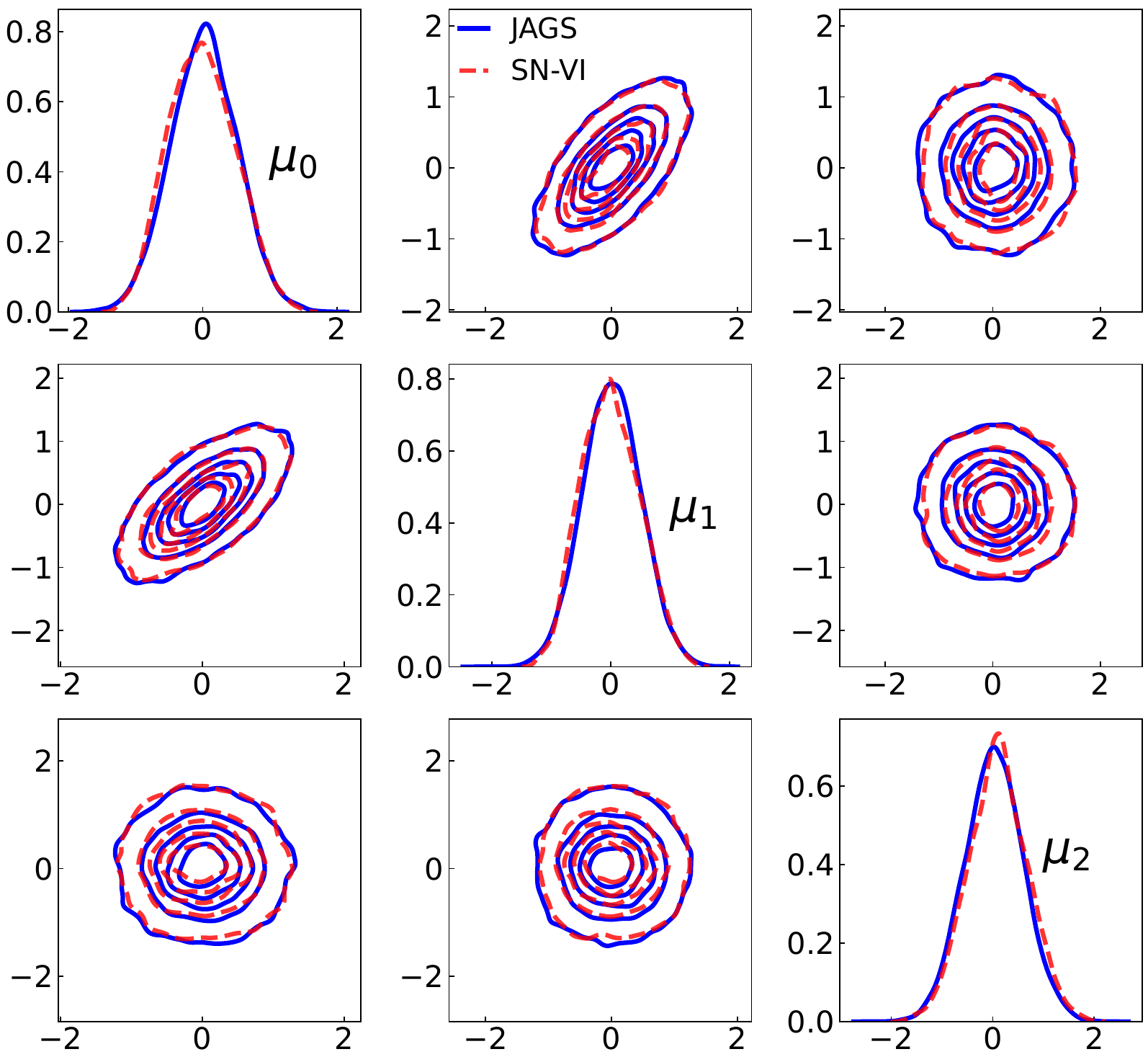}  & \includegraphics[width=0.45\columnwidth,height = 4.3cm]{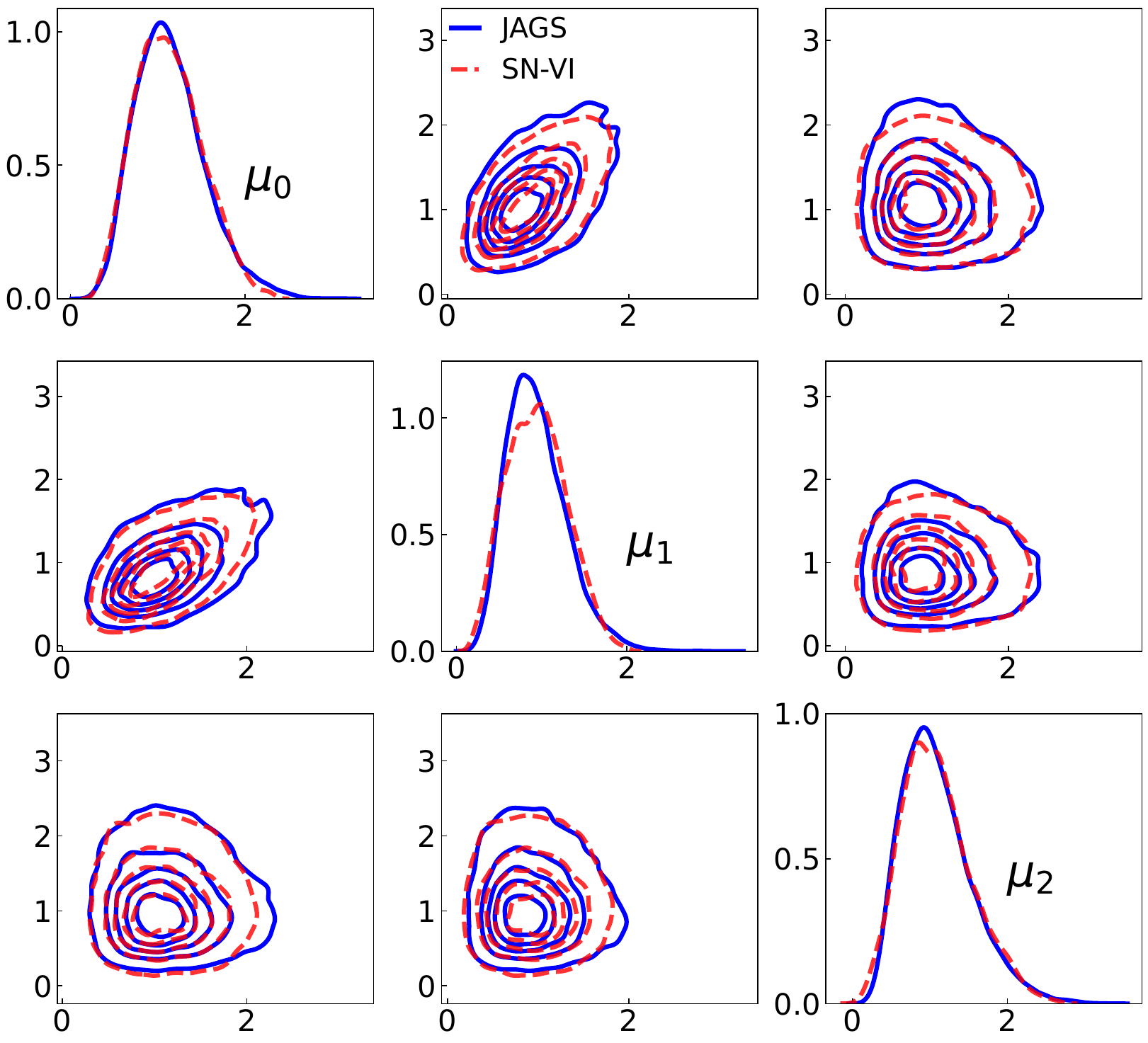} \\
   (a)  & (b)
\end{tabular}
    \caption{(a) Contour plots of MCMC samples (JAGS, ground truth) and the SN-VI estimated posterior for 
$p(\boldsymbol{\mu} \mid \boldsymbol{y} = (0.1, 0.1, 0.1)^{\top})$ under the selected group structure for Experiment~1. (b) Contour plots of MCMC samples (JAGS, ground truth) and the SN-VI estimated posterior for  $p(\bm{\mu} \mid \bm{y}=(0.8,0.5,1.0)^{\top})$ under the selected group structure for Experiment~2. }
    \label{fig:struc_2}
\end{figure}

\section{Real Data Applications}
\label{sec:application}

In this Section, SN-VI is evaluated on two real-data domains: spatial transcriptomics and computer vision.

\subsection{Spatial Transcriptomics} 

Recent advances in spatial transcriptomics (ST) have enabled high-resolution mapping of cellular niches and the investigation of intercellular communication that underlies tissue function \citep{codeluppi2018spatial, asp2017spatial}. 
Original ST data typically contain expression measurements for thousands of genes at each spatial location (spot), resulting in high-dimensional structured data characterized by inherent spatial dependencies and latent factors. 
Numerous methods have been proposed to reduce the dimensionality of ST data to facilitate cell-type or spatial domain identification. 
However, most existing approaches assume independent latent variables, which may obscure correlated gene-expression programs.  

In this section, we apply the proposed SN-VI method to the sagittal-anterior mouse brain ST dataset from 10X Genomics \citep{10x_mouse_brain_sagittal_anterior_2020}. 
SN-VI learns low-dimensional latent representations while preserving dependencies among latent variables, thereby capturing coupled biological signals. The sagittal-anterior mouse brain dataset includes 2,695 tissue spots, each with spatial coordinates 
$\mathbf{s}_i = (s_{i1}, s_{i2})$, and expression profiles for 32,285 genes. 
Following the quality-control recommendations in \citet{mccarthy2017scater}, we retain spots with more than 2,000 detected genes, total counts between 5,000 and 50,000, and mitochondrial content below 25\%. Genes expressed in fewer than 10 spots are excluded.  

To demonstrate the effectiveness of SN-VI, we model the expression count $x_{ig}$ for a specific gene $g$ at spot $i$ using a Negative Binomial (NB) likelihood:
$
x_{ig} \sim \text{NB}(\mu_{ig}, \theta_g),
$
where $\mu_{ig}$ denotes the mean expression level, and $\theta_g$ is a gene-specific inverse dispersion parameter shared across spots. 
Each spot’s expression vector $\bm x_i = \{x_{ig}\}_{g=1}^G$ serves as input, and SN-VI outputs samples from the posterior distribution $p(\bm z_i \mid \bm x_i)$. 
We parameterize $\mu_{ig}$ and $\theta_g$ using a two-layer multilayer perceptron (MLP) with 300--300 hidden units and LeakyReLU activation.

Figure~\ref{fig:st_latent} displays spatial maps of latent dimensions $z_1, \dots, z_8$ inferred by various variational inference methods on the sagittal-anterior mouse brain data. 
Each row corresponds to a different inference method (Gaussian-ADVI, S-ADVI, Planar, NSF, and SN-VI), and each column shows a specific latent dimension. 
Distinct spatial patterns highlight the differing abilities of these methods to capture structured dependencies in the underlying tissue. 
SN-VI with identified structure $\{(z_1,z_2),z_3,z_4,z_5,z_6,z_7,z_8\}$ achieves the lowest negative log-likelihood (821.20 $\pm$ 0.20), outperforming S-ADVI (822.42 $\pm$ 0.10) and Gaussian-ADVI (828.17 $\pm$ 0.41). 
Flow-based approaches such as Planar (862.64 $\pm$ 2.78) and NSF (861.10 $\pm$ 1.38) yield higher values.


We interpret the learned latent variables as low-dimensional representations that summarize the posterior distribution $p(\bm z \mid \bm x)$. As shown in Figure~\ref{fig:st_latent}, SN-VI produces spatially coherent latent maps that align well with major anatomical regions, indicating that the method preserves important spatial organization in the tissue. Joint posterior plots in Figures~\ref{fig:spot_post}(c) -- (d) further illustrate localized dependence between $(z_1, z_2)$ at two representative spots, consistent with the sparse SN-VI structure $\{(z_1,z_2),z_3,z_4,z_5,z_6,z_7,z_8\}$. In addition, the marginal distribution of $z_1$ is sometimes multimodal, suggesting heterogeneous local latent states rather than a single homogeneous regime.


Furthermore, we summarize the dependence  between latent variables $(z_1, z_2)$ by looking at the empirical posterior threshold probability $\widehat{\Pr}(|\rho(z_1,z_2)|>\delta\mid x_i) = \frac{1}{B}\sum_{b=1}^B \mathbf 1(|\hat\rho_i^b|>\delta)$ where $\hat\rho_i  = \mathrm{corr}(z_{i1}^{1:S}, z_{i2}^{1:S})$ is posterior correlation computed from  $S=1000$ draws. There are $B = 200$ bootstrap draws. We choose the threshold $\delta=0.1$ a priori as a small but non-negligible effect-size benchmark, following the ROPE-style practical-significance perspective in Bayesian reporting \citep{kruschke2018rejecting}. 
As shown Figure~\ref{fig:spot_post}(e),
darker colors indicate stronger evidence that the local posterior dependence magnitude exceeds the chosen threshold.
The regions with higher posterior dependence probability between $z_1$ and $z_2$ are concentrated in areas spanning the olfactory bulb, cerebral cortex, and hypothalamus, consistent with prior studies showing strong region-specific spatial transcriptomic organization in these mouse brain areas. 

To provide external biological context for this latent-dependence map, we  performed a marker-gene enrichment analysis using the BRETIGEA mouse-brain marker sets \citep{mckenzie2018brain}. Specifically, we ranked all measured genes by Spearman correlation between normalized log-expression and the spot-level dependence score $\widehat{\Pr}(|\rho(z_1,z_2)|>0.1\mid x_i)$. We then selected the top 500 positively associated genes and tested each BRETIGEA cell-class marker set for over-representation using a one-sided Fisher exact test, with Benjamini--Hochberg correction across marker sets. As shown in Figure~\ref{fig:spot_post}(f), the dependence-associated genes are enriched for astrocyte markers (74 overlapping genes, odds ratio 4.06, FDR $q=1.57\times 10^{-19}$) and neuron markers (76 overlapping genes, odds ratio 3.47, FDR $q=6.92\times 10^{-17}$). This enrichment supports the biological plausibility of the spatially localized latent-dependence pattern captured by SN-VI. This finding is consistent with external mouse-brain transcriptomic atlases showing that neuronal and astrocytic marker programs are major axes of molecular variation in mouse nervous system \citep{mckenzie2018brain,zeisel2018molecular}.

More importantly, to assess their practical usefulness more directly, we evaluate if the latent representation from each method can recover the structure of the fixed expression-based reference partition under a common graph-based Leiden pipeline. To achieve that, the reference partition is constructed with Scanpy's \texttt{pp.pca}, \texttt{pp.neighbors}, and \texttt{tl.leiden} routines \citep{wolf2018scanpy,traag2019leiden} by running a k-nearest-neighbor graph (kNN, $k=15$) on the 50 principal components computed from the expression matrix, followed by Leiden clustering with resolution $0.2$ on the kNN graph, which yields seven clusters. 
Then, for SN-VI and all competing models, we compute the posterior-mean embeddings using $50$ posterior draws for each spot, standardize embeddings within each latent dimension, and cluster the standardized embeddings directly with Scanpy's \texttt{pp.neighbors} and \texttt{tl.leiden} routines.
For each method, Leiden resolution was selected from a pre-specified grid to match the reference cluster count as closely as possible. We then report the Adjusted Rand Index (ARI) and Normalized Mutual Information (NMI), summarized over 20 random seeds.  ARI \citep{hubert1985comparing} quantifies pairwise cluster agreement after adjusting for chance, whereas NMI \citep{vinh2010information} measures the amount of shared clustering information between the predicted and reference partitions on a normalized $[0,1]$ scale. Higher ARIs and NMIs indicate better clustering alignment. To account for the stability of the clustering results, we repeated 20 Leiden clustering runs, each with a different random seed for initialization.

{Table~\ref{tab:st_all_method_clustering} compares SN-VI with all competing methods, where SN-VI remains strongest and improves on the second best model, S-ADVI, by 0.185 in ARI and 0.095 in NMI on average.
Figure~\ref{fig:st_cluster_spatial} provides spatial clustering comparison of all methods with reference partition, where SN-VI shows best recovery of reference partition.
Taken together, these results support a measured conclusion that SN-VI can produce useful spatial embeddings for downstream clustering, and that the downstream recovery of spatial domains depends materially on the chosen latent dependence structure.

\begin{figure}[!ht]
    \centering
    \includegraphics[width=0.98\linewidth]{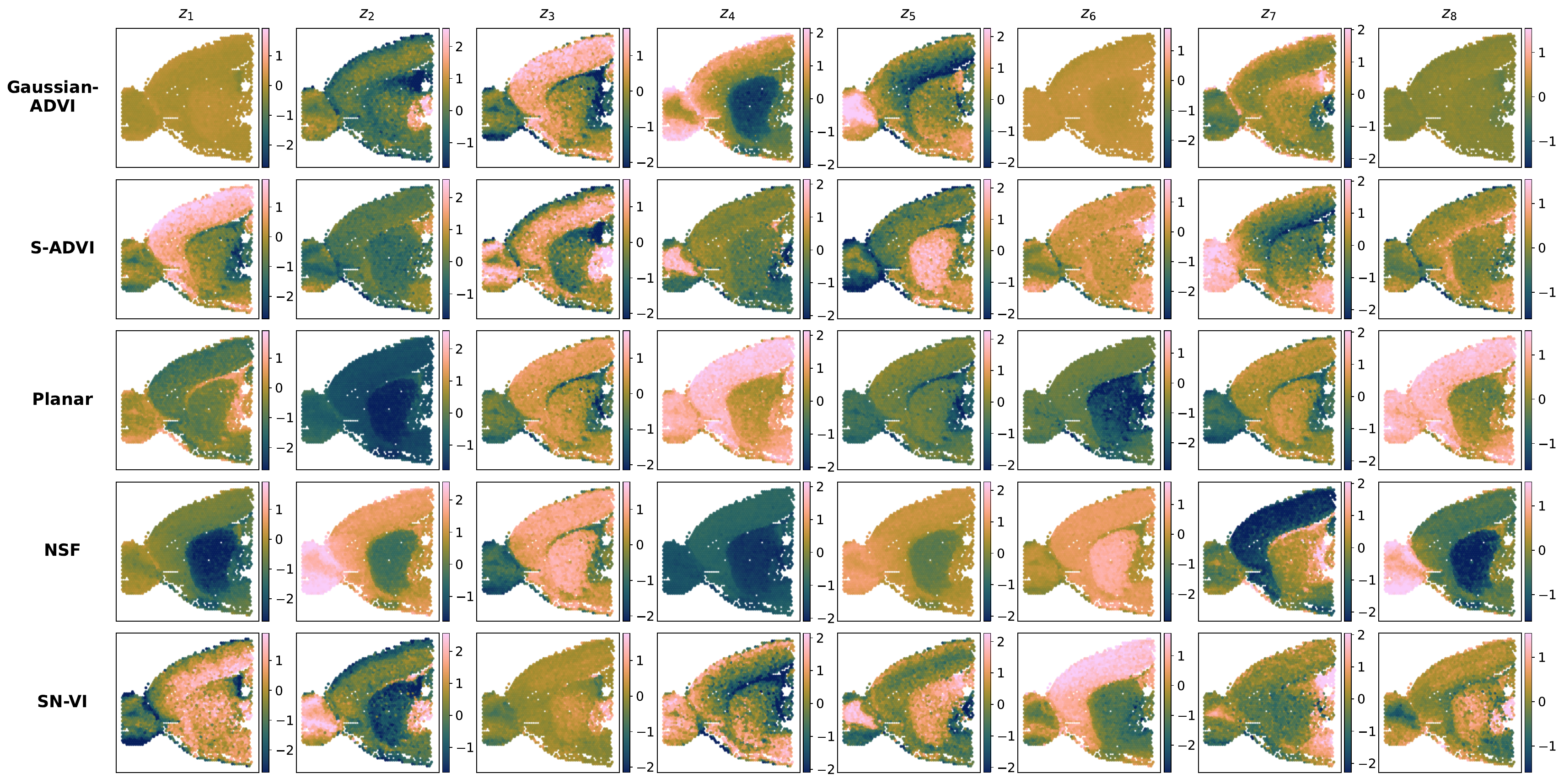}
    \caption{Spatial representation of latent variables $z_1$ to $z_8$ inferred by different variational inference methods on sagittal-anterior mouse brain data. The SN-VI is based on selected group: $\{(z_1,z_2),z_3,z_4,z_5,z_6,z_7,z_8\}$. The magnitude represents the posterior mean of latent variables on each spot using 1,000 posterior draws. Light-colored (dark-colored) represents higher (lower) value of latent variables' posterior means.}
    \label{fig:st_latent}
\end{figure}


\begin{figure}[!htbp]
\centering

\begin{subfigure}[t]{0.3\textwidth}
  \centering
\includegraphics[width=\textwidth]{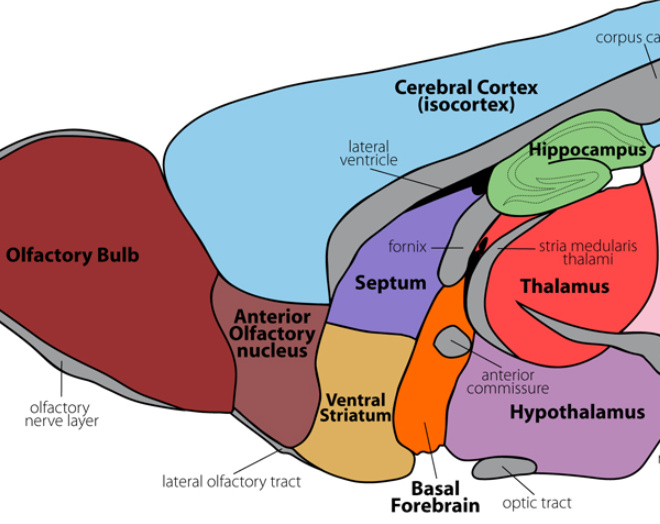}
  \caption{Anatomical Structure}
\end{subfigure} \quad\quad
\begin{subfigure}[t]{0.3\textwidth}
  \centering
 \includegraphics[width=\textwidth]{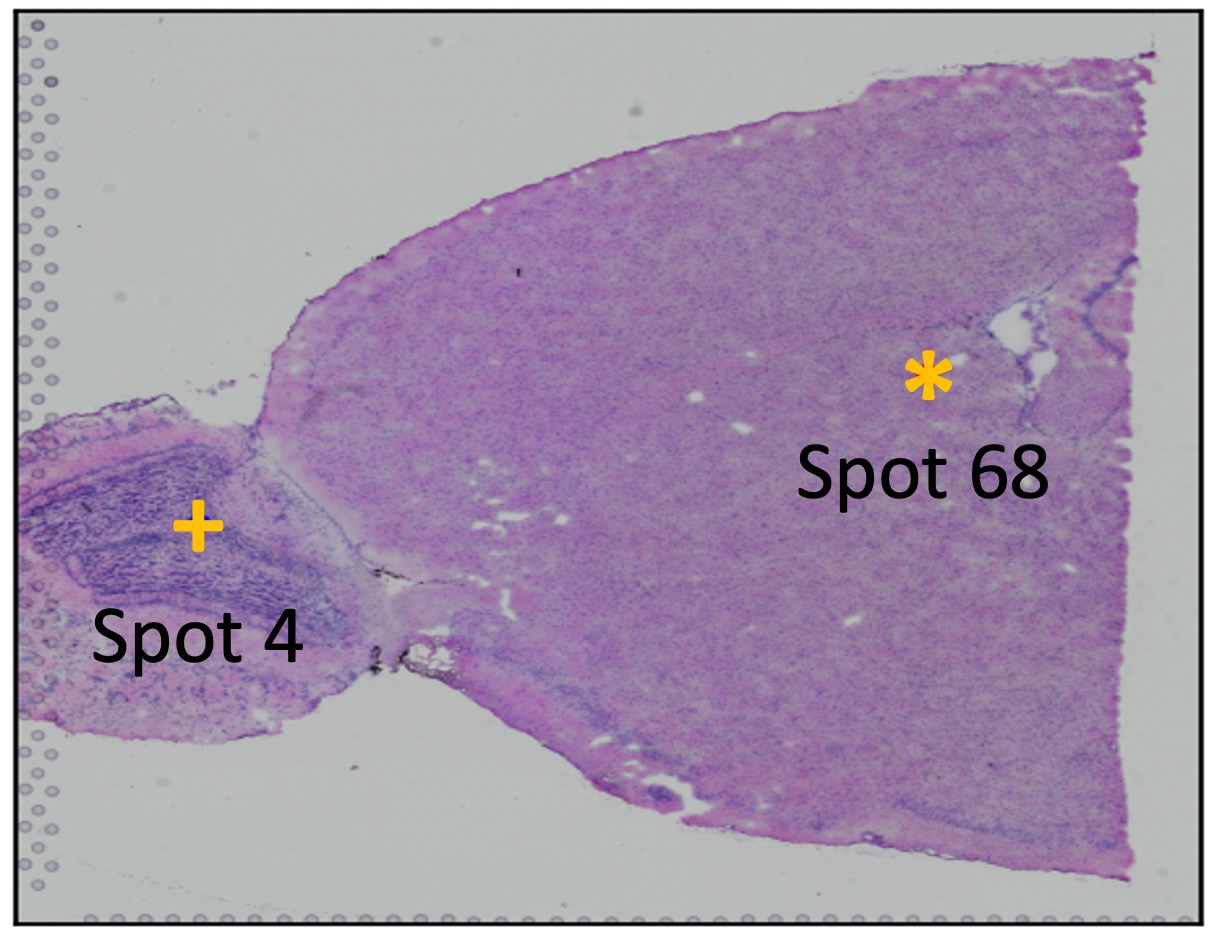}
  \caption{H\&E stained image}
\end{subfigure}
\begin{subfigure}[t]{0.32\textwidth}
  \centering
\includegraphics[width=\textwidth]{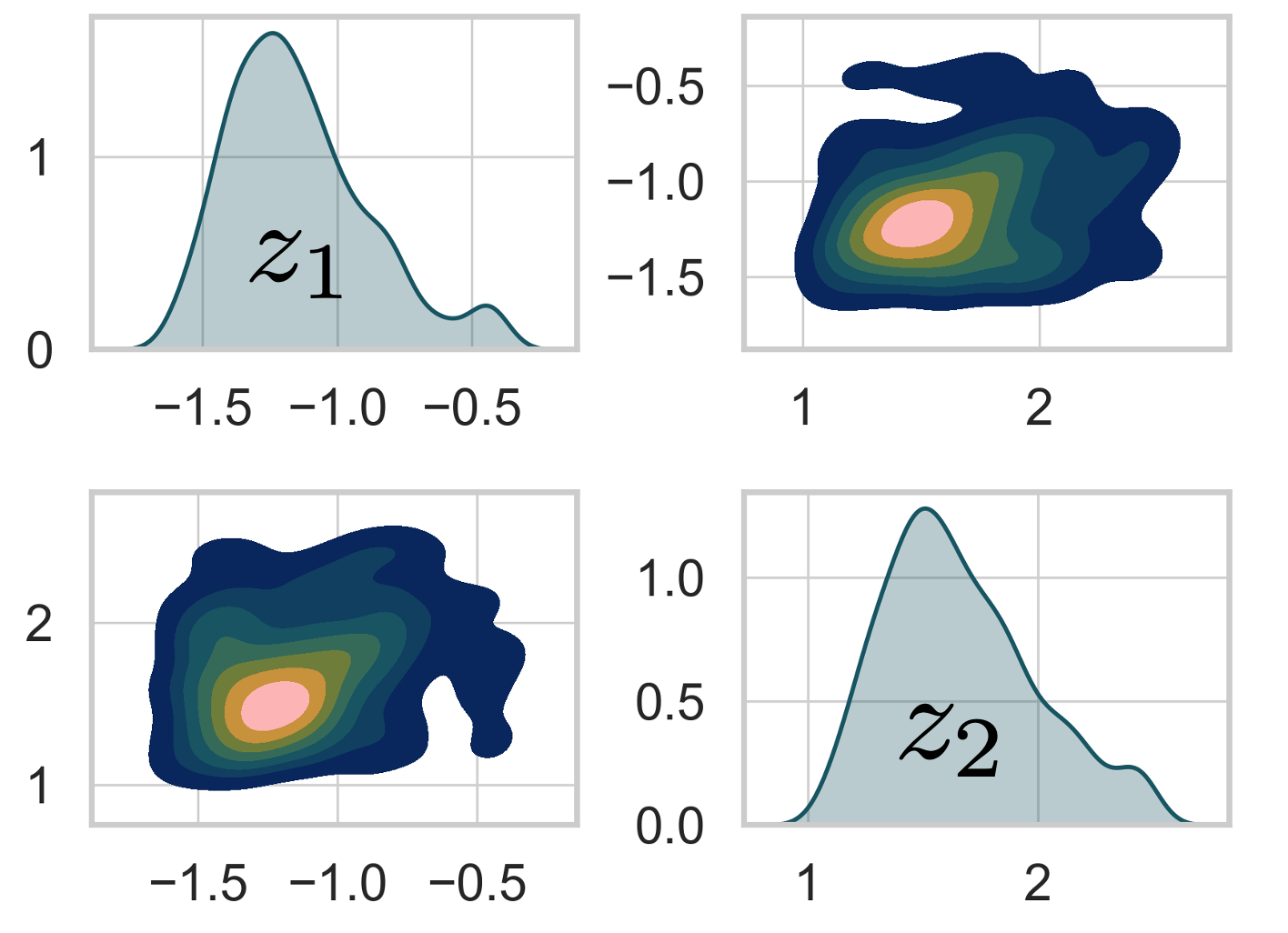}
  \caption{Posterior dist. at Spot 4}
\end{subfigure}
\begin{subfigure}[t]{0.32\textwidth}
  \centering
\includegraphics[width=\textwidth]{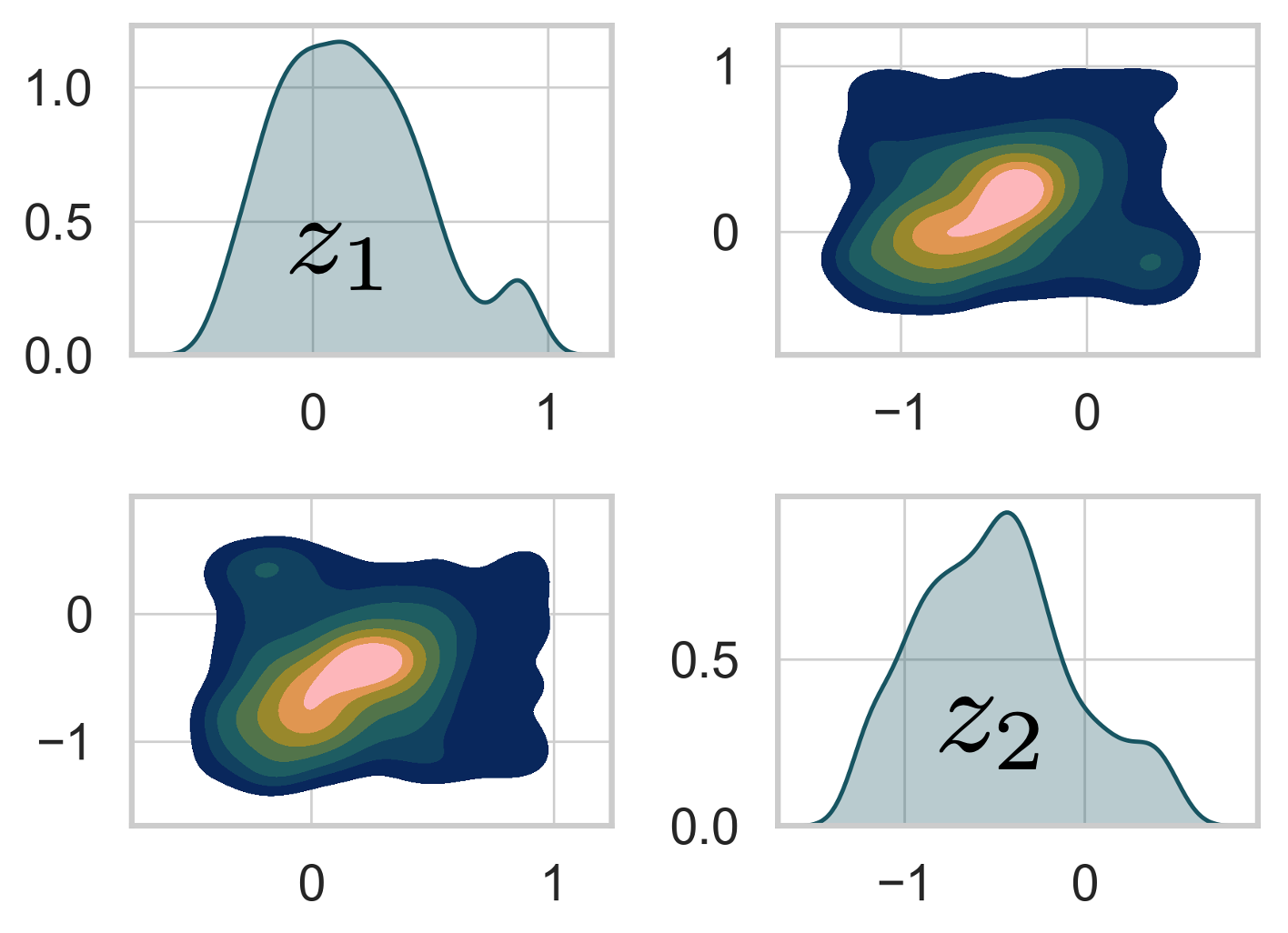}
  \caption{Posterior dist. at Spot 68}
\end{subfigure}
\begin{subfigure}[t]{0.26\textwidth}
  \centering
\includegraphics[width=\textwidth]{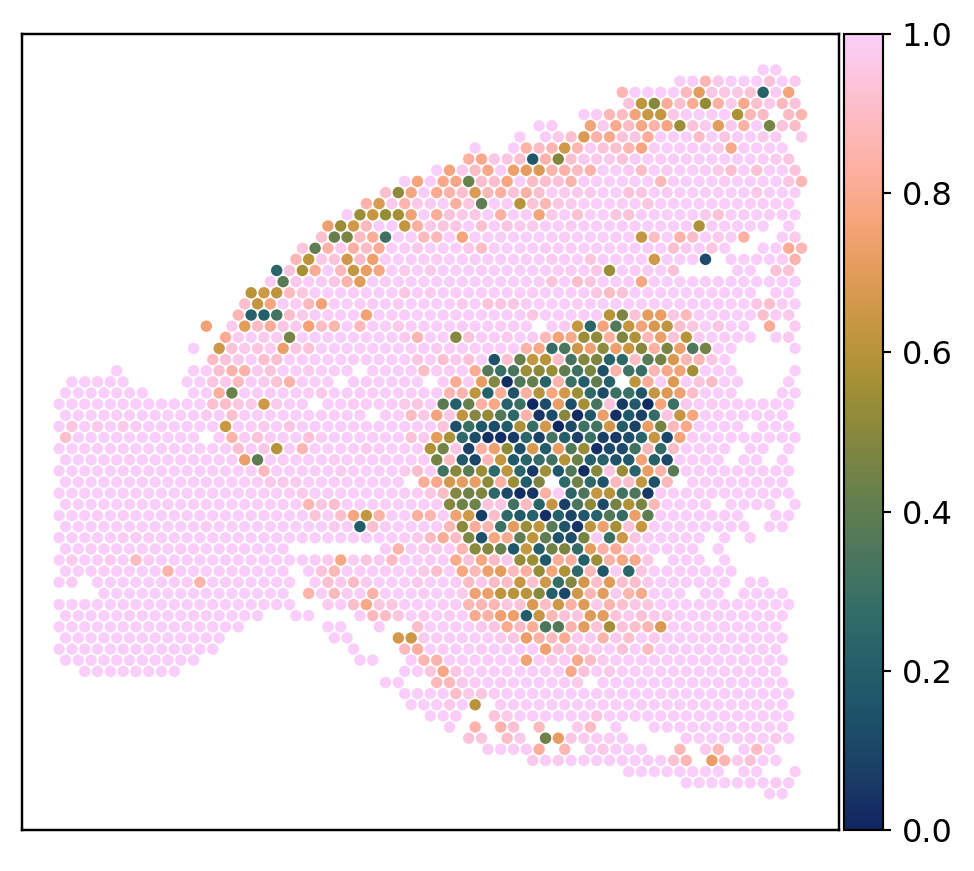}
  \caption{Posterior summary}
\end{subfigure}
\begin{subfigure}[t]{0.4\textwidth}
  \centering
    \includegraphics[width=\linewidth]{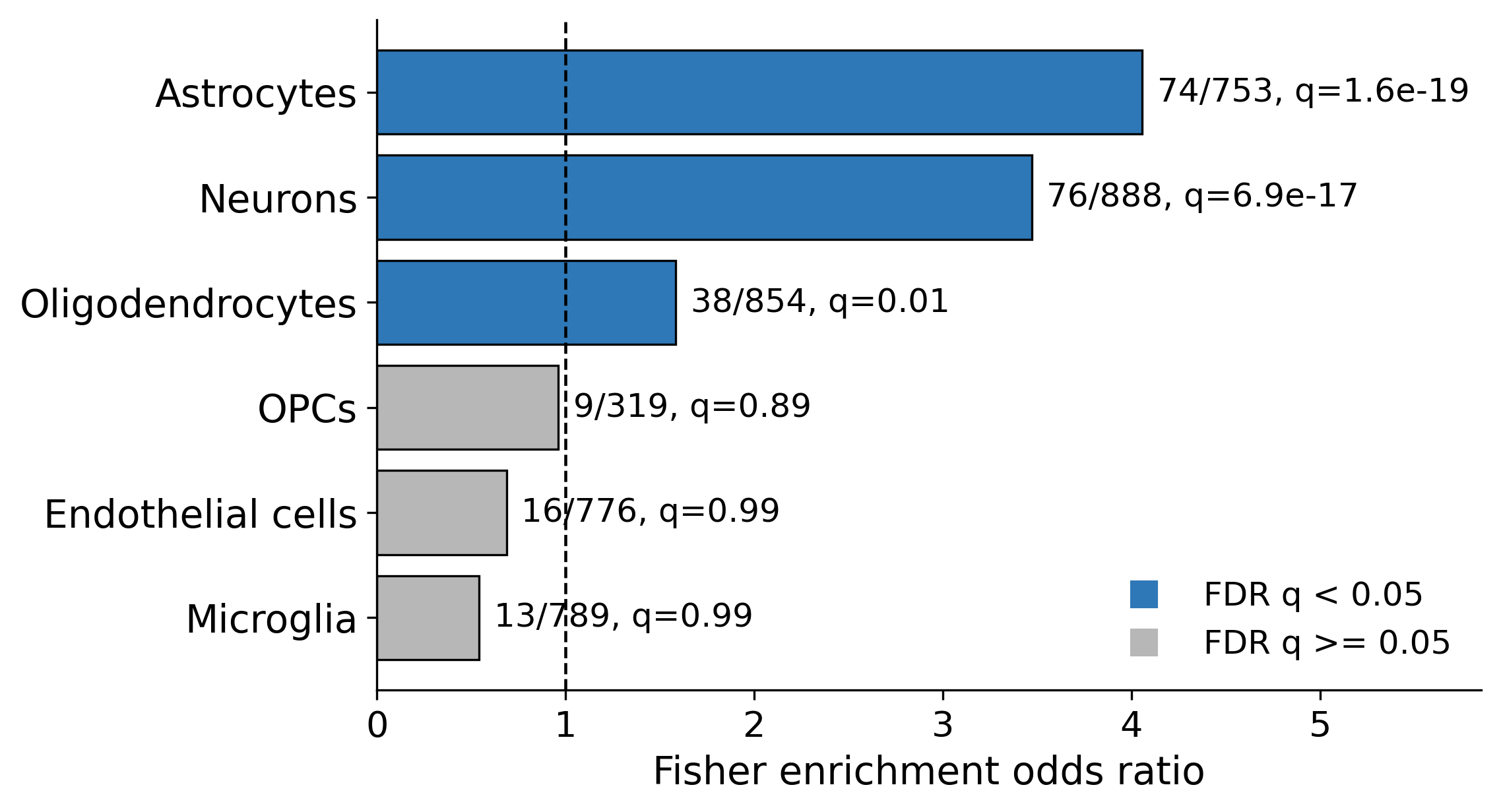}
\caption{Enrichment analysis}
\end{subfigure}
\caption{Sagittal-anterior mouse brain data application results. (a)  Annotated anatomical structure \citep{gensat_npy_image1214}; (b) Hematoxylin and eosin (H\&E) stained image. The randomly sampled spatial locations spot 4 and spot 68 are marked as ``+'' and ``*'', respectively, which are used to approximate posterior distributions shown in (c) and (d); (c) Posterior distributions of latent variables $(z_1, z_2)$ at spot 4; (d) Posterior distributions of $(z_1, z_2)$ at spot 68; (e) Spatial maps of the posterior threshold summary $\widehat{\Pr}(|\rho(z_1,z_2)|>0.1\mid x_i)$, using 1{,}000 posterior draws per spot, and 200 bootstrap resamples; (f) BRETIGEA marker-gene enrichment analysis for genes associated with the SN-VI posterior dependence score $\widehat{\Pr}(|\rho(z_1,z_2)|>0.1\mid x_i)$. Genes are ranked by Spearman correlation between normalized log-expression and the dependence score, and the top 500 positively associated genes are tested for over-representation in each BRETIGEA mouse-brain marker set using a one-sided Fisher exact test with Benjamini--Hochberg correction.}
    \label{fig:spot_post}
\end{figure}



\begin{table}[!htbp]
\centering
\caption{Cross-method downstream clustering comparison on the filtered sagittal-anterior mouse brain ST data. 
Values are mean $\pm$ standard deviation over twenty random seeds.}
\label{tab:st_all_method_clustering}
{
\begin{tabular}{lccccc}
\hline\hline
Metric & SN-VI & S-ADVI & NSF & Gaussian-ADVI & Planar \\
\hline
ARI & $0.768 \pm 0.037$ & $0.583 \pm 0.039$ & $0.572 \pm 0.023$ & $0.590 \pm 0.037$ & $0.447 \pm 0.033$ \\
NMI & $0.774 \pm 0.023$ & $0.679 \pm 0.017$ & $0.660 \pm 0.020$ & $0.641 \pm 0.012$ & $0.560 \pm 0.010$ \\
\hline
\end{tabular}
}
\end{table}

\begin{figure}[!htbp]
\centering
\includegraphics[width=0.6\linewidth]{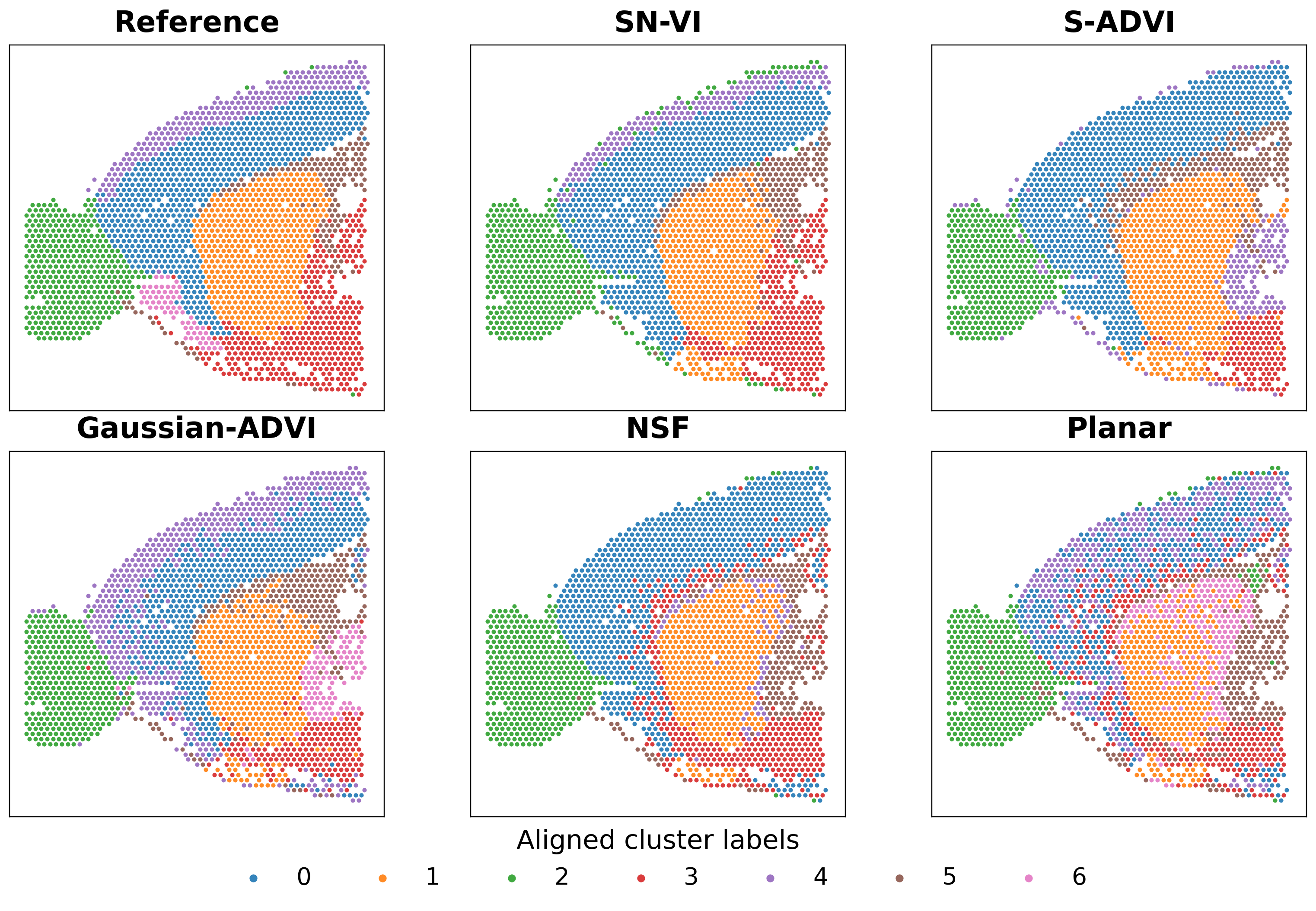}
\caption{Spatial clustering comparison on the filtered sagittal-anterior mouse brain ST data. The reference panel shows the fixed seven-cluster partition derived from the filtered expression matrix, and the remaining panels show representative Leiden clusterings from SN-VI, S-ADVI, Gaussian-ADVI, NSF, and Planar under the same posterior-mean-embedding pipeline. For each method, the displayed seed is the run whose ARI/NMI is closest to that method's 20-seed average. For display, predicted cluster labels are aligned to the reference labels so that colors are comparable across panels.}
\label{fig:st_cluster_spatial}
\end{figure}

\subsection{Single Column Classification} We conduct classification tasks with the computer vision datasets MNIST and Fashion-MNIST (FMNIST) to demonstrate the performance of SN-VI. To showcase flexible approximation and dependency preservation, the training data only includes the center column of the images as input. Figure C.5 shows a single-column sample. For each image, we use only the pixels within the red-dashed lines as training input.

We evaluate the performance of SN-VI  with $H = \{3,7\}$, $J =\{2,4,6,8\}$, and $T =10$. The optimal group structure is determined via Algorithm 1. For comparison, we combine VIB with Gaussian-ADVI, S-ADVI, and normalizing flows model (Planar with 10 flows and Radial with 40 flows) \citep{rezende2015variational}. For S-ADVI, we set $H=3$ for all the independent latent variables. 100 samples are drawn from $q_{\phi}(\bm z| \bm x)$ and the mean of logits for the model output are calculated. The experiment is repeated 10 times, and the mean and standard deviation of the classification accuracy are recorded. Figure \ref{fig:real_perf} (a) presents the performance of SN-VI and other compared methods. Overall, SN-VI outperforms other methods in the FMNIST dataset. As $J$ increases, the performance of SN-VI is better since more dependent structures are preserved. Additional experiment results and implementation details are stated in Supplement Section C.4 with different combinations of $H$ and $J$ to demonstrate the influence of hyperparameters.

\begin{figure}[!h]
\centering
\captionsetup{width=1\textwidth}
\begin{minipage}{0.8\textwidth}
\centering
\includegraphics[width=\linewidth]{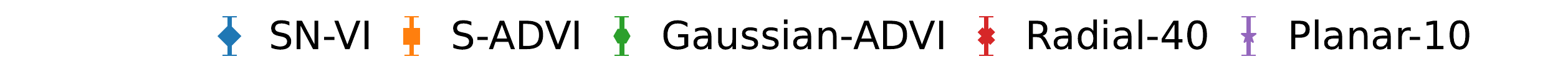}
    
\end{minipage}
\begin{subfigure}[t]{0.48\textwidth}
  \centering
  \includegraphics[width=\linewidth]{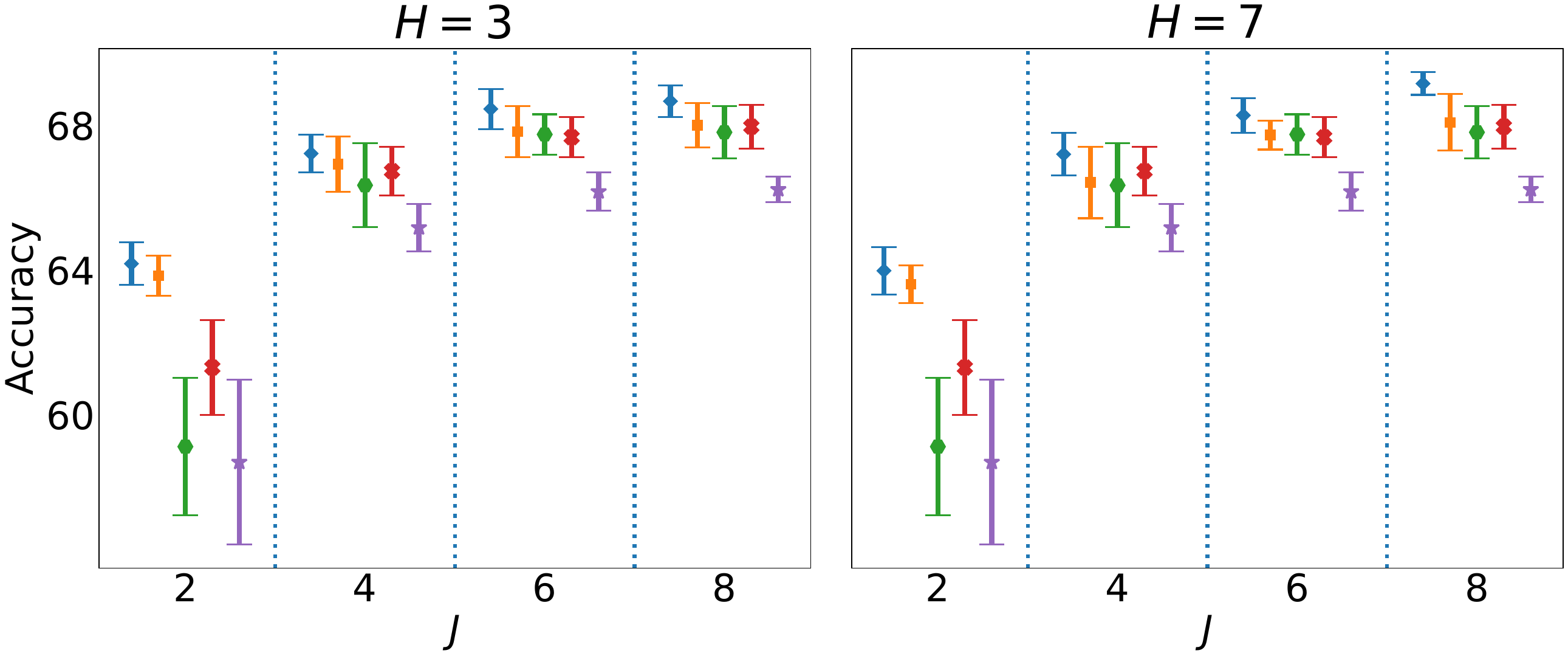}
  \caption{Classification performance (FMNIST).}
  \label{fig:clf_fmnist_sub}
\end{subfigure}
\hfill
\begin{subfigure}[t]{0.48\textwidth}
  \centering
  \includegraphics[width=\linewidth]{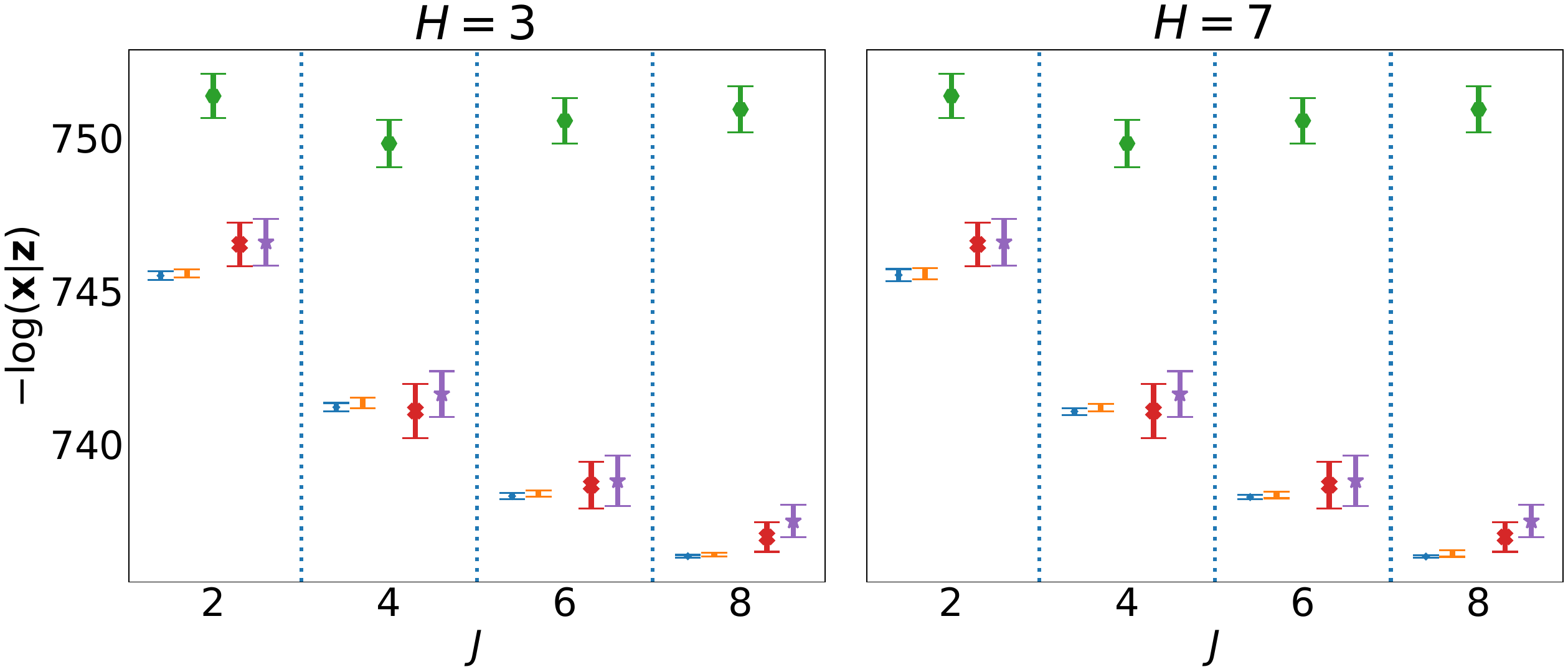}
  \caption{Reconstruction performance (KMNIST).}
  \label{fig:recon_kmnist}
\end{subfigure}

\vspace{-0.5em}

\caption{Performance comparison with varying number of latent variables $J$ and interior knots $H$ on FMNIST and KMNIST. The full numerical comparison results can be found in Table {C.4} and Table {C.5} in the supplementary materials.}
\label{fig:real_perf}
\end{figure}

\subsection{Image Reconstruction}
\label{sec:recon}
In this section, we demonstrate the capability of SN-VI in reconstructing images from the KMNIST dataset \citep{clanuwat2018deep}, comprised of a training set of 60,000 examples and a testing set of 10,000 examples of handwritten Kuzushiji (cursive Japanese) Hiragana characters. To increase the reconstruction difficulty and further evaluate the robustness of SN-VI, we corrupt the observed images with additive Gaussian noise. Specifically, the noise $\boldsymbol{\xi}$ is generated from $N(\mathbf{0}, \sigma_{\xi}^{2}\mathbf{I})$, and the corrupted observations are defined as $\bm{x}_{\text{corr}}=\bm{x}+\boldsymbol{\xi}$. The corruption level is controlled by the magnitude of $\sigma_{\xi}$.

Using the same neural network architecture as in Supplement Section C.4, we replace the decoder with a two-layer MLP with 512 and 256 hidden units. We then train SN-VI, S-ADVI, Gaussian-ADVI, Radial-40, and Planar-10 with $\sigma_{\xi}=0.25$, repeating each experiment 10 times. We report the mean and standard deviation of the negative log-likelihood $-\log p(\bm{x}\mid \bm{z})$, where we assume $p(\bm{x}\mid \bm{z})$ follows a Gaussian distribution.

Figure \ref{fig:real_perf} (b) summarizes the reconstruction performance across all methods. Gaussian-ADVI exhibits relatively limited performance due to its mean-field and Gaussian approximation assumptions. By contrast, SN-VI achieves better reconstruction accuracy than the normalizing-flow baselines, benefiting from its ability to preserve latent dependency structures and to flexibly approximate complex posterior distributions.

To further illustrate qualitative differences, Figure \ref{fig:kmnist_recon_eg} visualizes representative reconstructions from the test set. All methods can reconstruct relatively simple noisy images. However, for more complex characters, SN-VI and S-ADVI produce clearer reconstructions than the other methods. Although the reconstructions from SN-VI and S-ADVI are often similar, the ninth example, where the image has a more intricate structure, shows that SN-VI performs better and highlights the advantage of SN-VI over S-ADVI.

\begin{figure}[!htbp]
\centering\includegraphics[width=0.6\columnwidth]{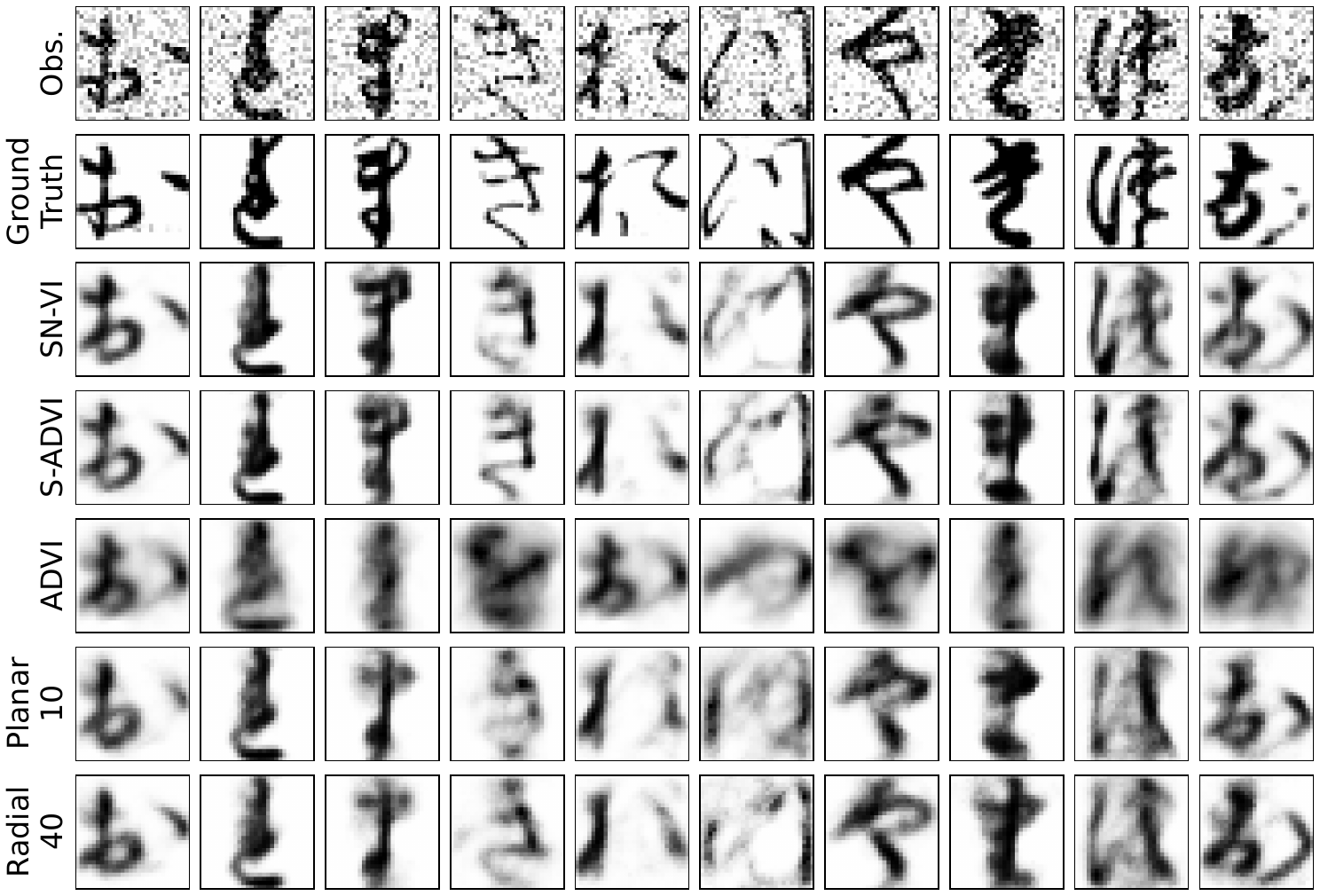}
    \caption{KMNIST image reconstruction examples from the test dataset.}
    \label{fig:kmnist_recon_eg}
\end{figure}

\section{Discussion}
We introduce a flexible nonparametric variational family by incorporating multivariate spline approximation to approximate variational distributions with complex shapes while preserving the dependency structure among latent variables. Leveraging the approximation power of multivariate splines, we establish the posterior estimation consistency of SN-VI. SN-VI not only achieves accuracy comparable to MCMC across various inference distributions but also improves the performance of generative models.


First, the proposed method is particularly suitable for posterior dependence structures that are sparse, local, or effectively low-rank. Although the tensor-product spline basis is flexible enough to approximate general multivariate dependence within each group, its computational cost grows exponentially with the group size. Therefore, for dependence structures involving a high degree of interactions among many latent variables, directly fitting a full high-dimensional tensor-product spline may become computationally prohibitive. In such cases, additional structural approximations, such as basis subsampling, sparse tensor-product bases, low-rank spline decompositions, or hierarchical/tree-structured variational families, may be needed. This observation is also consistent with the design principle of SN-VI: rather than modeling all latent variables jointly, SN-VI aims to identify relatively small dependent groups and model the joint dependence within each group. This allows the method to capture important non-mean-field posterior dependence while keeping the dimension of each spline approximation moderate.

Second, although this paper focuses on a group-based dependency structure, the proposed SN-VI framework is not restricted to this particular form of dependence. The key idea of SN-VI is to approximate low-dimensional components of the posterior distribution using flexible spline-based variational factors. Therefore, the same framework can be naturally extended to other structured dependency representations, such as, a tree-structured posterior approximation, where the latent variables or latent groups are organized according to a hierarchical tree and the posterior is factorized through parent-child conditional relationships. Such a formulation would allow dependence to propagate across groups and would relax the independence assumption between different groups in the current implementation. This type of tree structure may be particularly useful for time-series and spatial data, where latent variables often exhibit ordered, multiscale, or neighborhood-based dependence patterns.
\section{Significance Statement}

Modern scientific applications, including spatial transcriptomics, single-cell genomics, biomedical imaging, and AI-based image analysis, increasingly rely on latent-variable models to uncover hidden structure in high-dimensional data. In these problems, latent variables often exhibit nonlinear and scientifically meaningful dependencies. However, commonly used variational inference methods may either ignore such dependencies through mean-field assumptions or use highly flexible transformations that are difficult to interpret. 

We propose structured nonparametric variational inference (SN-VI), a flexible framework that uses multivariate splines to approximate complex posterior distributions while preserving structured latent dependence. SN-VI is especially well suited for applications where latent interactions carry scientific meaning, such as coordinated biological signals in spatial transcriptomics, correlated cellular states in single-cell and multi-omics studies, coupled pathological patterns in biomedical imaging, and robust image classification or reconstruction. Our results show that SN-VI can uncover richer latent structures and improve empirical performance over existing variational methods, demonstrating how nonparametric statistical ideas can enhance the reliability, interpretability, and scientific usefulness of modern AI tools.

\bibliographystyle{imsart-nameyear}
\bibliography{reference}
\newpage
\begin{center}
\Large
SUPPLEMENT TO ``STRUCTURED NONPARAMETRIC VARIATIONAL INFERENCE FOR DEPENDENT LATENT MODELING''
\end{center}
\renewcommand{\thesection}{\Alph{section}}
\renewcommand{\thefigure}{\Alph{section}.\arabic{figure}}
\renewcommand{\thetable}{\Alph{section}.\arabic{table}}

\setcounter{section}{0}
\setcounter{figure}{0}
\setcounter{table}{0}

\section{Additional Lemmas and Proofs Details}
\label{SEC:spline}

\subsection{Proof of Theorem 4.2}
\label{SEC:THE:Approximation-Proof}
Lemma 4.1 shows that the proposed SN-VI allows us to quantify the Lower Bound of IWAE.
\begin{proof} [Proof of Lemma 4.1]
We begin the proof with one group structure $|\mathcal{G}_{m}|$ with finite support $\mathcal{T}^{|\mathcal{G}_{m}|}$. According to Remark 2.2, for any posterior $p(\widetilde{\bm z}_{m}|\bm x) \in \mathcal{H}^{(\varrho)}(\mathcal{T}^{|\mathcal{G}_{m}|})$, there exists a spline function $s^{\ast}( \widetilde{\bm z}_{m}) = \sum_{\bm \kappa \in \bm \Kappa_{|\mathcal{G}_{m}|}} \gamma_{\bm \kappa} b_{\bm \kappa}(\widetilde{\bm{z}}_{m})$ such that $\sup_{\widetilde{\bm z}_{m} \in \mathcal{T}^{|\mathcal{G}_{m}|}}|p(\widetilde{\bm z}_{m}|\bm x) - s^{\ast}(\widetilde{\bm z}_{m})|\leq C_{1} H^{-(\varrho_m+1)}$. We denote the corresponding spline coefficients as $\phi^{\ast}$, and the optimal spline function as $\widehat{s}(\cdot)$  whose spline coefficients satisfy $\widehat{\phi} = \argmax   \mathcal{L}_{\text{IWAE}}(\phi)$. Therefore, we have $ \mathcal{L}_{\text{IWAE}}(\widehat{\phi})  \geq \mathcal{L}_{\text{IWAE}}(\phi^\ast).$ We denote $\Delta(\widetilde{\bm z}_{m}) = p(\widetilde{\bm z}_{m}|\bm x) - s^{\ast}(\widetilde{\bm z}_{m})$ and $\sup_{\widetilde{\bm z}_{m}\in \mathcal{T}^{|\mathcal{G}_{m}|}}|\Delta(\widetilde{\bm z}_{m})|\leq C_{1} H^{-(\varrho_m+1)}$. Then, we have
\begin{align}
\notag
    & \quad \mathcal{L}_{\text{IWAE}} (\widehat{\phi})  \geq  \mathcal{L}_{\text{IWAE}}(\phi^{\ast}) \\
\notag & = \int \left[ \log \frac{1}{T} \sum_{t=1}^T \frac{p\left(\bm x|\widetilde{\bm z}^{(t)}_{m}\right) p\left(\widetilde{\bm z}^{(t)}_{m}\right)}{p(\widetilde{\bm z}^{(t)}_{m}|\bm x) - \Delta(\widetilde{\bm z}^{(t)}_{m}|\bm x)} \right] \left\{\prod_{t=1}^T \left[p(\widetilde{\bm z}^{(t)}_{m}|\bm x) - \Delta(\widetilde{\bm z}^{(t)}_{m})\right]d\widetilde{\bm z}^{(t)}_{m} \right\} \\
    \notag
& = \int \left\{ \log \left[p(\bm x) + \frac{1}{T} \sum_{t=1}^T \frac{\Delta(\widetilde{\bm z}^{(t)}_{m}) p\left(\bm x|\widetilde{\bm z}^{(t)}_{m}\right) p\left(\widetilde{\bm z}^{(t)}_{m}\right)}{\left\{p\left(\widetilde{\bm z}^{(t)}_{m}\bm \right) - \Delta(\widetilde{\bm z}^{(t)}_{m})\right\} p(\widetilde{\bm z}^{(t)}_{m}|\bm x)} \right]\right\} \left\{\prod_{t=1}^T \left[p(\widetilde{\bm z}^{(t)}_{m}|\bm x) - \Delta(\widetilde{\bm z}^{(t)}_{m})\right]d\widetilde{\bm z}^{(t)}_{m} \right\} \\
\notag
    & \geq \int \left\{ \log \left[p(\bm x) - C_2H^{-(\varrho_m+1)} \right] \right\}\left\{\prod_{t=1}^T \left[p(\widetilde{\bm z}^{(t)}_{m}|\bm x) - C_3H^{-(\varrho_m+1)}\right]d\widetilde{\bm z}^{(t)}_{m} \right\} \\
     \label{EQU:S-ADVI-1}
    & \geq \log p(\bm x) - C_4 H^{-(\varrho_m+1)},
\end{align}
which yields the results in Lemma 4.1.

Next, we consider latent variables with infinite support. For a given $\epsilon$, consider a finite support $\mathcal{T}^{|\mathcal{G}_m|}$ such that $\int_{\mathcal{T}^{|\mathcal{G}_m|}}p\left(\widetilde{\bm z}_{m}
|\bm x\right) d\widetilde{\bm z}_{m}
 \geq 1-\epsilon$. The spline-based posterior has a finite support on $\mathcal{T}^{|\mathcal{G}_{m}|}$. According to the definition of IWAE and the fact that the IWAE has a tighter bound than ELBO, we have
\begin{align*}
    \mathcal{L}_{\text{IWAE}}&(\phi)  =  \mathbb{E}_{\left\{\widetilde{\bm z}^{(t)}_{m}  \sim q_\phi(\widetilde{\bm z}_{m} | \bm x)\right\}_{t=1}^T} \left[\log \frac{1}{T} \sum_{t=1}^T \frac{p\left(\bm x|\widetilde{\bm z}^{(t)}_{m} \right) p\left(\widetilde{\bm z}^{(t)}_{m} \right)}{q_\phi\left(\widetilde{\bm z}^{(t)}_{m} | \bm x\right)}\right] \\
    & =  \int_{\mathcal{T}^{|\mathcal{G}_{m}|}} \left[ \log \frac{1}{T} \sum_{t=1}^T \frac{p\left(\bm x|\widetilde{\bm z}^{(t)}_{m} \right) p\left(\widetilde{\bm z}^{(t)}_{m} \right)}{q_{\phi}(\widetilde{\bm z}^{(t)}_{m}| \bm x)} \right] \left[ \prod_{t=1}^T q_{\phi}(\widetilde{\bm z}^{(t)}_{m}| \bm x)\right] d\widetilde{\bm z}^{(1)}_{m} d\widetilde{\bm z}^{(2)}_{m} \ldots d\widetilde{\bm z}^{(T)}_{m} \\
    & \quad + \int_{\mathbb{R}^{|\mathcal{G}_{m}|}/ \mathcal{T}^{|\mathcal{G}_{m}|}} \left[ \log \frac{1}{T} \sum_{t=1}^T \frac{p\left(\bm x|\widetilde{\bm z}^{(t)}_{m} \right) p\left(\widetilde{\bm z}^{(t)}_{m} \right)}{q_{\phi}(\widetilde{\bm z}^{(t)}_{m}| \bm x)} \right] \left[ \prod_{t=1}^T q_{\phi}(\widetilde{\bm z}^{(t)}_{m}| \bm x)\right] d\widetilde{\bm z}^{(1)}_{m} d\widetilde{\bm z}^{(2)}_{m} \ldots d\widetilde{\bm z}^{(T)}_{m}\\
    & \geq \int_{\mathcal{T}^{|\mathcal{G}_{m}|
}} \left[ \log \frac{p\left(\bm x|\widetilde{\bm z}_{m}\right) p\left(\widetilde{\bm z}_{m}\right)}{q_{\phi}(\widetilde{\bm z}_{m}| \bm x)} \right] q_{\phi}(\widetilde{\bm z}_{m}| \bm x)  d\widetilde{\bm z}_{m} \\
& \quad + \int_{\mathbb{R}^{|\mathcal{G}_{m}|
}/\mathcal{T}^{|\mathcal{G}_{m}|
}} \left[ \log \frac{p\left(\bm x|\widetilde{\bm z}_{m}\right) p\left(\widetilde{\bm z}_{m}\right)}{q_{\phi}(\widetilde{\bm z}_{m}| \bm x)} \right]  q_{\phi}(\widetilde{\bm z}_{m}| \bm x) d\widetilde{\bm z}_{m}\\
&= \int_{\mathcal{T}^{|\mathcal{G}_{m}|
}} \left[ \log \frac{p\left(\bm x|\widetilde{\bm z}_{m}\right) p\left(\widetilde{\bm z}_{m}\right)}{q_{\phi}(\widetilde{\bm z}_{m}| \bm x)} \right] q_{\phi}(\widetilde{\bm z}_{m}| \bm x)  d\widetilde{\bm z}_{m}\\
    & = \log p(\bm x) + \int_{\mathcal{T}^{|\mathcal{G}_{m}|
}} \left[ \log \frac{p\left(\widetilde{\bm z}_{m}| \bm x\right)}{q_{\phi}(\widetilde{\bm z}_{m}| \bm x)} \right]  q_{\phi}(\widetilde{\bm z}_{m}| \bm x)d\widetilde{\bm z}_{m} \\
&\geq \log p(\bm x) + \int_{\mathcal{T}^{|\mathcal{G}_{m}|}} \left[ \log \frac{p^{\textrm{T}}\left(\widetilde{\bm z}_{m}|\bm x\right)}{q_{\phi}(\widetilde{\bm z}_{m}| \bm x)} + \log (1-\epsilon) \right]  q_{\phi}(\widetilde{\bm z}_{m}| \bm x) d\widetilde{\bm z}_{m},
\end{align*}
where $p^{\textrm{T}}\left(
\widetilde{\bm z}_{m}|\bm x\right) = (\int_{\mathcal{T}^{|\mathcal{G}_{m}|}}p\left(
\widetilde{\bm z}_{m}|\bm x\right) d\widetilde{\bm z}_{m}
)^{-1}p\left(\widetilde{\bm z}_{m}
|\bm x\right) $ is the density function for random variable $\widetilde{\bm z}_{m}
|\bm x$ truncated on the interval $\mathcal{T}^{|\mathcal{G}_{m}|
}$. Notice that when $\mathcal{T}^{|\mathcal{G}_{m}|}
$ is properly chosen and we can apply the conclusion from (\ref{EQU:S-ADVI-1}).
\[
\log p(\bm x) + \int_{\mathcal{T}^{|\mathcal{G}_{m}|}} \left[ \log \frac{p^{\textrm{T}}\left(\widetilde{\bm z}_{m}|\bm x\right)}{q_{\phi}(\widetilde{\bm z}_{m}| \bm x)} + \log (1-\epsilon) \right]  q_{\phi}(\widetilde{\bm z}_{m}| \bm x) d\widetilde{\bm z}_{m} \geq \log p(\bm x) - C_5 H^{-(\varrho_m+1)} - C_5\epsilon,
\]
which yields the results in Lemma 4.1.

It is straightforward to extend the above results to the multivariate case. If Assumption (A2) holds, then we have 
\begin{align*}
  \int_{\mathcal{T}^{|\mathcal{G}_{1}|} \cdots \mathcal{T}^{|\mathcal{G}_{M}|}} & \left[ \log \frac{p \left(\bm z|\bm x\right)}{q_{\phi}(\bm z| \bm x)}\right]  q_{\phi}(\bm z| \bm x) d\bm z \\
  & = \int_{\mathcal{T}^{|\mathcal{G}_{1}|} \cdots \mathcal{T}^{|\mathcal{G}_{M}|}} \sum_{m=1}^M \left[ \log \frac{p \left(\widetilde{\bm z}_{m}|\bm x\right)}{q_{\phi}(\widetilde{\bm z}_{m}| \bm x)}\right] \left[\prod_{m=1}^M q_{\phi}(\widetilde{\bm z}_{m}| \bm x) \right] d\widetilde{\bm z}_{1} \cdots d\widetilde{\bm z}_{M} \\
  &  = \sum_{m=1}^M  \int_{\mathcal{T}^{|\mathcal{G}_{m}|}} \left[ \log \frac{p \left(\widetilde{\bm z}_{m}|\bm x\right)}{q_{\phi}(\widetilde{\bm z}_{m}| \bm x)}\right] q_{\phi}(\widetilde{\bm z}_{m}| \bm x) d\widetilde{\bm z}_{m}.
\end{align*}
Similar to the results in (\ref{EQU:S-ADVI-1}), we can obtain the lower bound of IWAE based on the SN-VI is $\log p(\bm x) - C\sum_{m=1}^{M}H^{-(\varrho_m+1)}- M\epsilon$.
\end{proof}

According to the results in Lemma 4.1, we can further quantify the variational approximation error with respect to the class defined in (3). 
\begin{proof}[Proof of Theorem 4.2]
According to Remark 2.2 and similar to the Proof of Lemma 4.1, there exists  $q_{\phi}(\bm z| \bm x)$ such that $\sup_{\bm z}|p(\bm z|\bm x) - q_{\phi}(\bm z| \bm x)|\leq C\sum_{m=1}^{M}H^{-(\varrho_m+1)}$. When $H$ goes to infinity, $H^{-(\varrho_m+1)}$ goes to zero. One can also infer that $|p(\bm z|\bm x) - q_{\widehat{\phi}}(\bm z| \bm x)|\leq C\sum_{m=1}^{M}H^{-(\varrho_m+1)}$ almost everywhere on $\mathcal{T}^{|\mathcal{G}_{1}|} \times \cdots \times \mathcal{T}^{|\mathcal{G}_{M}|}$. Then, the following property holds that is, 
\begin{align*}
   & \mathcal{L}_{\text {IWAE }}(\widehat{\phi}) = \mathbb{E}_{\left\{q_{\widehat{\phi}}( \bm z_t| \bm x)\right\}_{t=1}^T}\left[\log \frac{1}{T} \sum_{t=1}^T \frac{p\left(\bm z_t|\bm x\right)}{q_{\widehat{\phi}}\left(\bm z_t| \bm x\right)}\right] + \log p(\bm x)  \\
   & = \mathbb{E}_{\left\{q_{\widehat{\phi}}( \bm z_t| \bm x)\right\}_{t=1}^T}\left\{\log  \left[1 + \frac{1}{T} \sum_{t=1}^T \frac{p\left(\bm z_t|\bm x\right) - q_{\widehat{\phi}}\left(\bm z_t| \bm x\right)}{q_{\widehat{\phi}}\left(\bm z_t| \bm x\right)}\right]\right\} + \log p(\bm x) \\
   & = \mathbb{E}_{\left\{q_{\widehat{\phi}}( \bm z_t| \bm x)\right\}_{t=1}^T}\left\{\frac{1}{T} \sum_{t=1}^T \frac{p\left(\bm z_t|\bm x\right) - q_{\widehat{\phi}}\left(\bm z_t| \bm x\right)}{q_{\widehat{\phi}}\left(\bm z_t| \bm x\right)} + o\left[\sum_{m=1}^{M}H^{-(\varrho_m+1)}\right]\right\} + \log p(\bm x)\\
   & = \mathbb{E}_{q_{\widehat{\phi}}( \bm z| \bm x)}\left\{\frac{p\left(\bm z|\bm x\right) - q_{\widehat{\phi}}\left(\bm z| \bm x\right)}{q_{\widehat{\phi}}\left(\bm z| \bm x\right)} + o\left[\sum_{m=1}^{M}H^{-(\varrho_m+1)}\right] \right\} + \log p(\bm x) \\
   & = \mathbb{E}_{q_{\widehat{\phi}}( \bm z| \bm x)}\left\{\log \frac{p\left(\bm z|\bm x\right)}{q_{\widehat{\phi}}\left(\bm z| \bm x\right)}+ o\left[\sum_{m=1}^{M}H^{-(\varrho_m+1)}\right]\right\} + \log p(\bm x) \\
   & = -D_{\textrm{KL}}\left[q_{\widehat{\phi}}\left( \bm z| \bm x\right)||p\left(\bm z|\bm x\right)\right] + \log p(\bm x) + o\left[\sum_{m=1}^{M}H^{-(\varrho_m+1)}\right]
\end{align*}
The result in Lemma 4.1 implies the KL divergence of the spline density approximation from the true posterior is in the order of $\sum_{m=1}^{M}H^{-(\varrho_m+1)} + M\epsilon$.  
\end{proof}

\subsection{Proof of Theorem 4.3} 
\label{SEC:THE:Approximation-Regression-Proof}

\begin{lemma}[Lemma 6.3 \citep{csiszar2006context}]
\label{LEM:Kl-L2}
    If $p$ and $q$ are probability densities both supported on a bounded interval $\mathcal{T}$, then we have the KL divergence between probability densities satisfies that $D_{\mathrm{KL}}(p||q) \leq \frac{1}{\inf _{x \in \mathcal{T}} q(x)}\|p-q\|_2^2.$
\end{lemma}
\begin{proof} Notice that
    $$
\begin{aligned}
D_{\mathrm{KL}}(p||q) & =\int_{\mathcal{T}} p(x) \log \frac{p(x)}{q(x)} \mathrm{d} x \leq \int_{\mathcal{T}}p(x)\left\{\frac{p(x)}{q(x)}-1\right\} \mathrm{d} x  =\int_{\mathcal{T}}\frac{\{p(x)-q(x)\}^2}{q(x)} \mathrm{d} x
\end{aligned}
$$
from which the claim follows.
\end{proof}

Here we provide the proof sketch of Theorem 4.3. Our proof is based on the assumption when the observed data points $\boldsymbol x$ and $\boldsymbol x'$ are close enough, the corresponding posteriors $p(\bm z|\boldsymbol x)$ and $p(\bm z|\boldsymbol x')$ are similar to each other. For any given posterior $p(\bm z|\boldsymbol x)$ under satisfying Assumption (A2), we can identify the density function based on spline approximation $\prod_{m=1}\sum_{\bm \kappa \in \bm \Kappa_{|\mathcal{G}_{m}|}} \gamma_{\bm \kappa} b_{\bm \kappa}(\widetilde{\bm{z}}_{m})$ close to $p(\bm z|\boldsymbol x)$ with differences bounded by $\sum_{m=1}^{M}H^{-(\varrho_m+1)}$. Under some mild assumptions, $\phi^{\ast}$ can be well approximated by nonparametric regression, such as the deep neural network. Combining the results in Theorem 4.2, we can further obtain the KL divergence between the proposed SN-VI estimator and the true posterior.

\begin{proof}[Proof of Theorem 4.3]

    In the following, we denote that the unknown parameters as $\phi = \{ \bm \mu_m(\bm x), \bm \sigma_m(\bm x), \bm \Gamma_{m}(\bm x) | m =1 , \dots, M\}$. We notice that $ \sum_{i=1}^n \mathcal{L}_{\text{ELBO}}\{\widehat{\phi}(\bm x_i)\} = \sum_{i=1}^n \log p(\bm x_i)-\sum_{i=1}^n D_{KL}\left\{q_{\widehat{\phi}(\bm x_i)}(\bm z)~||~p(\bm z | \bm x_i)\right\}$, we can infer to the average of the KL divergence between $q_{\widehat{\phi}(\bm x_i)}(\bm z)$ and $p(\bm z | \bm x_i)$ equaling to 
    $$n^{-1}\sum_{i=1}^n \log p(\bm x_i) - n^{-1}\sum_{i=1}^n \mathcal{L}_{\text{ELBO}}\{\widehat{\phi}(\bm x_i)\}.$$
    In addition, for a given $\xi$ and any $j\in \mathcal{G}_{m}$, consider a finite support $\mathcal{T}_j^{\ast}=[\mu_j^\ast(\bm x), \mu_j^\ast(\bm x) + \sigma_j^\ast(\bm x)]$ such that $\int_{\prod_{j \in \mathcal{G}_{m}}\mathcal{T}_j^{\ast}}p\left(\widetilde{\bm z}_{m}|\bm x\right) d\widetilde{\bm z}_{m} \geq 1-\epsilon$. For a specific posterior $p(\bm z|\bm x)$, the optimal parameters are $\phi^\ast(\bm x) = \{ \bm \mu^{\ast}_m(\bm x), \bm \sigma^{\ast}_m(\bm x), \bm \Gamma^{\ast}_{m}(\bm x) | m =1 , \dots, M\}$, where $\Gamma^{\ast}_{m}(\bm x)$'s are spline coefficients satisfying that  $\sup_{\widetilde{\bm z}_{m}
 \in \prod_{j \in \mathcal{G}_{m}}\mathcal{T}_j^{\ast}}|p(\widetilde{\bm z}_{m}
|\bm x) - s^\ast(\widetilde{\bm z}_{m}
;\bm x)|\leq \sum_{m=1}^{M}H^{-(\varrho_m+1)}$, where $s^\ast(\widetilde{\bm z}_{m};\bm x) =\sum_{\bm \kappa \in \bm \Kappa_{|\mathcal{G}_{m}|}} \gamma_{\bm \kappa} b_{\bm \kappa}(\widetilde{\bm{z}}_{m})$ and $m = 1,\dots, M$.

    Next, we denote the estimators of the optimal parameters generated from nonparametric regression, such as the deep neural network, as $\widetilde{\phi}(\bm x) = \{ \widetilde{\bm \mu}_m(\bm x), \widetilde{\bm \sigma}_m(\bm x), \widetilde{\bm \Gamma}_{m}(\bm x) | m =1, \dots, M\}$. When the observed datapoints $\boldsymbol x$ and $\boldsymbol x'$ are close enough, the corresponding posteriors $p(\bm z|\boldsymbol x)$ and $p(\bm z|\boldsymbol x')$ are close, so do the corresponding optimal parameters $\phi^\ast(\bm x)$ and $\phi^\ast(\bm x')$.  Then, we have 
    \begin{align*}
    \frac{1}{n}& \sum_{i=1}^n D_{\mathrm{KL}}\left\{q_{\widehat{\phi}(\bm x_i)}(\bm z)~||~p(\bm z | \bm x_i)\right\} \\
      & = \frac{1}{n}\sum_{i=1}^n \log p(\bm x_i) - \frac{1}{n}\sum_{i=1}^n \mathcal{L}_{\text{ELBO}}\{\widehat{\phi}(\bm x_i)\}  \leq \frac{1}{n}\sum_{i=1}^n \log p(\bm x_i) - \frac{1}{n}\sum_{i=1}^n \mathcal{L}_{\text{ELBO}}\{\widetilde{\phi}(\bm x_i)\}\\
      & = \frac{1}{n} \sum_{i=1}^n D_{\mathrm{KL}}\left\{q_{\widetilde{\phi}(\bm x_i)}(\bm z)~||~p(\bm z | \bm x_i)\right\}. 
    \end{align*}

    According to Lemma \ref{LEM:Kl-L2}, we have 
    \begin{align*}
    \frac{1}{n}&\sum_{i=1}^nD_{\mathrm{KL}}  \left\{q_{\widetilde{\phi}(\bm x_i)}(\bm z)~||~p(\bm z | \bm x_i)\right\} \\
    & = \frac{1}{n}\sum_{i=1}^n \int  q_{\widetilde{\phi}(\bm x_i)}(\bm z) \log \frac{q_{\widetilde{\phi}(\bm x_i)}(\bm z)}{p(\bm z | \bm x_i)} \mathrm{d} \bm z = \sum_{i=1}^n \sum_{m=1}^M \int  q_{\widetilde{\phi}(\bm x_i)}(\widetilde{\bm z}_{m}
) \log \frac{q_{\widetilde{\phi}(\bm x_i)}(\widetilde{\bm z}_{m}
)}{p(\widetilde{\bm z}_{m}
 | \bm x_i)} d \widetilde{\bm z}_{m}
 \\
    & \leq \frac{1}{n}\sum_{i=1}^n \sum_{m=1}^M \frac{1}{\inf_{\widetilde{\bm z}_{m}
 \in \prod_{j\in \mathcal{G}_m}\widetilde{\mathcal{T}}_j} p(\widetilde{\bm z}_{m}
|\bm x_i)} \|p(\widetilde{\bm z}_{m}
| \bm x_i) - q_{\widetilde{\phi}(\bm x_i)}(\widetilde{\bm z}_{m}
)\|^2  \\
    & \leq \frac{1}{n}\sum_{i=1}^n \sum_{m=1}^M \frac{2}{\inf_{\widetilde{\bm z}_{m}
 \in \widetilde{\mathcal{T}}_j} p(\widetilde{\bm z}_{m}
 |\bm x_i)} \left\{\|p(\widetilde{\bm z}_{m}
| \bm x_i) - q_{\phi^{\ast} (\bm x_i)}(\widetilde{\bm z}_{m}
)\|^2 + \|q_{\widetilde{\phi}(\bm x_i)}(\widetilde{\bm z}_{m}
) - q_{\phi^{\ast}(\bm x_i)}(\widetilde{\bm z}_{m}
)\|^2\right\}\\
    & \leq C_1 \left\{\sum_{m=1}^{M}H^{-2(\varrho_m+1)} + M\epsilon^2\right\} + C_1\sum_{m=1}^{M}H^{2|\mathcal{G}_{m}|}  \Delta^2,
    \end{align*}
    where $\prod_{j\in \mathcal{G}_m}\widetilde{\mathcal{T}}_j$ is the support of density function $q_{\widetilde{\phi}(\bm x_i)}(\widetilde{\bm z}_{m})$ and $\widetilde{\mathcal{T}}_j = [\widetilde{\mu}_j(\bm x), \widetilde{\mu}_j(\bm x) + \widetilde{\sigma}_j(\bm x)]$. 
    
    According to the properties of tensor product spline basis functions, we have \begin{align*}
        & n^{-1}\sum_{i=1}^n\|q_{\widetilde{\phi}(\bm x_i)}(\widetilde{\bm z}_{m}) - q_{\phi^{\ast}(\bm x_i)}(\widetilde{\bm z}_{m})\|^2 \\
        & \quad \leq n^{-1}\sum_{i=1}^n \left(C H^{|\mathcal{G}_{m}|} \right)\sum_{\bm \kappa \in \bm \Kappa_{|\mathcal{G}_{m}|}} \left[\gamma^{\ast}_{\bm \kappa}(\bm x_i)-\widetilde{\gamma}_{\bm \kappa}(\bm x_i) \right ]^2 = O(H^{2|\mathcal{G}_{m}|}  \Delta^2).
    \end{align*}

\end{proof}

\newpage
\section{Implementation Details}

\subsection{Penalized Spline} Penalized splines in nonparametric smoothing capture complex patterns while using regularization to prevent overfitting and control complexity \citep{Wood:03}. 
Here, we consider a roughness penalty for a spline function $s(\widetilde{\bm z}_{m})$, defined as 
$\mathcal{E}(s) = \int \sum_{j \in \mathcal{G}_m} \{\nabla^{2}_{z_j}s(\widetilde{\bm z}_m)\}^2d\widetilde{\bm z}_m.
$ See Supplement Section \ref{sec:penalty} for the implementation details of $\mathcal{E}(s)$. With the penalty term, the objective function becomes 
$\mathcal{L}^P_{\text{IWAE}}\{\phi(\bm x)\} = \mathcal{L}_{\text{IWAE}}\{\phi(\bm x)\}- \lambda \mathcal{E}(s),$ where $\lambda$ is the roughness hyperparameter. As $\lambda$ increases, the penalty on highly fluctuating approximations increases, leading to a smoother approximation by reducing variance but at the cost of potentially increasing bias. Conversely, if 
$\lambda$ is too small, the model will be overfitting. Simulation studies in Supplement Section \ref{sec:rough} demonstrate that increasing $\lambda$
results in overly smooth posterior approximations, while decreasing 
$\lambda$ reduces the smoothness constraint and allows more flexible, detailed approximations.

\subsection{Penalized Tensor Product Spline}
\label{sec:penalty}
Here, we introduce the details of the penalty term $\mathcal{E}(s)$ mentioned in Section 3.1. Denote $c_0 =0$ and $c_{m} = \sum_{m' \leq m} |\mathcal{G}_m|$, where $|\mathcal{G}_m|$ is the cardinality of the set $\mathcal{G}_m$. According to properties of spline polynomials, $\mathcal{E}(s) = \bm \Gamma_{m}^{\top} (\sum_{j=c_{m-1}+1 }^{c_m}\mathbf{P}_j)\bm \Gamma_{m}$, where $\mathbf{P}_j = (\bigotimes^{j-1}_{l=c_{m-1}+1}\mathbf{M}_{l})\otimes\mathbf{S}_{j}\otimes (\bigotimes_{l=j+1}^{c_m}\mathbf{M}_{l})$, $\bigotimes_{l=a}^{b} \mathbf{M}_{l} =\mathbf{M}_{a}\otimes \cdots \otimes \mathbf{M}_{b}$, and ``$\otimes$" denotes the Kronecker product. $\mathbf{S}_{j}$ and $\mathbf{M}_{j}$ are $K_j\times K_j$ matrices with $(\mathbf{S}_{j})_{k_j,k_j^{'}} = \int_{\mathcal{T}} \nabla_{z_j}^{2} b_{k_j}(z_j) \nabla_{z_j}^{2} b_{k_j'}(z_j)dz_j$ and $(\mathbf{M}_{j})_{k_j,k_j'} = \int_{\mathcal{T}} b_{k_j}(z_j) b_{k_j'}(z_j)dz_j$. For the specific case where $\widetilde{\bm z}_1 = (z_1, z_2)$ and $\mathcal{G}_1 = \{1, 2\}$, $\mathcal{E}(s) = \bm \gamma_{\mathcal{G}_1}^{\top} (\sum_{j\in{\mathcal{G}_1}}\mathbf{P}_j\bm )\gamma_{\mathcal{G}_1}  = \bm \gamma_{\mathcal{G}_1}[(\mathbf{S}_{1}\otimes \mathbf{M}_{2}) +(\mathbf{M}_{1}\otimes \mathbf{S}_{2})]  \gamma_{\mathcal{G}_1}.$

%
\subsection{Derivation with Reparameterization Trick}
\label{sec:re_trick}
\begin{lemma}
    \label{Prop:density}
    If $\bm z = g(\bm \epsilon)$, $\bm g$ is a monotone function, then the density function of $\bm z$ is 
    $\bm f_{\bm z}(\bm z) = \bm f_{\bm \epsilon} \{\bm g^{-1}(\bm z)\} \left|\frac{\partial g^{-1}(\bm z)}{\partial \bm z} \right|,$
    where $\left|\frac{\partial g^{-1}(\bm z)}{\partial \bm z} \right|$ is the determinate of the Jacob matrix $\frac{\partial g^{-1}(\bm z)}{\partial \bm z} $. 
\end{lemma}
Let's recall IWAE:
\begin{align}
   & \mathcal{L}_{\text{IWAE}}\{\phi(\bm x)\} =\mathbb{E}_{\left\{\bm z^{(t)} \sim q_\phi(\bm z| \bm x)\right\}_{t=1}^T}\left[\log \frac{1}{T} \sum_{t=1}^T \frac{p\left(\bm x|\bm z^{(t)}\right) p\left(\bm z^{(t)}\right)}{q_\phi\left(\bm z^{(t)}| \bm x\right)}\right].
\end{align}

Note that $\widetilde{\bm{z}}_{m} = \bm \mu_m(\bm x) + \bm \sigma_m (\bm x) \bm \epsilon_{m}$ for each group $\mathcal{G}_m$. According to Lemma \ref{Prop:density}, the function $q_{\phi}(\widetilde{\bm z}_{m}|\bm x)$ is 
\begin{equation}
\label{EQU:density}
\left\{\prod_{j \in \mathcal{G}_m} \frac{1}{\sigma_j(\bm x)} \right\} \cdot \sum_{\bm \kappa \in \bm \Kappa_{|\mathcal{G}_{m}|}} \gamma_{\bm \kappa}(\bm x) b_{\bm \kappa}\left(\frac{\widetilde{\bm z}_{m} - \bm \mu_{m}(\bm x)}{\bm \sigma_{m}(\bm x)}\right)
\end{equation} 
Plugging in (\ref{EQU:density}) to $\mathcal{L}_{\text{IWAE}}(\phi)$, we have
\begin{align*}
   & \mathcal{L}_{\text{IWAE}}(\phi) = \mathbb{E}_{\left\{\bm z^{(t)} \sim q_\phi(\bm z| \bm x)\right\}_{t=1}^T}\left[\log \frac{1}{T} \sum_{t=1}^T \frac{p\left(\bm x|\bm z^{(t)}\right) p\left(\bm z^{(t)}\right)}{q_\phi\left(\bm z^{(t)}| \bm x\right)}\right]\\
    &~= \int \left[\log \frac{1}{T} \sum_{t=1}^T \frac{p\left(\bm x, \bm z^{(t)}\right)}{q_{\phi}(\bm z^{(t)}| \bm x)} \right] \left[ \prod_{t=1}^T q_{\phi}(\bm z^{(t)}| \bm x) \right] d\bm z^{(t)}  \\
    &~= \int \left[\log \frac{1}{T} \sum_{t=1}^T \frac{p\left(\bm x, \bm z^{(t)}\right)}{q_{\phi}(\bm z^{(t)}| \bm x)} \right]  \left\{\prod_{t=1}^T  \prod_{m=1}^M \prod_{j \in \mathcal{G}_m} \frac{1}{\sigma_j(\bm x)} \cdot \sum_{\bm \kappa \in \bm \Kappa_{|\mathcal{G}_{m}|}} \gamma_{\bm \kappa}(\bm x) b_{\bm \kappa}\left(\frac{\widetilde{\bm z}_{m} - \bm \mu_{m}(\bm x)}{\bm \sigma_{m}(\bm x)}\right) \right\} d\bm z^{(t)} \\
    &~= \int \left[\log \frac{1}{T} \sum_{t=1}^T \frac{p_{\theta}\left(\bm x, \bm \mu(\bm x) \!+\! \bm \sigma (\bm x) \cdot \bm \epsilon^{(t)} \right)}{\prod_{m=1}^M \left\{\prod_{j \in \mathcal{G}_m} \frac{1}{\sigma_j(\bm x)} \right\} \cdot \sum_{\bm \kappa \in \bm \Kappa_{|\mathcal{G}_{m}|}} \gamma_{\bm \kappa}(\bm x) b_{\bm \kappa}\left(\bm \epsilon^{(t)}_{m}\right)} \right] \\
    & \quad \times \left\{\prod_{t=1}^T  \prod_{m=1}^M \sum_{\bm \kappa \in \bm \Kappa_{|\mathcal{G}_{m}|}} \gamma_{\bm \kappa}(\bm x) b_{\bm \kappa}\left(\bm \epsilon^{(t)}_m\right) \right\} d\bm \epsilon^{(t)} \\
    &~= \int \left[\log \frac{1}{T} \sum_{t=1}^T \frac{p_{\theta}\left(\bm x, \bm \mu(\bm x) + \bm \sigma (\bm x) \cdot \bm \epsilon^{(t)} \right)}{\prod_{m=1}^M \sum_{\bm \kappa \in \bm \Kappa_{|\mathcal{G}_{m}|}} \gamma_{\bm \kappa}(\bm x) b_{\bm \kappa}\left(\bm \epsilon^{(t)}_{m}\right)} \right] \\
    & \quad \times \left\{\prod_{t=1}^T  \prod_{m=1}^M \sum_{\bm \kappa \in \bm \Kappa_{|\mathcal{G}_{m}|}} \gamma_{\bm \kappa}(\bm x) b_{\bm \kappa}\left(\bm \epsilon^{(t)}_m\right) \right\}d\bm \epsilon^{(t)}+ \sum_{j=1}^J \log \sigma_j(\bm x) \\
    &~= \mathbb{E}_{\left\{\bm \epsilon^{(t)}\right\}_{t=1}^T}\left[\log \frac{1}{T} \sum_{t=1}^T \frac{p\left\{\bm x,\bm \mu(\bm x) + \bm \sigma (\bm x) \cdot \bm \epsilon^{(t)} \right\}}{\prod_{m=1}^M\sum_{\bm \kappa \in \bm \Kappa_{|\mathcal{G}_{m}|}} \gamma_{\bm \kappa}(\bm x) b_{\bm \kappa}\left(\bm \epsilon_{m}^{(t)}\right)}\right]+ \sum_{j=1}^J\log \sigma_j(\bm x).
\end{align*}

\subsection{Concrete Relaxation, Annealing, and Computational Efficiency}
\label{sec:supp_concrete_annealing}

This section provides additional details on the use of the concrete distribution and the annealing strategy in the proposed SN-VI algorithm. These techniques are introduced to approximate a discrete hierarchical sampling step by a differentiable relaxation, so that the resulting objective can be optimized efficiently using automatic differentiation.

\subsubsection{Concrete distribution}
In the proposed method, the hierarchical representation involves selecting one component from a finite set of candidate spline basis components. Directly sampling from a categorical distribution is not differentiable with respect to the selection probabilities, which makes it difficult to incorporate this sampling step into gradient-based variational optimization. To address this issue, we use the concrete distribution \citep{maddison2017concrete}, also known as the Gumbel--Softmax relaxation, as a continuous approximation to the categorical distribution.

Let $L$ denote the number of candidate spline basis components. Let
\[
\bm{\alpha}=(\alpha_1,\ldots,\alpha_L)\in(\mathbb{R}^{+})^L
\]
be the positive parameters controlling the relative selection probabilities of these components. A larger value of $\alpha_l$ corresponds to a larger probability of selecting the $l$th component. Let $\tau>0$ denote the temperature parameter. To generate a relaxed categorical vector, we first sample independent Gumbel random variables $G_1,\ldots,G_L$, and then define
\[
u_l
=
\frac{
\exp\{(\log \alpha_l+G_l)/\tau\}
}{
\sum_{r=1}^L \exp\{(\log \alpha_r+G_r)/\tau\}
},
\qquad l=1,\ldots,L.
\]
The resulting random vector
\[
\bm{u}=(u_1,\ldots,u_L)
\]
lies on the probability simplex, satisfying $u_l\geq 0$ and $\sum_{l=1}^L u_l=1$. Therefore, $\bm{u}$ can be interpreted as a differentiable approximation to a one-hot categorical indicator vector. Here, $u_l$ denotes the relaxed weight assigned to the $l$th spline basis component.

As $\tau\to 0$, the vector $\bm{u}$ becomes increasingly concentrated on a single component and converges to a one-hot categorical random vector. In this limiting case, the selected component follows a categorical distribution with probabilities proportional to $\bm{\alpha}$. In contrast, when $\tau$ is large, the components of $\bm{u}$ are smoother and more evenly distributed, corresponding to a soft weighted combination of multiple spline basis components.

\subsubsection{
Role of \texorpdfstring{$\bm{\alpha}$}{alpha}
and
\texorpdfstring{$\bm{u}$}{u}
}The parameter vector $\bm{\alpha}$ controls the relative importance of the candidate spline basis components. During optimization, $\bm{\alpha}$ is updated by stochastic gradients and therefore learns which basis components contribute more to the variational approximation. The random vector $\bm{u}$ is the relaxed selection vector generated from the concrete distribution. Instead of selecting exactly one component, the algorithm uses $\bm{u}$ to form a differentiable weighted combination of candidate components. For example, if $\{B_l(\cdot):l=1,\ldots,L\}$ denotes the candidate spline basis components, then the relaxed selected basis representation can be written as
\[
\sum_{l=1}^L u_l B_l(\cdot).
\]
When the temperature is high, this expression combines several basis components. When the temperature is low, most of the weight is assigned to one component, and the expression approximates a discrete basis selection.

\subsubsection{Annealing strategy}
We further use an annealing strategy \citep{abid2019concrete} for the temperature parameter. Let $\tau(c)$ denote the temperature at epoch $c$. We initialize the algorithm with a relatively large temperature $\tau_0$ and then gradually reduce the temperature according to a deterministic schedule, for example
\[
\tau(c)
=
\max\{\tau_{\min},\tau_0\exp(-\eta c)\},
\]
where $\eta>0$ is the decay rate and $\tau_{\min}>0$ is a lower bound used for numerical stability. This schedule produces smooth relaxed weights in the early stage of training and sharper, nearly discrete weights in the later stage.

The use of annealing has two main purposes. First, a relatively large initial temperature stabilizes optimization by avoiding premature selection of a single basis component. This allows the algorithm to explore different basis components and obtain stable gradient updates. Second, gradually lowering the temperature encourages the relaxed weights to become more concentrated, so that the final approximation is closer to the intended discrete hierarchical sampling procedure. In this sense, the annealing strategy provides a smooth transition from a continuous weighted approximation to a nearly discrete component-selection mechanism.

\subsubsection{Connection to ADVI and convergence interpretation}
The proposed algorithm can be viewed as a modification of ADVI in which the non-differentiable discrete sampling step is replaced by a differentiable concrete relaxation. For any fixed positive temperature $\tau$, the algorithm optimizes a well-defined relaxed variational objective. Since the sampling operation admits a reparameterized form through the Gumbel random variables, gradients of the relaxed objective can be estimated using automatic differentiation.

Under standard regularity conditions for stochastic-gradient variational inference, such as smoothness of the relaxed objective, bounded or controlled gradient variance, and an appropriate learning-rate schedule, the stochastic-gradient iterates converge to a stationary point of the relaxed objective. This is the same type of convergence guarantee commonly associated with ADVI and other nonconvex stochastic-gradient variational inference algorithms. Because the variational objective is generally nonconvex, neither standard ADVI nor the proposed relaxed SN-VI algorithm guarantees convergence to a global optimum.

As the temperature decreases, the relaxed objective approaches the objective associated with the discrete hierarchical sampling procedure. Therefore, the concrete relaxation preserves the usual stationary-point convergence interpretation of ADVI for the relaxed objective, while providing a differentiable approximation to the discrete selection mechanism. We emphasize that this modification is not intended to provide a stronger global convergence guarantee than ADVI. Rather, it enables efficient gradient-based optimization of a model component that would otherwise require non-differentiable discrete sampling or exhaustive search.

\subsubsection{Computational efficiency}
The concrete relaxation and annealing strategy improve computational efficiency in two ways. First, they avoid exhaustive enumeration over all possible discrete spline-basis configurations. Without the relaxation, one would need to either search over many possible basis selections or use non-differentiable discrete sampling, both of which are computationally expensive and unstable in gradient-based optimization. The concrete distribution replaces this step with a differentiable weighted combination of basis functions, allowing all parameters to be updated jointly by stochastic gradients.

Second, the annealing strategy reduces the effective number of active basis components during training. At the beginning of optimization, several components may receive non-negligible weights, which helps stabilize learning. As the temperature decreases, the weights become increasingly concentrated on a smaller number of components. Consequently, the later stages of optimization focus on a reduced set of important basis components rather than treating all components equally. This reduces the effective computational burden and memory usage compared with a full-basis ADVI implementation that updates all spline coefficients simultaneously throughout training.

Overall, the concrete distribution serves as a differentiable approximation to the categorical selection of spline components, while annealing gradually sharpens this approximation. Together, these techniques allow the proposed SN-VI algorithm to retain the flexibility of spline-based posterior approximation while improving the stability and efficiency of model estimation.

\subsection{Details of Hierarchical Structure Identification Process}

Algorithm \ref{alg:improved_str} presents the details of hierarchical structure identification process.

\begin{algorithm}[!htb]
\caption{Hierarchical Structure Identification Process}
\label{alg:improved_str}
\begin{algorithmic}[1]
\REQUIRE Train dataset $\bm x_{\text{train}}$ and validation dataset $\bm x_{\text{valid}}$
\STATE \textbf{Initialize:} Set iteration counter $s = 0$. 
Each latent variable $z_j, \; j = 1, \dots, J$ forms a single group:
$\mathbb{G}^{(0)} = \{\mathcal{G}_1, \dots, \mathcal{G}_J\}
$
\STATE Compute the predictive log-likelihood: $L(\mathbb{G}^{(0)})$ and set $L(\mathbb{G}^{(-1)}) = -\infty$ for initialization.

\WHILE{$L(\mathbb{G}^{(s)}) > L(\mathbb{G}^{(s-1)})$}
    \STATE Generate candidate group structures from $\mathbb{G}^{(s)}$:  
    \[
    \mathbb{G}^{(s+1)}_{mm'} = \left( \mathbb{G}^{(s)} \setminus \{\mathcal{G}_m, \mathcal{G}_{m'}\} \right) \cup \{\mathcal{G}_{mm'}\},
    \]
    where $\forall \; m, m' = 1, \ldots, J - s$ and $\mathcal{G}_{mm'} = \mathcal{G}_m \cup \mathcal{G}_{m'}$
    \FOR{each candidate structure $\mathbb{G}^{(s+1)}_{mm'}$}
        \STATE Estimate the model parameters:
        $ \widehat{\phi}(\bm x) = \argmax_{\phi(\bm x)} \mathcal{L}^P_{\text{IWAE}} \{\phi(\bm x)\} $
        \STATE Compute the predictive log-likelihood: $ L(\mathbb{G}^{(s+1)}_{mm'})$
    \ENDFOR
    
    \STATE Update the group structure: $ \mathbb{G}^{(s+1)} = \argmax_{\mathbb{G}^{(s+1)}_{mm'}} L(\mathbb{G}^{(s+1)}_{mm'}) $
    \STATE Record the corresponding predictive log-likelihood: $ L(\mathbb{G}^{(s+1)}) = L(\mathbb{G}^{(s+1)}) $
    \STATE $s \gets s + 1$
\ENDWHILE

\STATE \textbf{Output:} Optimal group structure $\mathbb{G}^{(s)}$ with the highest predictive log-likelihood.
\end{algorithmic}
\end{algorithm}

\section{Additional Experimental Results}
\label{sec:add_exp}

\subsection{Computation Complexity Analysis}
\label{sec:ComputationComplexity}
In this experiment, we evaluate the computational scalability of the proposed SN-VI method with respect to the group size. We consider a single group of latent variables $\bm{z} = (z_1, \ldots, z_d)$, where $d$ denotes the total number of latent variables in the group. We generate training data from the following model. The prior distribution is specified as a standard multivariate Gaussian,
\[
\bm{z} \sim \mathcal{N}(\mathbf{0}, \mathbf{I}_d).
\]
To induce dependence in the posterior, we define a scalar observation model
\[
x \mid \bm{z} \sim \mathcal{N}\left(\sum_{j=1}^d z_j,\ \sigma_x^2\right),
\]
where $\sigma_x = 0.5$. This construction ensures that all latent variables jointly contribute to the likelihood, resulting in a strongly dependent posterior distribution.

Each latent group is modeled using a tensor-product spline basis with $K$ basis functions per dimension, leading to a total of $K^d$ spline coefficients. In this experiment, we consider $K=9$. We use a batch size of 64 and a total sample size of 2048. The number of importance samples is set to $T=10$, with regularization parameter $\beta=0.7$, spline penalty coefficient $5\times10^{-8}$, and no entropy penalty. The temperature parameter follows an exponential annealing schedule with decay rate $\eta=4$. All experiments are conducted on an NVIDIA A100 80GB GPU. For each configuration, we record the wall-clock training time and peak GPU memory usage. Table \ref{tab:scaling_results} summarizes the single-epoch wall-clock time, and peak GPU memory usage under different group sizes.
\begin{table}[]
    \centering
    \caption{Scalability of SN-VI with respect to the group size $d$: the number of spline coefficients, $K^d$, the single-epoch wall-clock time, and the peak GPU memory usage.}
    \begin{tabular}{cccc}
\hline
Group Size ($d$) & \# Spline Coefficients ($K^d$) & Time (sec) & Peak Memory (GB) \\
\hline
2 & 81     & 4.15    & 0.02 \\
3 & 729     & 24.95   & 0.10 \\
4 & 6561   & 299.10  & 1.28 \\
5 & 59049  & 3373.70 & 47.47 \\
\hline
\end{tabular}
    \label{tab:scaling_results}
\end{table}

\subsubsection{Subsampling} 
\label{sec:subsampling}
We further note that the above computational burden can be substantially reduced by using a subsampling strategy over the spline basis functions. Instead of optimizing over the full tensor-product spline basis, one may randomly sample a fraction of the basis functions and perform the SN-VI optimization using only this reduced set. This provides a simple and effective approximation to the full spline representation, especially when the posterior dependence structure is relatively low-dimensional or can be well captured by a subset of basis functions.

To evaluate this strategy, we conducted additional experiments for group sizes \(d=3\) and \(d=4\). For each setting, we randomly sampled different fractions of the full basis, including \(10\%,20\%,30\%,40\%,60\%,80\%\), and \(100\%\), and recorded both the predictive log-likelihood and the training time. The results are shown in Figure~\ref{fig:basis_subsampling}. For \(d=3\), the predictive log-likelihood improves substantially when increasing the fraction from \(10\%\) to \(20\%\), and then remains relatively stable as more basis functions are included. In particular, using only \(20\%\)--\(40\%\) of the basis functions already achieves performance close to the full-basis implementation, while reducing the training time from approximately 16 minutes to about 3--5 minutes. For \(d=4\), the training time increases more rapidly with the number of basis functions. Nevertheless, the subsampling approach still leads to a clear computational gain: using \(40\%\) of the basis functions reduces the training time from nearly 200 minutes to about 75 minutes, while maintaining comparable predictive log-likelihood to the full-basis implementation. These results suggest that basis subsampling can be a practical way to mitigate the computational burden of SN-VI when the full tensor-product basis becomes expensive.

\begin{figure}
    \centering
    \includegraphics[width=0.95\textwidth]{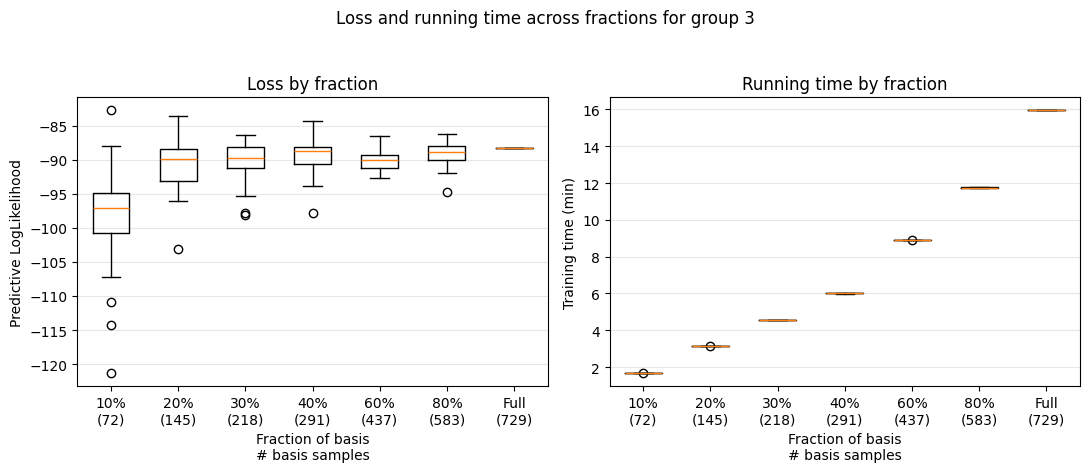}
\includegraphics[width=0.95\textwidth]{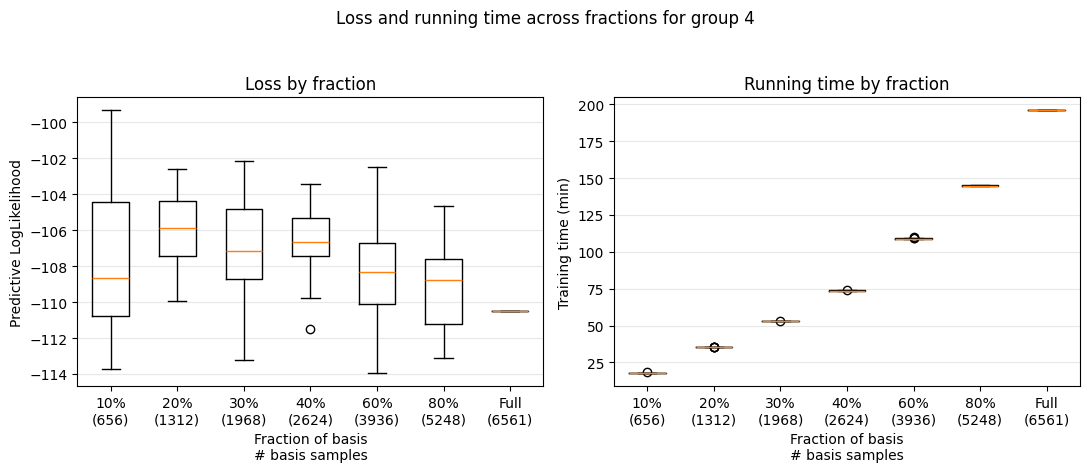}
    \caption{Effect of basis-function subsampling on SN-VI performance and computational time for group sizes \(d=3\) and \(d=4\).}
    \label{fig:basis_subsampling}
\end{figure}

\subsection{Baseline Implementation Details} We first introduce the baseline implementation details. To ensure a fair comparison, all competing variational inference methods were implemented using the same encoder--decoder backbone whenever applicable. For the spatial transcriptomics analysis, the encoder network has hidden-layer dimensions $(2000,300,300,16)$ and maps the input expression profile to an 8-dimensional latent representation. The decoder network has hidden-layer dimensions $(8,300,300,2000)$ and maps the latent variables back to the parameters of the observation model. The same training data, validation split, optimizer, learning-rate schedule, mini-batch size, and stopping rule were used across SN-VI and all baseline methods.

For Gaussian-ADVI, the variational posterior was specified as a diagonal Gaussian distribution. For S-ADVI, each latent variable was modeled independently using the univariate spline-based posterior approximation. For the normalizing-flow baselines, we considered planar flow and neural spline flow (NSF). Both flow-based methods used 10 flow transformations and were trained with the same encoder--decoder architecture and optimization protocol as the other neural-network-based variational methods. Hyperparameters, including the learning rate and regularization parameters, were selected using validation predictive log-likelihood rather than test-set performance. All methods were trained for the same maximum number of epochs under comparable computational budgets, and no baseline method was stopped prematurely.

For the computer-vision experiments, we used the same encoder and classifier/decoder architecture for all methods. The encoder was a two-layer MLP with 512 and 256 hidden units and ELU activation. For the classification task, the decoder was a classifier following the variational information bottleneck formulation. For the image reconstruction task, the decoder was a two-layer MLP with 512 and 256 hidden units. In addition to Gaussian-ADVI and S-ADVI, we included planar flow with 10 transformations and radial flow with 40 transformations. Each experiment was repeated over multiple random seeds, and the reported results are summarized as mean $\pm$ standard deviation.

\subsection{Additional results of Posterior Approximation}
\label{sec:rough}
This section presents additional experimental results to illustrate the impact of the hyperparameter on the posterior approximation. Table \ref{tab:example} presents the details of simulation settings.

\begin{table}[h!]
 \caption{Description of inference models for Cases 1 -- 4. }
    \centering
    \begin{tabular}{cll} 
        \hline \hline
       & Prior distribution & Likelihood function \\ 
        \hline
       Case 1 &$\bm{z} \sim \mathcal{N}\left((0.2,0.2)^{\top},\left(\begin{smallmatrix}
0.5 & 0\\
0 & 0.5
\end{smallmatrix}\right)\right)$\rule{0pt}{1.2em} &$\bm{x|z}\sim \mathcal{N}\left(\bm{z},\left(\begin{smallmatrix}
1 & 0.9\\
0.9 & 1
\end{smallmatrix}\right)\right)$ \\ 
        \hline
       Case 2 & 
$ (\mu,\tau)\sim \text{Normal-Gamma}(0,5,1,2,2)$\rule{0pt}{1.2em} & $x|\mu,\tau\sim \mathcal{N} (\mu, \tau^{-1})$ \\ 
        \hline
    Case 3& $\bm{z} \sim \mathcal{N}\left((0,0)^{\top},\left(\begin{smallmatrix}
1 & 0\\
0 & 1
\end{smallmatrix}\right)\right)$\rule{0pt}{1.2em}& $x|\bm z \sim \mathcal{N}(z_1+z^{2}_{2},1)$ \\
\hline
Case 4 & \makecell{$\bm z  \sim\sum_{k=1}^2 0.5 \mathcal{N}(\bm{\mu}_k,0.5\bm{I})$ \\ $\bm{\mu}_1 = (0.9,0.9)^{\top}$ and $\bm{\mu}_2 = (-0.9,-0.9)^{\top}$}&$\bm{x|z}\sim \mathcal{N}\left(\bm{z},\left(\begin{smallmatrix}
1 & 0.5\\
0.5 & 1
\end{smallmatrix}\right)\right)$
\\
    \hline 
    \end{tabular}
    \label{tab:example}
\end{table}

In Figure \ref{fig:rough}, we demonstrate how the roughness penalty parameter $\lambda$ will affect the posterior approximation of Case 4 in Section 6.1. When $\lambda$ is relatively large, the penalty term carries greater weight, encouraging the approximated surface to be smoother with less fluctuation. Conversely, as 
$\lambda$ decreases, the approximated surface becomes more flexible, allowing it to capture the multimodal nature of the true posterior distribution. Thus, the roughness parameter $\lambda$ plays a crucial role in balancing bias and variance in the approximation.
\begin{figure}
\centering\includegraphics[width=0.9\columnwidth]{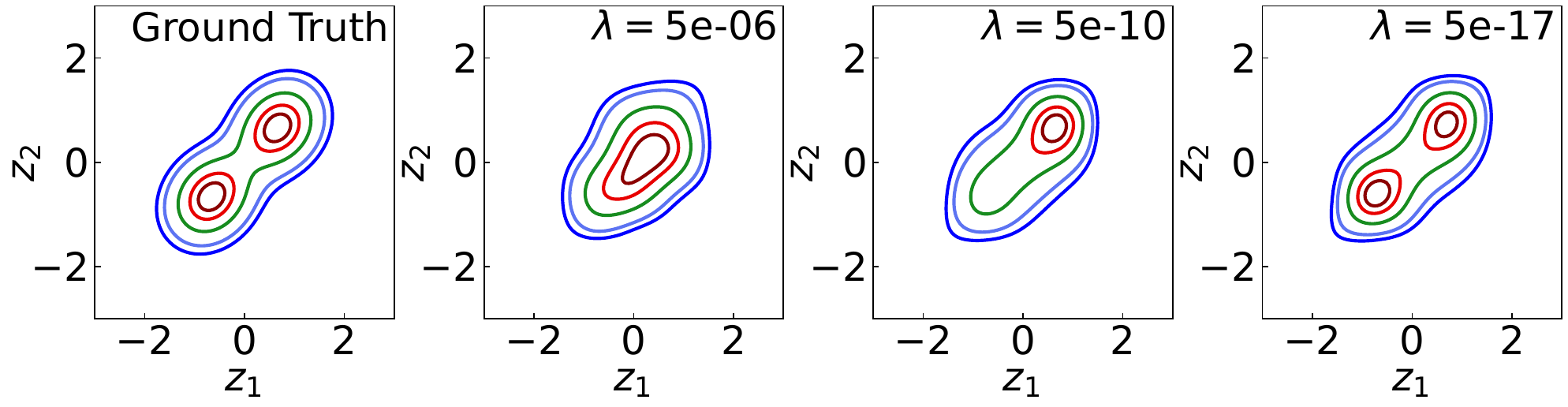}
    \caption{Effect of Hyperparameters in Posterior 
Approximations in Case 4.}
    \label{fig:rough}
\end{figure}

\begin{figure}
\centering\includegraphics[width=0.9\columnwidth]{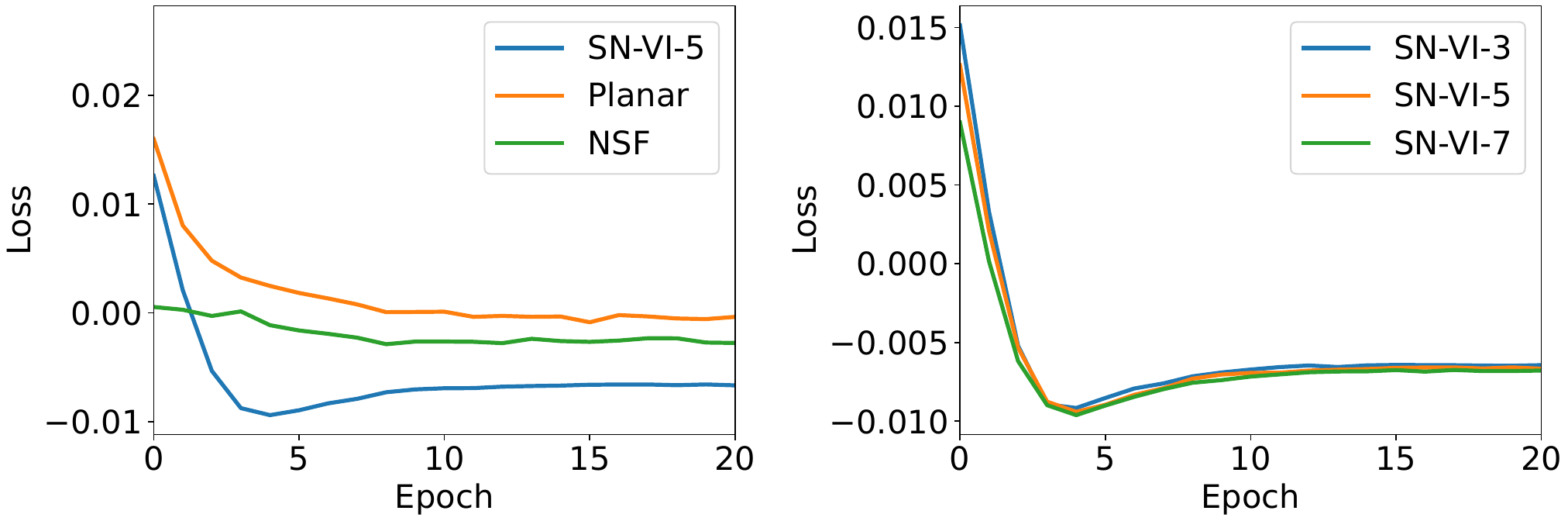}
    \caption{Computation cost comparisons of Case 2. The first 20 epochs are shown for visualization. The number after SN-VI represents number of knots.}
    \label{fig:cost}
\end{figure}

\subsection{Additional Results of Single Column Classification}
\label{sec:sup_clf}

\begin{figure}[!htbp]
\centering\includegraphics[width=1.0\columnwidth]{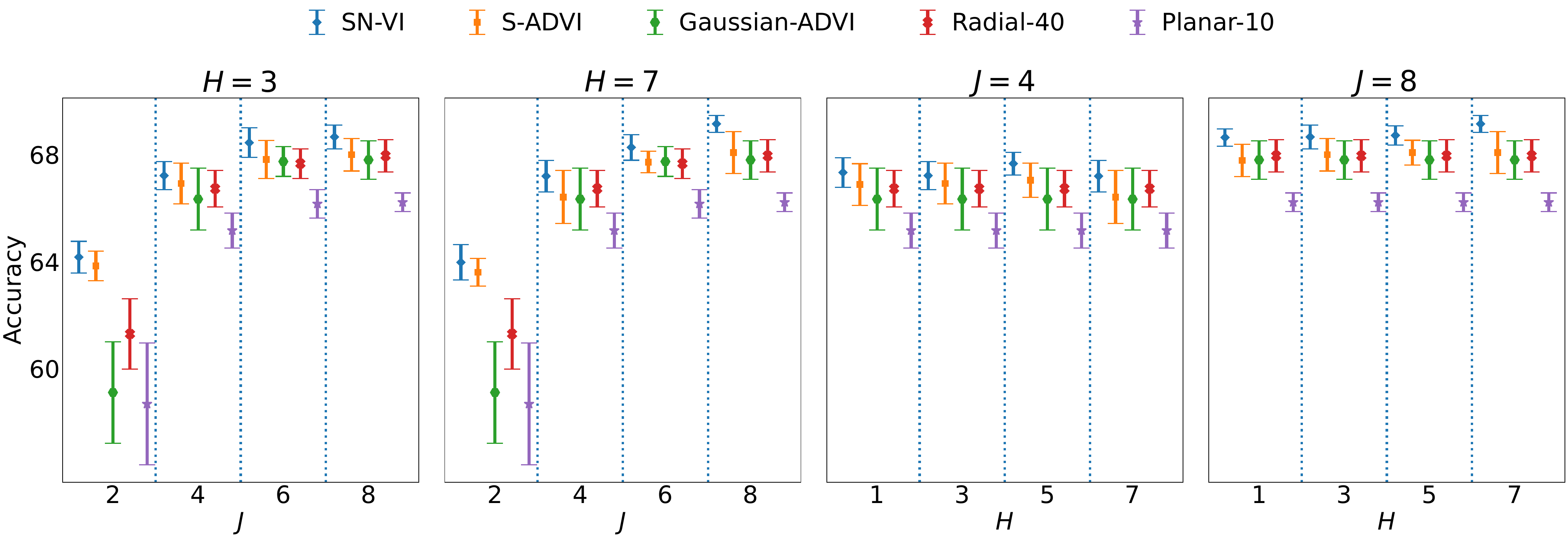}
    \caption{Single column classification performance comparison with different combinations of latent variables $J$ and interior knots $H$ varying number latent variables for FMNIST.}
    \label{fig:fmnist_clf}
\end{figure}
In this section, we include the additional results of single column classification. Figure \ref{fig:clf} presents and example of a single-column input sample enclosed by red dotted lines. We utilize an two-layer MLP with 512
and 256 hidden nodes and ELU activation function 
\citep{clevert2015fast} as the encoder. For the classification task, we apply the Variational Information Bottleneck (VIB) \citep{alemi2017deep} sharing a structure similar to VAE with the definition of the decoder as a classifier. The output layer estimates the parameters $\widehat{\phi}$ for the approximated posterior distributions. For both MNIST and FMNIST, we use a batch size of 64 and train for 50 epochs. The Adam optimizer is applied with an initial learning rate of 0.001, which decreases by 5\% at each epoch. 

\begin{figure}[!htbp]
    \centering
    \includegraphics[width=0.25\columnwidth]{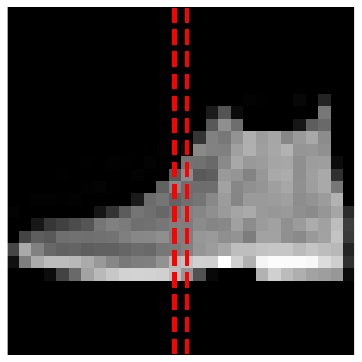}
    \caption{Example of a single-column input sample enclosed by red dotted lines.}
    \label{fig:clf}
\end{figure}

For SN-VI, we train the models with the combination of $H=\{1,3,5,7\}$ and $J=\{2,4,6,8\}$ to demonstrate the effect of hyperparameters. We also include two additional normalizing flow models, Planar-10 and Radial-40 (10 and 40 flows), in the comparison. All models are trained 10 times repeatedly, and the means and standard deviations are recorded. Figure \ref{fig:fmnist_clf} and Figure \ref{fig:mnist_clf} show the full results of the single column classification performance comparison for FMNIST and MNIST, where the error bar indicates the standard deviations. 

Both figures demonstrate that the accuracy of SN-VI improves as the number of latent variables 
$J$ and interior knots 
$H$ increase. When the number of latent variables is set to 4, SN-VI benefits significantly from adding more interior knots. However, as the number of latent variables becomes large, the effect of additional interior knots becomes negligible. Also, SN-VI outperforms normalizing flow methods and Gaussian-ADVI, especially when $J$ is small. S-ADVI exhibits similar performance to SN-VI when $J$ is small, as the dependency remains weak with fewer latent variables. However, as $J$ increases, SN-VI preserves a stronger dependent structure resulting in much better performance than S-ADVI based on mean-field assumption.

\begin{figure}[!htbp]
\centering\includegraphics[width=1.0\columnwidth]{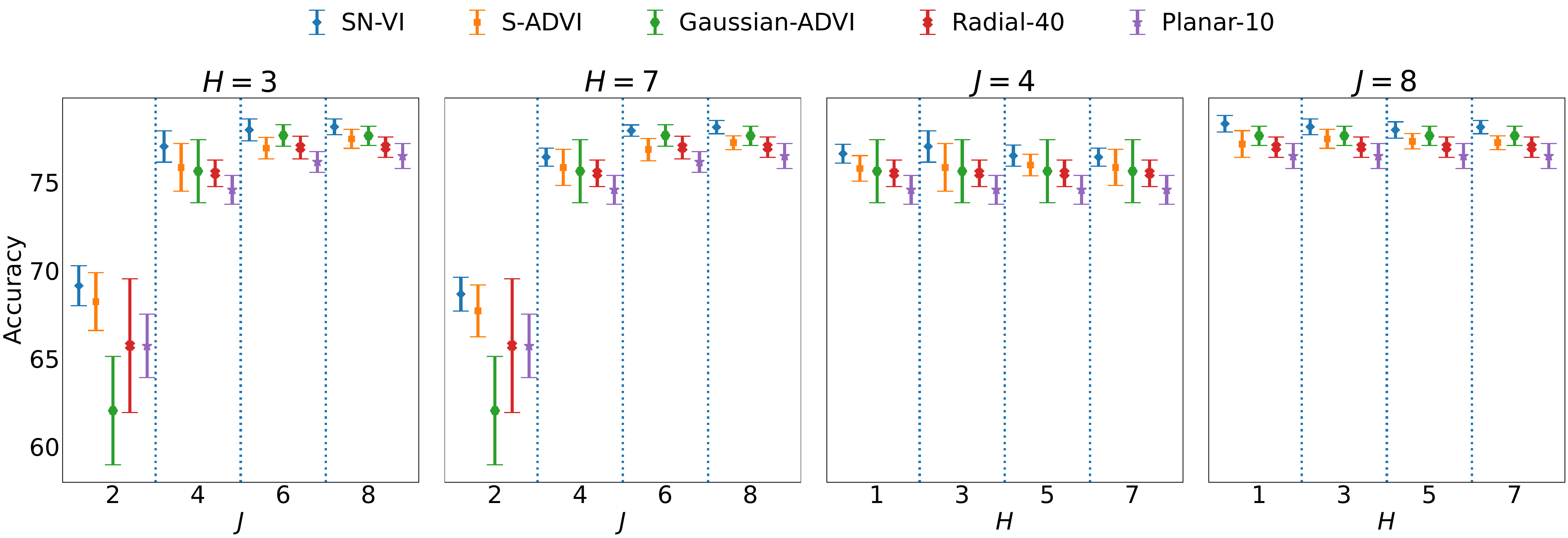}
    \caption{Single column classification performance comparison with different combinations of latent variables $J$ and interior knots $H$ varying number latent variables for MNIST.}
    \label{fig:mnist_clf}
\end{figure}

\subsection{Additional Results of Image Reconstruction}
\label{sec:recon}
In this section, we aim to demonstrate the capability of SN-VI in reconstructing image datasets by KMNIST \citep{clanuwat2018deep}, which are more complicated data than MNIST and FMNIST. To increase the difficulty of reconstruction and evaluate the capability of SN-VI, we introduce Gaussian noises into the dataset. Gaussian noises $\bm \xi$ are generated from $N(0,\sigma_{\bm \xi})$ and corrupted observed data $\bm x_{\text{corrp}} = \bm x+\bm \xi$. The corrupted level is determined by the magnitude of $\sigma_{\bm \xi}$. By applying the same neural network structure in Supplement Section \ref{sec:sup_clf} and replacing the decoder with a two-layer MLP with 512 and 256 hidden units, we repeat 10 times training for SN-VI, S-ADIV, Gaussian-ADVI, Radial-40, and Planar-10. The means and standard deviations of negative log-likelihood, $-\log(\bm x| \bm z)$, is reported, where we assume $\bm x|\bm z$ follows Gaussian distribution and set $\sigma_{\bm \xi} = 0.25$. We repeat 10 times training for SN-VI, S-ADIV, Gaussian-ADVI, Radial-40, and Planar-10.   
Figure \ref{fig:kmnist_recon_neg} demonstrates the performance of all methods in image reconstruction of KMNIST. Gaussian-ADVI has relatively limited performance due to the mean-field and Gaussian distribution assumptions. As $J$ increases, SN-VI performs better than normalizing flows by preserving the dependent structure and flexible posterior approximation.

\begin{figure}[!htbp]
\centering\includegraphics[width=1.0\columnwidth]{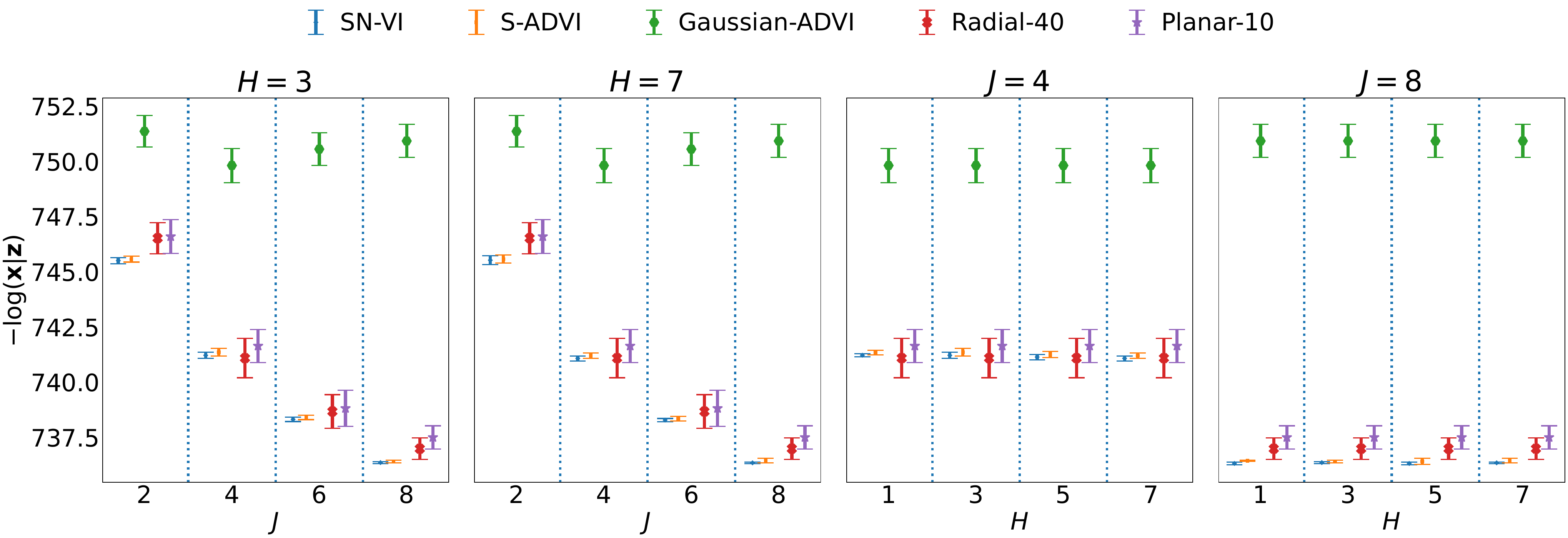}
    \caption{Reconstruction of corrupted KMNIST data by varying latent variables $J$ and interior knots $H$.}
    \label{fig:kmnist_recon_neg}
\end{figure}
\begin{figure}[!htbp]
\centering\includegraphics[width=0.6\columnwidth]{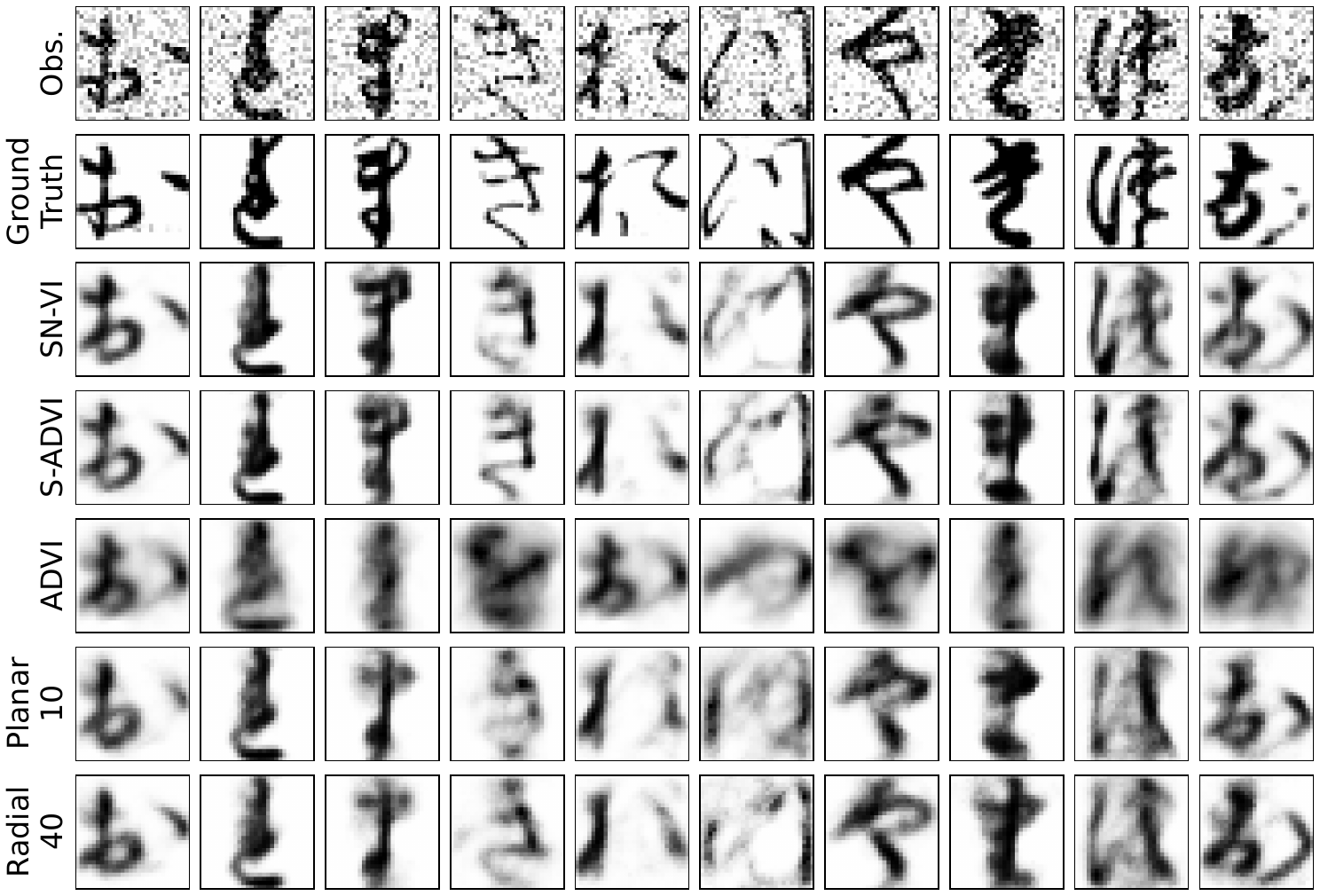}
    \caption{KMNIST image reconstruction examples from the test dataset.}
    \label{fig:kmnist_recon_eg}
\end{figure}

\subsection{Sensitivity Analysis of the Annealing Schedule}
Annealing is important in the proposed method to ensure most of the combinations of coefficients are explored. In this section, we evaluate the influences of annealing functions on model performance. We consider two annealing functions in this experiment:
\begin{description}
  \item[Exponential] $\tau(c)=\tau_1+ (\tau_0-\tau_1)e^{-c/\eta}$,
  \item[Linear] 
  \begin{equation*}
    \tau(c)=
    \begin{cases}
      \tau_0 - \frac{\tau_0 - \tau_1}{\eta'}c & \text{if $c \le \eta'$}\\
      \tau_1 & \text{if $c > \eta'$}    \end{cases}.
  \end{equation*}
\end{description}
The parameter $\eta$ in the exponential annealing schedule controls the rate at which the temperature decreases, whereas $\eta'$ in the linear schedule governs both the rate of decrease and the point at which the temperature stabilizes. Figure \ref{fig:temperature} illustrates the two annealing schedules across training epochs with parameters $\tau_0 = 1$, $\tau_1 = 0.05$, $\eta = 4$, and $\eta' = 10$, which are used in both our simulation studies and real data applications.

\begin{figure}
    \centering
    \includegraphics[width=0.5\linewidth]{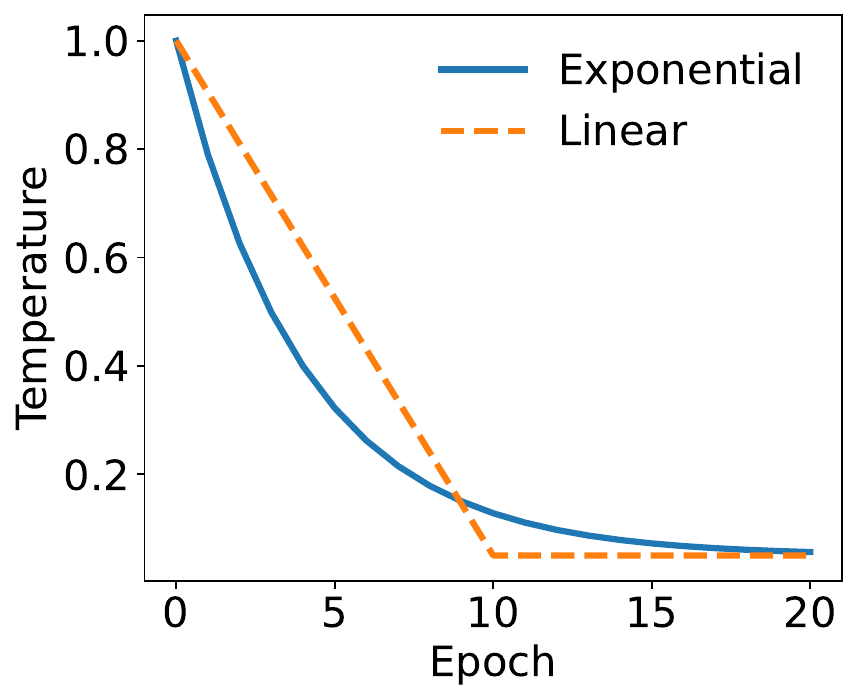}
    \caption{Example of two annealing functions.}
    \label{fig:temperature}
\end{figure}

In this experiment, we fix $\eta = 4$ and $\eta' = 10$, and consider two temperature settings: $(\tau_0, \tau_1) = (1, 0.05)$ and $(1.2, 0.07)$. We evaluate these annealing strategies on Case 2 and Case 3 in Section~6.1 and compare their performance using the RISE metric. To reduce variability due to random initialization, each configuration is trained over 20 independent runs. The results are summarized in Table \ref{tab:annealing}.
\begin{table}[]
    \centering
    \caption{Comparison of annealing strategies under different temperature settings \((\tau_0, \tau_1)\) across Case 2 and Case 3 using the RISE metric.}
    \begin{tabular}{lcccc}
\hline
 & \multicolumn{2}{c}{\(\tau_0 = 1, \tau_1 = 0.05\)} & \multicolumn{2}{c}{\(\tau_0 = 1.2, \tau_1 = 0.07\)} \\
\cline{2-5}
 & Exponential & Linear & Exponential & Linear \\
\hline
Case 2 & $0.090 \pm 0.016$ & $0.095 \pm 0.018$ & $0.098 \pm 0.015$ & $0.096 \pm 0.013$ \\
Case 3 & $0.085 \pm 0.020$ & $0.087 \pm 0.024$ & $0.086 \pm 0.013$ & $0.091 \pm 0.013$ \\
\hline
\end{tabular}
    \label{tab:annealing}
\end{table}

\subsection{Detailed Results for Figure 7}
The numerical results in Figure 7 are in Table \ref{tab:fmnist_full} and Table \ref{tab:KMNIST_FULL}.

\begin{table}
    \centering
    \caption{Classification accuracy (mean $\pm$ std) across different $J$ and $H$.}
    
    \begin{tabular}{llcccc}
\hline
Method & $H$ & $J=2$ & $J=4$ & $J=6$ & $J=8$ \\
\hline

\multirow{4}{*}{SN-VI}
& 1 & 63.50 $\pm$ 0.52 & 67.34 $\pm$ 0.55 & 68.01 $\pm$ 0.51 & 68.65 $\pm$ 0.32 \\
& 3 & 64.18 $\pm$ 0.59 & 67.23 $\pm$ 0.53 & 68.46 $\pm$ 0.55 & 68.67 $\pm$ 0.44 \\
& 5 & 64.22 $\pm$ 0.43 & 67.68 $\pm$ 0.42 & 68.01 $\pm$ 0.41 & 68.73 $\pm$ 0.36 \\
& 7 & 63.98 $\pm$ 0.66 & 67.21 $\pm$ 0.59 & 68.28 $\pm$ 0.48 & 69.16 $\pm$ 0.31 \\
\hline

\multirow{4}{*}{S-ADVI}
& 1 & 62.34 $\pm$ 1.10 & 66.89 $\pm$ 0.78 & 67.46 $\pm$ 0.63 & 67.79 $\pm$ 0.61 \\
& 3 & 63.85 $\pm$ 0.56 & 66.93 $\pm$ 0.76 & 67.83 $\pm$ 0.71 & 68.01 $\pm$ 0.61 \\
& 5 & 63.71 $\pm$ 0.93 & 67.05 $\pm$ 0.64 & 67.38 $\pm$ 0.58 & 68.09 $\pm$ 0.46 \\
& 7 & 63.61 $\pm$ 0.52 & 66.43 $\pm$ 0.99 & 67.73 $\pm$ 0.40 & 68.09 $\pm$ 0.78 \\
\hline

Gaussian-ADVI
& -- & 59.12 $\pm$ 1.90 & 66.35 $\pm$ 1.16 & 67.75 $\pm$ 0.56 & 67.81 $\pm$ 0.72 \\
Radial-40
& -- & 61.31 $\pm$ 1.31 & 66.74 $\pm$ 0.68 & 67.68 $\pm$ 0.55 & 67.97 $\pm$ 0.60 \\
Planar-10
& -- & 58.69 $\pm$ 2.28 & 65.18 $\pm$ 0.65 & 66.17 $\pm$ 0.53 & 66.23 $\pm$ 0.35 \\
\hline
\end{tabular}
\label{tab:fmnist_full}
\end{table}

\begin{table}[!ht]
    \centering
    \caption{Negative conditional log-likelihood, $\log p(x|z)$, (mean $\pm$ std).}
\begin{tabular}{llcccc}
\hline
Method & $H$ & $J=2$ & $J=4$ & $J=6$ & $J=8$ \\
\hline

\multirow{4}{*}{SN-VI}
& 1 & 745.46 $\pm$ 0.21 & 741.23 $\pm$ 0.07 & 738.30 $\pm$ 0.11 & 736.33 $\pm$ 0.06 \\
& 3 & 745.51 $\pm$ 0.14 & 741.22 $\pm$ 0.14 & 738.33 $\pm$ 0.10 & 736.37 $\pm$ 0.04 \\
& 5 & 745.46 $\pm$ 0.14 & 741.15 $\pm$ 0.12 & 738.25 $\pm$ 0.12 & 736.33 $\pm$ 0.06 \\
& 7 & 745.53 $\pm$ 0.20 & 741.09 $\pm$ 0.11 & 738.29 $\pm$ 0.07 & 736.36 $\pm$ 0.04 \\
\hline

\multirow{4}{*}{S-ADVI}
& 1 & 745.48 $\pm$ 0.22 & 741.35 $\pm$ 0.11 & 738.47 $\pm$ 0.14 & 736.45 $\pm$ 0.03 \\
& 3 & 745.58 $\pm$ 0.14 & 741.37 $\pm$ 0.17 & 738.41 $\pm$ 0.10 & 736.42 $\pm$ 0.06 \\
& 5 & 745.62 $\pm$ 0.17 & 741.26 $\pm$ 0.14 & 738.45 $\pm$ 0.10 & 736.42 $\pm$ 0.14 \\
& 7 & 745.58 $\pm$ 0.18 & 741.21 $\pm$ 0.13 & 738.36 $\pm$ 0.10 & 736.46 $\pm$ 0.11 \\
\hline

Gaussian-ADVI
& -- & 751.37 $\pm$ 0.71 & 749.82 $\pm$ 0.77 & 750.56 $\pm$ 0.74 & 750.93 $\pm$ 0.75 \\
Radial-40
& -- & 746.53 $\pm$ 0.70 & 741.10 $\pm$ 0.89 & 738.68 $\pm$ 0.76 & 737.00 $\pm$ 0.49 \\
Planar-10
& -- & 746.61 $\pm$ 0.76 & 741.65 $\pm$ 0.75 & 738.83 $\pm$ 0.82 & 737.51 $\pm$ 0.52 \\
\hline
\end{tabular}
\label{tab:KMNIST_FULL}
\end{table}

\subsection{Boundary Concentration Diagnostics}
we provide diagnostic checks to assess whether the compact-support assumption induces undesirable concentration of posterior mass near the boundaries. Specifically, under the setting of Case 1 in Section 6.1, we evaluate the probabilities $P(\epsilon_j < 0.02)$ and $P(\epsilon_j > 0.98)$ as measures of boundary behavior. To obtain an unbiased assessment across the input domain, we randomly generate 1,000 covariate vectors $\bs{x}$ from $[0,1]\times[0,1]$ and compute the corresponding probabilities under the conditional variational posterior $p(\epsilon_j \mid \bs{x})$. The resulting distributions of these probabilities are reported in Figure  \ref{fig:bound}. Figure \ref{fig:bound} shows that the estimated boundary probabilities are concentrated near zero, suggesting that the compact-support assumption does not induce substantial boundary concentration.
\begin{figure}
    \centering
    \includegraphics[width=1.0\linewidth]{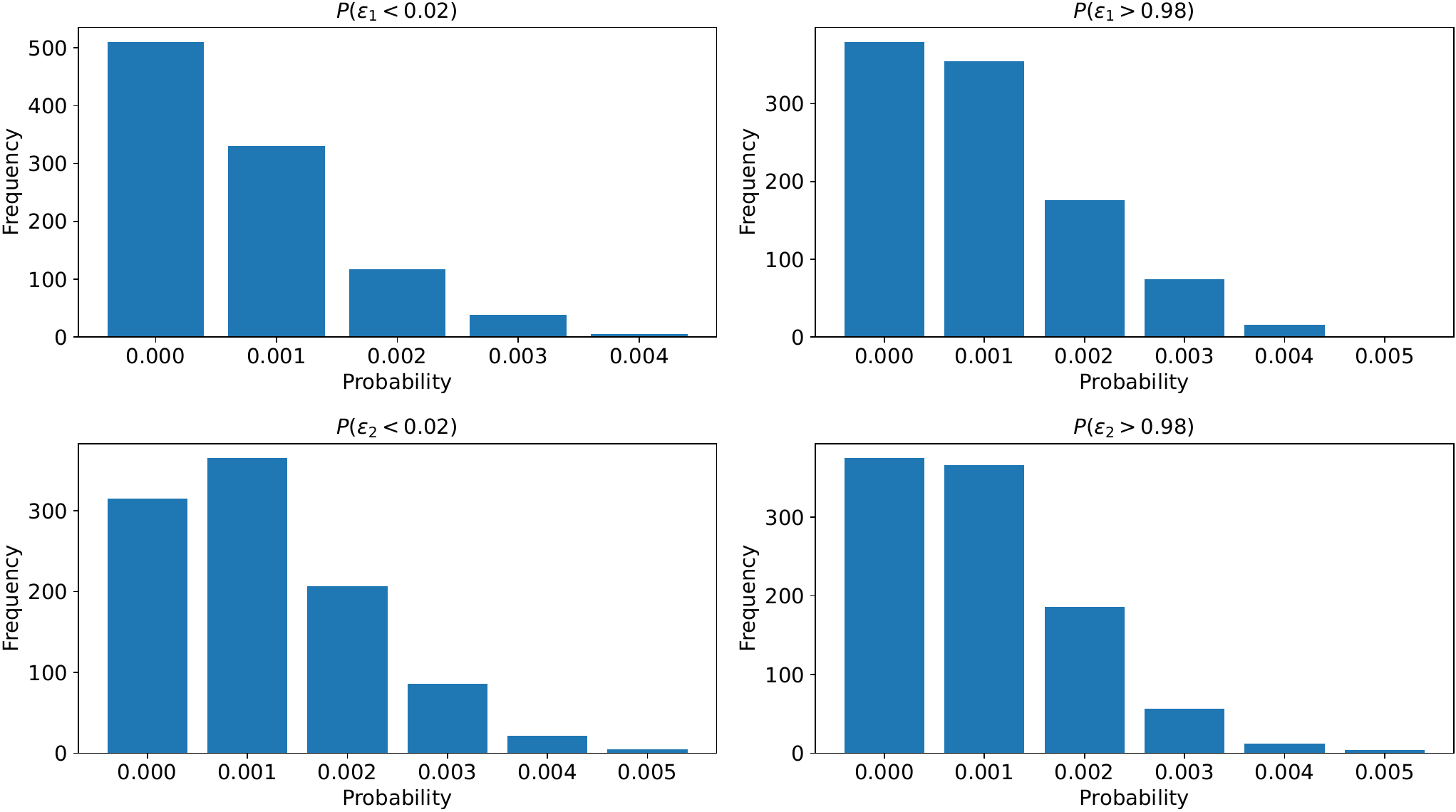}
    \caption{Empirical distributions of boundary probabilities under the SN-VI
variational posterior of Case 1 in Section 6.1.}
    \label{fig:bound}
\end{figure}
\subsection{Scalability Comparison}
To empirically assess scalability, we conducted additional experiments by varying the sample size while keeping the latent dimension, network architecture, and optimization settings fixed. Figure \ref{fig:scalability} shows that the runtime of all methods increases approximately linearly as the sample size grows, indicating favorable scalability with respect to the number of observations. Although SN-VI incurs additional computational cost due to the evaluation of grouped spline densities, its runtime remains approximately linear across all sample sizes considered. Combined with the simulation studies and real-data applications presented in Section 6, where SN-VI consistently demonstrates strong predictive and inferential performance, these results suggest that the proposed method scales well to larger datasets without sacrificing the accuracy and flexibility of the variational approximation. The empirical findings therefore support the practical applicability of SN-VI in both moderate- and large-scale settings.
\begin{figure}
    \centering
    \includegraphics[width=0.6\linewidth]{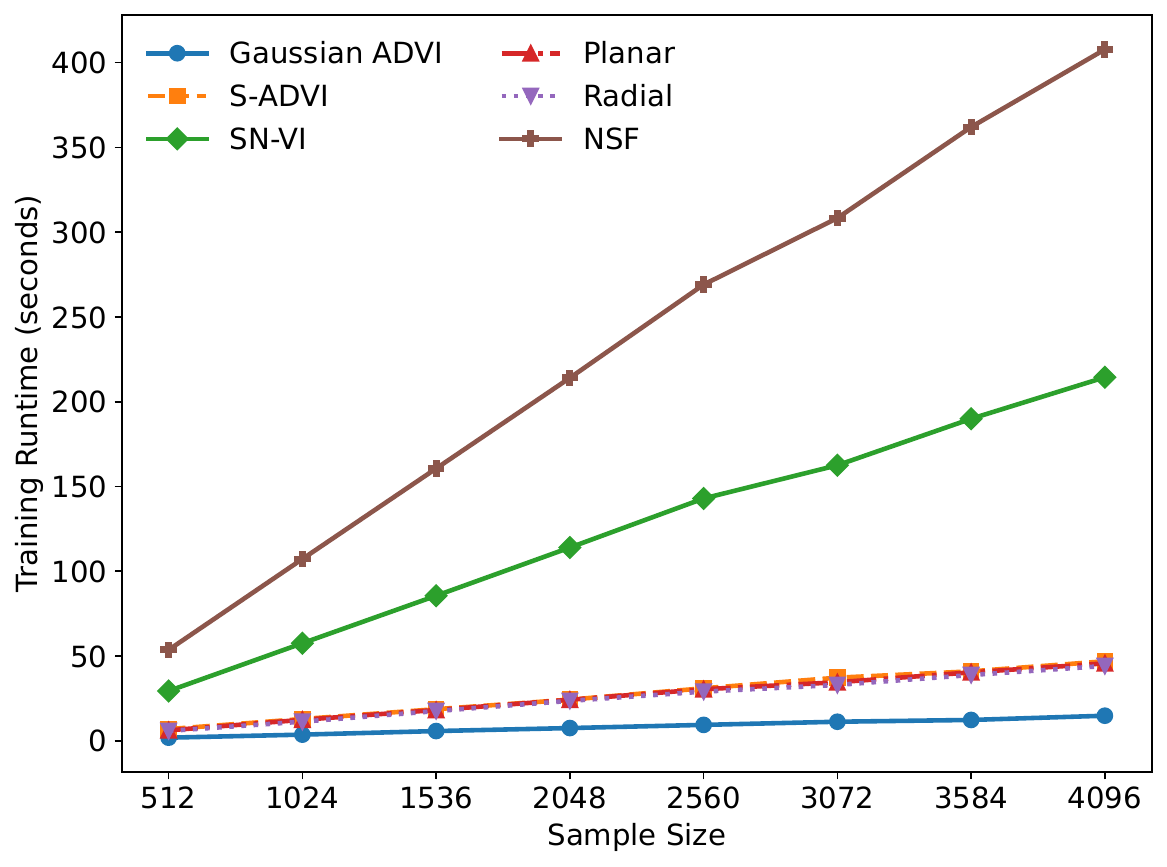}
    \caption{Training runtime versus sample size under the settings of Section~6.1.}
    \label{fig:scalability}
\end{figure}

\newpage

\end{document}